%% file: copy.tex
\documentclass[10pt,letterpaper]{article}
\input{math_commands}

\usepackage[margin=0.85in]{geometry}
\usepackage[T1]{fontenc}
\usepackage[utf8]{inputenc}

\usepackage{natbib}
\usepackage{hyperref}
\usepackage{url}
\usepackage[T1]{fontenc}
\usepackage{textcomp,booktabs,tabularx,array,microtype,graphicx}
\usepackage{amsmath,amsthm,amssymb,bm,mathtools,mathrsfs}
\usepackage{enumitem}
\usepackage{hyperref}
\usepackage{url}
\usepackage{algorithm}
\setcitestyle{numbers}
\usepackage{algpseudocode}
\usepackage{titletoc}

\newcommand{\Prob}{\mathbb P}
\newcommand{\U}{\mathcal U}

\newcommand{\one}{\mathbf1}

\newcommand{\spn}{\operatorname{sp}}

\theoremstyle{plain}
\newtheorem{theorem}{Theorem}[section]
\newtheorem{lemma}[theorem]{Lemma}
\newtheorem{proposition}[theorem]{Proposition}
\newtheorem{corollary}[theorem]{Corollary}

\newtheorem{example}{Example}

\newtheorem{remark}{Remark}

\title{Vector Bellman Theory for Multichain Robust Average-Reward Markov Decision Processes}

\author{Yue Wang, George Atia \\University of Central Florida}
\date{}
\begin{document}
\maketitle

\input{rebuttal-main-polished}

\bibliographystyle{acm}
\bibliography{iclr2027_conference}

\input{rebuttal-app-polished}

\end{document}

%% file: math_commands.tex
\usepackage{amsmath,amsfonts,bm}

\def\eqref#1{equation~\ref{#1}}

\def\1{\bm{1}}

\def\eps{{\epsilon}}

\DeclareMathAlphabet{\mathsfit}{\encodingdefault}{\sfdefault}{m}{sl}
\SetMathAlphabet{\mathsfit}{bold}{\encodingdefault}{\sfdefault}{bx}{n}

\newcommand{\E}{\mathbb{E}}

\newcommand{\R}{\mathbb{R}}

\DeclareMathOperator*{\argmin}{arg\,min}

%% file: rebuttal-main-polished.tex
\begin{abstract}
Robust average-reward Markov decision processes provide a fundamental framework for long-term performance optimization under uncertainty, and can have optimal long-run rewards that depend on the initial state. This state dependence requires a vector Bellman theory that accounts for both recurrent-class rewards and transition uncertainty. We develop such a theory for finite models with compact, post-action $(s,a)$-rectangular ambiguity. A gain-first, bias-second optimization principle yields a coupled vector gain-bias system, and every finite solution identifies the optimal robust gain and supplies stationary saddle strategies against history-dependent opponents, simultaneously from all initial states. We further characterize solvability through stationary gain conditions and a uniform bound on canonical transient corrections, and give sufficient conditions that permit distinct recurrent-class gains. The certificates also yield asymptotically affine trajectories of the robust Bellman operator, based on which we design a robust approximately shifted Halpern planning algorithm. Under finite Bellman solvability, the gain estimates and Bellman displacements converge to the optimal gain vector, and every extracted greedy controller is average-optimal after a finite, instance-dependent budget. These results thus connect finite Bellman certificates to undiscounted planning for state-dependent robust average rewards, providing theoretical understandings.

\end{abstract}

\section{Introduction}
Markov decision processes (MDPs) \cite{puterman2014markov} provide a standard framework for sequential decision-making under stochastic agent-environment interactions, in which an agent seeks a policy that maximizes expected reward under a specified performance criterion. A policy optimized for a single nominal MDP with a fixed transition model, however, can perform poorly when the deployed dynamics differ from that model or are imperfectly known. Such model mismatches widely exist in practice, due to, e.g., inaccurate model estimation, non-stationary environments, or unexpected perturbations upon deployment. Robust MDPs are thus proposed to address this mismatch by optimizing worst-case performance over a prescribed family of transition models \cite{iyengar2005robust,nilim2004robustness}, which captures the uncertain models and is called an ambiguity set or uncertainty set. To find such a robust policy, one needs to set the performance criterion. Among different criteria, average reward is particularly important as it measures sustained performance without imposing a discount factor, making it natural for systems operating over long horizons. Robust average-reward theory must therefore reconcile uncertainty in the transition dynamics with the chain structure that governs long-run performance.

In average-reward models, the limiting average-reward function (i.e., the long-run gain) is naturally dependent on the initial state. Even under a fixed stationary policy and fixed transition kernel, different starting states may reach recurrent classes with different class gains and different absorption probabilities: even two absorbing states with unequal rewards can yield a nonconstant gain vector. 
Under transition uncertainty, stationary gains can also change abruptly, which makes the analysis more challenging. If a zero-reward state reaches an absorbing unit-reward state with probability $\varepsilon$ per step, its gain is $1$ for every $\varepsilon>0$ but $0$ at $\varepsilon=0$. Vanishing transition probabilities can therefore change the recurrent structure. This raises two separate questions: whether nature attains the worst stationary gain and whether transient rewards admit a finite common bias.
These features thus require a state-dependent gain vector and an all-state gain-bias formulation, to tackle the complicated statewise dependence and provide a concrete Bellman-type characterization.

Existing work supplies two complementary foundations for such a theory. The early line of work \cite{wang2023robust,wang2024robust} develops the robust Bellman equation and algorithms for $(s,a)$-rectangular models under a uniform unichain condition, which ensures that the average reward (or the gain) is independent of the initial state. More recently, \cite{wang2025bellman} develop robust Bellman optimality with constant (optimal) gain, which covers more general conditions like one-sided weak communication. These conditions may permit individual stationary controller-nature pairs to be multichain, while the Bellman result remains in the state-independent-gain regime. Strategic value theory developed in \cite{grand2023beyond} establishes stationary controller optimality for general compact $(s,a)$-rectangular ambiguity, while polytopic models admit planning through finite stochastic games \cite{chatterjee2023solving}. For compact curved sets (e.g., KL or $\chi^2$ balls), no finite gain-bias characterization of a state-dependent robust gain is available; finite stochastic games supply one in the polytopic case \cite{akian2012policy}. These studies thus leave an open question: \textit{when the gain depends on the initial state, does this vector gain admit a finite gain-bias Bellman certificate, and how can such a certificate support planning?}
In this work, we provide concrete answers to the question. Our main contributions are summarized as follows.

\textbf{A gain-first, bias-second vector Bellman certificate.}
We first formulate a coupled Bellman system that gives continuation gain priority for both players {(nature and controller)} to characterize the robust average reward. Nature first minimizes continuation gain, and the controller maximizes the resulting minimum. Reward-bias optimization then takes place among the choices tied at this first level. We further show that every finite solution identifies the robust optimal gain and produces an all-state stationary saddle against history-dependent opponents. This verification controls nature's deviations from gain-minimizing rows, including rows with arbitrarily small gain gaps. It also connects the certificate to an asymptotically affine trajectory of the original Bellman operator (Section~\ref{m:formulation}).

\textbf{Exact solvability beyond communication assumptions.}
We then characterize when the finite certificate exists. For optimal control, the criterion combines a stationary saddle in the gain-restricted game with a uniform lower bound on canonical biases over all zero-gain replies to one securing controller. This separates stationary gain attainment from the control of transient rewards needed for a finite bias. We further explore concrete regimes that permit distinct recurrent-class gains: polytopic row sets, a uniform support gap, and different concrete distributional ambiguity sets, showing they are sufficient conditions for Bellman solvability (Section~\ref{m:stationary}).

\textbf{Direct undiscounted planning with a vanishing affine defect.}
We adapt approximately shifted Halpern iteration
\cite[Algorithm~1]{zurek2025faster} to the robust Bellman operator.
For compact row sets, a finite Bellman certificate can produce an
asymptotically affine trajectory whose error remains nonzero at every
finite time. We show that this vanishing defect suffices for convergence
of the gain estimates and Bellman displacements. Directional control of
the iterates then makes every extracted greedy controller eventually
average-optimal from all states. The analysis gives finite-budget bounds
in terms of the affine defect and compact examples with arbitrarily
small algebraic convergence exponents for the gain estimator
(Section~\ref{plan:section}).

\textit{Roadmap.} Section~\ref{m:foundations} recalls the strategic value results of \cite{grand2023beyond} and extends their finite-horizon convergence guarantee to general compact row sets. Section~\ref{m:formulation} develops the vector Bellman equation system and optimality guarantees of its solutions. Section~\ref{m:stationary} characterizes when that solution exists and identifies sufficient regimes. Section~\ref{plan:section} uses the certificate to analyze direct undiscounted planning. Thus the gain and optimal controller precede the Bellman system logically; finite solvability is the condition for the certificate and the planning theorem.

\section{Problem Formulation and Preliminaries}
\label{m:model}
We consider a finite robust MDP $(S,A,\mathcal U,r)$. The state space is $S=\{1,\ldots,n\}$, each action set $A(i)$ is finite and nonempty, and $r_{ia}\in\mathbb R$ is the one-stage reward. At time $t\ge0$, the controller observes $S_t$ and chooses $A_t\in A(S_t)$. Nature observes the chosen action, selects a transition row $p_t\in\U_{S_tA_t}\subseteq\Delta(S)$, and the next state is sampled from $p_t$. We assume $(s,a)$-rectangular ambiguity \cite{iyengar2005robust,nilim2004robustness}: $\mathcal U=\prod_{i\in S}\prod_{a\in A(i)}\U_{ia}$, where every row set $\U_{ia}$ is nonempty and compact. Vector inequalities are componentwise. We use the supremum norm $\|x\|_\infty:=\max_i|x_i|$ and the span $\spn(x):=\max_i x_i-\min_i x_i$. 

The planning objective is to compute a controller policy that maximizes worst-case long-run reward. Write $\Pi_D\subseteq\Pi_S\subseteq\Pi_H$ for deterministic stationary, randomized stationary, and randomized history-dependent controller strategies. A stationary policy belongs to $\prod_i\Delta(A(i))$, while a history-dependent policy maps the observed history $(S_0,A_0,\ldots,S_t)$ to a distribution on $A(S_t)$. Nature's randomized history-dependent strategies form $\mathcal Q_H$, and its full stationary selectors, that is, the transition kernels nature can select, form $\mathcal Q_S=\prod_i\prod_{a\in A(i)}\U_{ia}$. A full selector specifies a row for every state-action pair, including actions unused by a particular controller. Both players observe the initial state.

\textbf{Payoffs.}  Our primary performance criterion is the lower limiting expected average reward. We also record the corresponding upper limit and the two criteria that take the sample-path limit before expectation. These distinctions specify the strength of the strategy guarantees proved below. For $(\sigma,\tau)\in\Pi_H\times\mathcal Q_H$, let $\mathbb E_i^{\sigma,\tau}$ denote expectation under the induced process starting at $S_0=i$. For the trajectory $(S_0,A_0,S_1,A_1,\ldots)$, define $X_N=N^{-1}\sum_{t=0}^{N-1}r_{S_tA_t}$ and $J_N(i;\sigma,\tau)=\mathbb E_i^{\sigma,\tau}X_N$. We distinguish four average-payoff conventions \cite{puterman2014markov}:
\begin{equation}
\begin{aligned}
 J_i^-&=\liminf_N\mathbb E_i X_N,&
 J_i^+&=\limsup_N\mathbb E_i X_N,\\
 I_i^-&=\mathbb E_i\liminf_N X_N,&
 I_i^+&=\mathbb E_i\limsup_N X_N,
\end{aligned}
\label{m:four-payoffs}
\end{equation}
with the strategy pair suppressed. These criteria can differ for a fixed history-dependent pair.  For a payoff $\Psi$ and strategy classes $\mathcal C\subseteq\Pi_H$, $\mathcal N\subseteq\mathcal Q_H$, define the lower and upper values
\begin{equation}
 \underline v_i^\Psi(\mathcal C,\mathcal N)
 :=\sup_{\sigma\in\mathcal C}\inf_{\tau\in\mathcal N}
        \Psi_i(\sigma,\tau),\qquad
 \overline v_i^\Psi(\mathcal C,\mathcal N)
 :=\inf_{\tau\in\mathcal N}\sup_{\sigma\in\mathcal C}
        \Psi_i(\sigma,\tau).
\label{m:lower-upper-values}
\end{equation}
For a fixed controller $\sigma$, its robust performance is $\inf_{\tau\in\mathcal Q_H}\Psi_i(\sigma,\tau)$. The lower value maximizes this guarantee, while the upper value minimizes the controller's best response. A pair $(\bar\sigma,\bar\tau)$ is an all-state saddle for $\Psi$ at $u\in\mathbb R^S$ if $\Psi_i(\bar\sigma,\tau)\ge u_i\ge\Psi_i(\sigma,\bar\tau)$ for every $i$, $\sigma\in\Pi_H$, and $\tau\in\mathcal Q_H$. \cite[Theorems~3.5--3.6 and Corollary~3.7]{grand2023beyond} identify a common value for the four conventions; Appendix~\ref{app:strategic-duality} records the all-state formulation. Corollary~\ref{m:nature-attainment} gives the corresponding stationary-attainment test.

\textbf{Dynamic programming operators.} The order of play within each stage explains the robust Bellman maps. Given continuation value $x$ at state $i$, the controller first chooses $a$ and nature, having observed that action, chooses $p\in\U_{ia}$. The resulting one-step saddle operators are \cite{iyengar2005robust} 
\[
 (Tx)_i=\max_{a\in A(i)}\left\{r_{ia}+
                  \min_{p\in\U_{ia}}p^\top x\right\},\qquad
 (T^\pi x)_i=r_i^\pi+\sum_a\pi(a\mid i)
                  \min_{p\in\U_{ia}}p^\top x,
\]
where $r_i^\pi=\sum_a\pi(a\mid i)r_{ia}$ for $\pi\in\Pi_S$. Nature observes the sampled action and then selects the row, which explains the separate minima in $T^\pi$. For discount factor $1-\eps$, $0<\eps<1$, define the value vectors coordinatewise by $(V_\eps^\pi)_i=\inf_{\tau\in\mathcal Q_H}\mathbb E_i^{\pi,\tau}
       \sum_{t=0}^{\infty}(1-\eps)^t r_{S_tA_t}$, and $(V_\eps)_i=\sup_{\sigma\in\Pi_H}\inf_{\tau\in\mathcal Q_H}
       \mathbb E_i^{\sigma,\tau}\sum_{t=0}^{\infty}(1-\eps)^t r_{S_tA_t}.$ 
Rectangular discounted dynamic programming \cite{iyengar2005robust,nilim2004robustness} gives $V_\eps=T((1-\eps)V_\eps)$ and $ V_\eps^\pi=T^\pi((1-\eps)V_\eps^\pi).$
Each discounted fixed point is unique and is attained simultaneously from all states by deterministic stationary discounted-optimal selectors. Consequently, $(V_\epsilon)_i=\max_{\pi\in\Pi_D}(V_\epsilon^\pi)_i$. Appendix \ref{app:preliminaries} reviews the discounted dynamic-programming and finite-chain facts used in the analysis.

\textbf{Multichain structure and vector gain.} A stationary pair $(\pi,q)$ induces the transition matrix $P^{\pi,q}_{ij}=\sum_a\pi(a|i)q_{ia,j}$. Its gain vector is $\eta^{\pi,q}:=(P^{\pi,q})^\infty r^\pi$, where the Ces\`aro projector $ P^\infty:=\lim_{N\to\infty}\frac1N\sum_{t=0}^{N-1}P^t$ exists for every finite stochastic matrix, including periodic multichain matrices, and each coordinate of $\eta^{\pi,q}$ is an absorption-weighted average of recurrent-class rewards \cite{puterman2014markov}. This representation explains how initial states can have different gains. Unichain and irreducible assumptions in prior analyses make each stationary chain's gain constant \cite{wang2023robust,wang2024robust,xu2025efficient,roch2025reduction}. Here both evaluation and control concern the complete statewise gain vector. See Appendix~\ref{sec:model} for a review of existing results.

\section{Statewise values and stationary strategies}
\label{m:foundations}
We adopt the stationary-optimality, strategic-duality, payoff-equivalence, and discounted-limit results of \cite[Theorems~3.4--3.6, Corollary~3.7, and Lemmas~4.7--4.8]{grand2023beyond}. The statements below collect these foundations in statewise form and add finite-horizon convergence for general compact row sets. Appendix~\ref{app:preliminaries} gives the cited inputs and their all-state reductions; Appendix~\ref{app:finite-horizon-foundations} proves the finite-horizon extension.
\begin{theorem}[Policy evaluation and finite-horizon limit, extended from \cite{grand2023beyond}]
\label{m:policy-value}
For every $\pi\in\Pi_S$, there is $g^\pi\in\mathbb R^S$ such that, with both limits in supremum norm,
\begin{equation}
 g^\pi=\lim_{\eps\downarrow0}\eps V_\eps^\pi
      =\lim_{N\to\infty}\frac{(T^\pi)^N 0}{N},\text{ and }
 g_i^\pi=\inf_{q\in\mathcal Q_S}\eta_i^{\pi,q}, \forall i\in S.
\label{m:policy-value-formula}
\end{equation}
For every
$\Psi\in\{I^-,J^-,J^+,I^+\}$,
$\inf_{\tau\in\mathcal Q_H}\Psi_i(\pi,\tau)
 =\inf_{q\in\mathcal Q_S}\Psi_i(\pi,q)=g_i^\pi$.
Moreover, for each $\nu>0$, there exists some  $q^{\pi,\nu}\in\mathcal Q_S$ such that $ g^\pi\le\eta^{\pi,q^{\pi,\nu}}\le g^\pi+\nu\one$. 
\end{theorem}
The vector $g^\pi$ is the robust gain of $\pi$. A common stationary worst-case selector need not exist, but one selector approximates the entire vector to any positive tolerance.

The stationary evaluation and discounted limit are from \cite[Lemma~3.3, Corollary~3.7, and Lemma~4.7]{grand2023beyond}; the common approximate selector follows from their Theorem~4.3 by using a full-support initial law. The finite-horizon limit follows by applying \cite[Theorem~1.2]{ziliotto2016tauberian} to $T^\pi$.

The next result identifies the optimal robust gain and strategies securing it from all initial states. 

\begin{theorem}[Uniform finite-horizon guarantees, extended from \cite{grand2023beyond}]
\label{m:uniform-value}
There exist $g^\star\in\mathbb R^S$ and a deterministic stationary $\pi^\star\in\Pi_D$ such that
\begin{equation}
 g^\star=\lim_{\eps\downarrow0}\eps V_\eps
        =\lim_{N\to\infty}\frac{T^N0}{N}
        =\max_{\pi\in\Pi_D}g^\pi=g^{\pi^\star},
\label{m:value-formula}
\end{equation}
where the limits are in supremum norm and the maximum is coordinatewise. Moreover, for every $\delta>0$, there exist $q_\delta\in\mathcal Q_S$ and
$N_\delta<\infty$ such that, for all $i\in S$ and $N\ge N_\delta$,
\begin{equation}  
 J_N(i;\pi^\star,\tau)\ge g_i^\star-\delta
     \text{ for every }\tau\in\mathcal Q_H, \text{ and }
 J_N(i;\sigma,q_\delta)\le g_i^\star+\delta
    \text{ for every }\sigma\in\Pi_H. 
\label{m:uniform-foundation}
\end{equation}
\end{theorem}
Equation \ref{m:value-formula} identifies the optimal robust gain through discounted values, finite-horizon values, and policy optimization. One deterministic stationary controller $\pi^\star$ attains every coordinate. For each tolerance $\delta$, equation \ref{m:uniform-foundation} supplies a single stationary nature selector and one horizon threshold that work for every initial state, every later horizon, and every history-dependent opponent. The controller remains the same for all tolerances; nature's approximating selector can depend on $\delta$.

\begin{remark}
Convergence of $T^N0/N$ holds for compact row sets without any additional assumption. Theorem~\ref{m:uniform-value} therefore supplies a planning target throughout the model class. Later we will study the stronger certificate and the guarantees that follow when that certificate is finite.
The finite-horizon conclusion extends \cite[Theorem~5.2]{grand2023beyond} beyond definable ambiguity by combining their discounted limits with \cite[Theorem~1.2]{ziliotto2016tauberian}. 
\end{remark}
By \cite[Theorems~3.5--3.6 and Corollary~3.7]{grand2023beyond}, $g^\star$ is also the common lower and upper value under all four payoff conventions. The same controller $\pi^\star$ is optimal in each case. The values are preserved for intermediate strategy classes containing $\Pi_D$ and $\mathcal Q_S$, respectively, including the stationary strategy classes. See Appendix \ref{app:strategic-duality}.

For a full stationary selector $q$, let $d_i(q):=\max_{\pi\in\Pi_D}\eta_i^{\pi,q}$ be the optimal gain of the nominal MDP with $q$ fixed. Nature is exactly optimal from all states against every history-dependent controller if and only if $d(q)=g^\star$. This criterion checks the controller's best response over every action, including actions unused by $\pi^\star$. The equality $\eta^{\pi^\star,q}=g^\star$ checks only the prescribed controller.
Appendix~\ref{app:nature-attainment} records this consequence of the cited strategic results.

\section{Vector Bellman equation system}
\label{m:formulation}
Theorems~\ref{m:policy-value}-\ref{m:uniform-value} identify the robust value and an all-state optimal controller without assuming a finite bias. In this section, we investigate a stronger certificate, a gain-bias certificate: it must reconcile long-run gains and transient rewards for both players using one finite pair of vectors. We now formulate a local system and establish what any finite solution certifies.

\subsection{Robust Bellman optimality system: Gain first, bias second}
In the constant-gain setting \cite{wang2023robust,wang2025bellman}, the robust Bellman equation is $\rho\one+h=Th$. Every probability row preserves the same continuation gain because $p^\top(\rho\one)=\rho$. For a vector gain, different rows can lead to different long-run reward rates. The equation $g+h=Th$ alone does not require the selected transitions to preserve the proposed gain. The Bellman system must therefore compare continuation gains before comparing finite reward and bias terms.

To see the required order, consider a continuation vector $tg+h$ for large $t$. The one-step objective is $t p^\top g+r_{ia}+p^\top h$. A fixed gain difference dominates the bounded reward-bias term as $t$ grows. Nature therefore first minimizes $p^\top g$, and the controller maximizes this minimum. Among the choices tied in gain, both players optimize reward and continuation bias. Our system below imposes this ordering on both players, extending the nominal multichain gain-bias separation \cite{puterman2014markov,zurek2025faster}.

\textbf{Gain first, bias second.} For $g\in\R^n$, define the worst continuation gain of each action and the controller's optimal continuation gain by 
\begin{equation}
 m_{ia}(g):=\min_{p\in\U_{ia}}p^\top g,
 \qquad (\widehat{T}g)_i:=\max_{a\in A(i)}m_{ia}(g).
 \label{m:recession}
\end{equation}
For any $g\in\mathbb R^n$, define
\begin{equation}
 A_g(i):=\{a\in A(i):m_{ia}(g)=g_i\},\qquad
 F_{ia}(g):=\argmin_{p\in\U_{ia}}p^\top g,
 \label{m:faces}
\end{equation}
and we call $F_{ia}(g)$ a gain face even when the row set is nonconvex. Compactness makes every $F_{ia}(g)$ nonempty. When
$g=\widehat T g$, each $A_g(i)$ is also nonempty, and we define
the optimal-control bias operator on these active actions:
\begin{equation}
 (L_gh)_i:=\max_{a\in A_g(i)}\min_{p\in F_{ia}(g)}
                     (r_{ia}+p^\top h),\qquad K_gh:=L_gh-g.
 \label{m:tangent}
\end{equation}
Our  robust vector Bellman optimality system is 
\begin{equation}
g=\widehat{T}g,\qquad g+h=L_gh.
 \label{m:bellman}
\end{equation}
The two equations perform distinct tasks. The first enforces consistency of continuation gain, while the second determines the reward-bias balance among gain-optimal choices. They must be solved together: the first equation is reward-free and accepts every constant vector. The restrictions to $A_g(i)$ and $F_{ia}(g)$ preserve gain priority for the controller and nature, respectively. Specifically, Examples~\ref{er:ex-faces}-\ref{er:ex-actionfaces} show that omitting either restriction can certify an incorrect gain.

For singleton row sets, the system reduces to the classical multichain MDP optimality equations \cite{puterman2014markov}. If $g=\rho\one$, all actions and rows are gain-active, so $L_g=T$ and \eqref{m:bellman} becomes the constant-gain robust Bellman equation \cite{wang2023robust,wang2025bellman}.

The following theorem verifies the system against history-dependent opponents.
\begin{theorem}[Bellman optimality]
\label{m:verification}
Every finite solution $(g,h)$ of \eqref{m:bellman} satisfies $g=g^\star$. Choose a deterministic stationary policy $\pi(i)\in\mathop{\rm argmax}_{a\in A_g(i)}\min_{p\in F_{ia}(g)}\{r_{ia}+p^\top h\}$, and for every active state-action pair choose $q_{ia}\in\argmin_{p\in F_{ia}(g)}p^\top h$; at inactive actions choose any $q_{ia}\in F_{ia}(g)$. Then $(\pi,q)$ forms an all-state saddle for every payoff in \eqref{m:four-payoffs}: 
\begin{equation}
 \liminf_NJ_N(i;\pi,\tau)\ge g_i\quad\forall\tau\in\mathcal Q_H,
 \qquad
 \limsup_NJ_N(i;\sigma,q)\le g_i\quad\forall\sigma\in\Pi_H,
 \label{m:verify-guarantees}
\end{equation}
and this pair achieves the optimal robust average reward:
\(\lim_{N\to\infty}J_N(i;\pi,q)=g_i^\star\) for every $i\in S$.
\end{theorem}
\begin{proof}[Proof sketch]
Fix a finite solution $(g,h)$ and the selectors $(\pi,q)$ in the theorem.
We first prove the two expected-average guarantees, then identify the gain and
extend the saddle property to all four payoffs.
Appendix~\ref{fv:sec:finite} gives the full proof.

For a state $j$, action $a$, and row $p\in\U_{ja}$, define the gain drift
$d_{ja}(p):=p^\top g-g_j$ and the bias residual
$e_{ja}(p):=r_{ja}+p^\top h-g_j-h_j$.
Along a trajectory, write $d_t:=d_{S_tA_t}(p_t)$ and
$e_t:=e_{S_tA_t}(p_t)$, where $p_t$ is nature's realized row.
Conditionally on the history, action, and row, the next state has law $p_t$.
Taking expectations and summing the definitions therefore gives
\begin{equation}
\begin{aligned}
 NJ_N(i;\sigma,\tau)
 &=\sum_{t<N}\E_i g(S_t)+h_i-\E_i h(S_N)+\sum_{t<N}\E_i e_t,\\
 \sum_{t<N}\E_i d_t&=\E_i g(S_N)-g_i.
\end{aligned}
\tag{$\ast$}\label{m:sketch-identity}
\end{equation}
The endpoint terms are bounded by $\spn(h)$ and $\spn(g)$, respectively.
Thus it suffices to control the residuals by gain drift of one sign.

\emph{1. The controller's guarantee.}
At $a=\pi(j)$, gain activity gives $d_{ja}(p)\ge0$ for every feasible row,
and $d_{ja}(p)=0$ exactly on $F_{ja}(g)$.
The bias equation gives $e_{ja}(p)\ge0$ on this face.
A row outside the face can have a negative bias residual, but compactness
controls how much it can gain this way: for each $\theta>0$,
Lemma~\ref{fv:lem:nearface} gives a finite $C_\theta$ such that
$e_{j,\pi(j)}(p)\ge-\theta-C_\theta d_{j,\pi(j)}(p)$.
Taking the largest constant over the finitely many states makes this estimate
uniform in the row, the time, and nature's strategy.
Since the drift is nonnegative, $\E_i g(S_t)\ge g_i$ and
$\sum_{t<N}\E_i d_t\le\spn(g)$.
Substituting in \eqref{m:sketch-identity} yields
\[
 NJ_N(i;\pi,\tau)
 \ge N(g_i-\theta)-\spn(h)-C_\theta\spn(g)
 \qquad(\tau\in\mathcal Q_H).
\]
Divide by $N$, let $N\to\infty$, and then let $\theta\downarrow0$.
This proves the controller's guarantee. The order of these limits matters:
$C_\theta$ need not remain bounded as $\theta\downarrow0$.

\emph{2. Nature's guarantee.}
Fix the full selector $q$ and allow any controller.
At every active action, $d_{ja}(q_{ja})=0$ and
$e_{ja}(q_{ja})\le(L_gh)_j-g_j-h_j=0$.
At an inactive action, the drift is strictly negative.
Because there are finitely many actions, the constant
\[
 C:=\max_{j,a:\,a\notin A_g(j)}
       \frac{\max\{e_{ja}(q_{ja}),0\}}{g_j-m_{ja}(g)}
\]
is finite; set $C:=0$ if there are no inactive actions.
Thus $e_{ja}(q_{ja})\le-Cd_{ja}(q_{ja})$ at every state-action pair.
The nonpositive drift gives $\E_i g(S_t)\le g_i$ and
$-\sum_{t<N}\E_i d_t\le\spn(g)$, so
\eqref{m:sketch-identity} yields
$NJ_N(i;\sigma,q)\le Ng_i+\spn(h)+C\spn(g)$.
Dividing by $N$ proves nature's guarantee.

\emph{3. Gain identification and the four payoff conventions.}
Under the selected pair $(\pi,q)$, both drift and residual vanish.
Hence $J_N(i;\pi,q)=g_i+[h_i-\E_i^{\pi,q}h(S_N)]/N\to g_i$.
Together with Steps~1--2, this shows that the stationary robust gain of $\pi$
and the nominal optimal gain under $q$ both equal $g$:
$g^\pi=g=d(q)$.
The optimal robust gain lies between these two vectors,
$g^\pi\le g^\star\le d(q)$, and therefore $g=g^\star$.
Finally, Theorem~\ref{m:policy-value} and
Lemma~\ref{fd:lem:fixed-q}(ii) identify these same two values under each
$\Psi\in\{I^-,J^-,J^+,I^+\}$.
Thus $\inf_\tau\Psi_i(\pi,\tau)=g_i=\sup_\sigma\Psi_i(\sigma,q)$
for every initial state, which is the asserted saddle property.
\end{proof}

\textbf{Asymptotically affine Bellman trajectory.} The same gain-first ordering describes how the unrestricted Bellman operator acts far along a gain direction. Lemma~\ref{fv:lem:tangent} proves 
\[
  g=\widehat T g,\qquad g+h=L_g h
  \quad\Longleftrightarrow\quad
  T(h+tg)=h+(t+1)g+o(1),\qquad t\to\infty.
\]

The proof of this equivalence uses a compactness argument close to Step~1 of the sketch above: both control the bias improvement available from rows near a gain face. Theorem~\ref{m:verification} uses the near-face estimate of Lemma~\ref{fv:lem:nearface} directly, whereas Lemma~\ref{fv:lem:tangent} supplies the asymptotically affine trajectory used in the planning analysis of Section~\ref{plan:section}.

\subsection{Fixed-policy equation}
For a fixed stationary controller, the Bellman system certifies its robust gain and a stationary nature selector that is worst from all initial states. The action maximum is replaced by the policy average, while nature continues to minimize separately after each realized action.
\begin{proposition}[Fixed-policy Bellman equations]
\label{fv:prop:fixed-policy}
Fix any stationary randomized policy $\pi\in\Pi_S$. Suppose finite vectors $(g,h)$ satisfy the fixed-policy Bellman system: for every state $i$,
\begin{align}
g_i=
\sum_{a}\pi(a\mid i)
\min_{p\in\mathcal U_{ia}}p^\top g, \quad
g_i+h_i=
r_i^\pi
+
\sum_a\pi(a\mid i)
\min_{p\in F_{ia}(g)}p^\top h,
\label{fv:eq:fixed-bias}
\end{align}
where
$r_i^\pi:=\sum_a\pi(a\mid i)r_{ia}$, and $F_{ia}(g):=
\arg\min_{p\in\mathcal U_{ia}}p^\top g$. Then $g=g^\pi$. Moreover, for every state-action pair with $\pi(a\mid i)>0$, choose $q_{ia}^\star\in\arg\min_{p\in F_{ia}(g)}p^\top h$, and choose arbitrary feasible rows for zero-probability actions. Then for every initial state $i$,
\begin{equation}
\inf_{\tau\in\mathcal Q_H}
\liminf_{N\to\infty}J_N(i;\pi,\tau)
=
\inf_{\tau\in\mathcal Q_H}
\limsup_{N\to\infty}J_N(i;\pi,\tau)
=
g_i^\pi,
\end{equation}
and the single stationary full selector $q^\star$
attains both infima simultaneously at every initial state.
\end{proposition}
The same finite pair certifies the complete robust gain vector and one stationary selector attaining it. We will later characterize exactly when such a pair exists.

\section{Exact stationary conditions for solvability}
\label{m:stationary}

Section~\ref{m:formulation} establishes that a finite Bellman solution guarantees the optimal robust gain and optimal controller. We now determine when such a solution exists.

A finite bias requires both attainment of the stationary gain and uniform control of the associated transient reward corrections. We first make these requirements precise for policy evaluation. We then characterize optimal-control solvability through a stationary saddle in the gain-restricted game and a one-sided bound on its canonical biases. For a finite stochastic matrix $P$, let $P^\infty=\lim_{N\to\infty}N^{-1}\sum_{t=0}^{N-1}P^t$ and $Z_P=(I-P+P^\infty)^{-1}$. Both exist without irreducibility or aperiodicity \cite{puterman2014markov}. Fix $\pi\in\Pi_S$ and define its effective row set by
$\U_i^\pi=\{\sum_{a\in A(i)}\pi(a\mid i)p_a:p_a\in\U_{ia}\text{ for every }a\in A(i)\}$.
Independent actionwise minimization gives $(T^\pi v)_i=r_i^\pi+\min_{p\in\U_i^\pi}p^\top v$. Thus, nature's fixed-policy problem is a compact-action MDP with reward $r^\pi$. For an effective selector $q\in\prod_i\U_i^\pi$, write $(P_q)_{i\cdot}=q_i^\top$.

\begin{theorem}[Fixed-policy solvability, adapted from {\cite[Theorem~1]{schweitzer1985undiscounted}}]
\label{m:fixed-existence}\label{m:fixed-policy-existence}
Fix $\pi\in\Pi_S$. For effective selectors $q\in\prod_i\U_i^\pi$, define $ \eta(q):=P_q^\infty r^\pi,$ $ \gamma_i^\pi:=\inf_q\eta_i(q),$ $\mathcal Q_*^\pi:=\{q:\eta(q)=\gamma^\pi\},$ and $ w(q):=Z_{P_q}(r^\pi-\eta(q)).$ 
Then \eqref{fv:eq:fixed-bias} has a finite solution if and only if
\begin{equation}
 \mathcal Q_*^\pi\ne\varnothing
\quad\text{and}\quad
 \exists B<\infty \text{ such that } w(q)\ge-B\one\ \text{ for any } q\in\mathcal Q_*^\pi.
 \label{m:fixed-criterion}
\end{equation}
Every finite solution $(g,h)$ satisfies $g=g^\pi=\gamma^\pi$.
\end{theorem}
The two conditions identify separate requirements. The set $\mathcal Q_*^\pi$ contains the stationary kernels that attain the robust gain simultaneously from every state. For such a kernel, $w(q)$ is its canonical transient reward correction, with normalization $P_q^\infty w(q)=0$. The uniform lower bound prevents these normalized corrections from becoming arbitrarily negative across gain-attaining kernels. Average-gain attainment and finite transient corrections are therefore distinct parts of Bellman solvability.

For optimal control, the candidate gain supplies the appropriate reward centering. Every gain-active controller-nature pair satisfies $P^{\pi,q}g=g$, so replacing $r_{ia}$ by $c_{ia}=r_{ia}-g_i$ subtracts $g$ from its original gain. The bias equation is therefore a zero-gain problem on the gain-restricted action and row sets. A stationary saddle secures this zero gain, and a bound on transient corrections determines whether the saddle value has a finite Bellman representation.  The active controller policies and nature selectors are
$\Pi_g:=\prod_iA_g(i)$, $\mathcal Q_g:=\prod_i\prod_{a\in A_g(i)}F_{ia}(g)$, and $\mathcal Q_g^\pi:=\prod_iF_{i,\pi(i)}(g)$.
A full plan $q\in\mathcal Q_g$ specifies rows for every active action, while a reply $q\in\mathcal Q_g^\pi$ specifies only the rows used by $\pi$. Accordingly, $P_{\pi,q}$ has row $q_{i,\pi(i)}^\top$ for a full plan and row $q_i^\top$ for a restricted reply. Set $c_i^\pi=c_{i,\pi(i)}$, $\eta^{\pi,q}=P_{\pi,q}^\infty c^\pi$, and $w^{\pi,q}=Z_{P_{\pi,q}}c^\pi$ when $\eta^{\pi,q}=0$. Here $\eta^{\pi,q}$ uses the centered rewards $c$, whereas $\eta(q)$ in Theorem~\ref{m:fixed-existence} uses the original rewards $r^\pi$.

\begin{theorem}[Optimal-control solvability]
\label{m:optimal-existence}
Fix $g\in\mathbb R^S$ with $g=\widehat{T}g$. There is a finite $h$ with $K_gh=h$ if and only if there exist $\bar{\pi}\in\Pi_g$, a full plan $\bar{q}\in\mathcal Q_g$, and $B<\infty$ such that
\begin{align*}
 &(i)  \eta^{\bar{\pi},q}\ge0,~  \text{for every }q\in\mathcal Q_g^{\bar{\pi}};\quad (ii)  \eta^{\pi,\bar{q}}\le0,~ \text{for every }\pi\in\Pi_g;\\
 &(iii) w^{\bar{\pi},q}\ge-B\one, ~\text{for every }q\in\mathcal Q_g^{\bar{\pi}}
                    \text{ with }\eta^{\bar{\pi},q}=0.
\end{align*}
Equivalently, finite solvability is equivalent to $\sup_{N\ge0}\|K_g^N0\|_\infty<\infty$.
Another equivalent condition is the existence of finite $\ell,u$ with $\ell\le K_g\ell$ and $K_gu\le u$.
\end{theorem}
\begin{proof}[Proof sketch]
Write $F:=K_g$; it is monotone and additively homogeneous, hence nonexpansive in the supremum norm. The proof runs through the cycle solvability $\Rightarrow$ (i)--(iii) $\Rightarrow$ finite barriers $\Rightarrow$ bounded orbit $\Rightarrow$ solvability; Appendix~\ref{sol:finite-bias-section} gives the details.

\emph{Solvability $\Rightarrow$ (i)--(iii).} Let $Fh=h$, and select maximizing actions $\bar\pi$ and minimizing rows $\bar q$ as in Theorem~\ref{m:verification}. Because these choices are optimal in the two nested optimizations defining $(Fh)_i=h_i$,
\[
 c^{\bar\pi}+(P_{\bar\pi,q}-I)h\ge0\quad(q\in\mathcal Q_g^{\bar\pi}),\qquad
 c^{\pi}+(P_{\pi,\bar q}-I)h\le0\quad(\pi\in\Pi_g).
\]
Write $\Gamma_{\pi,q}:=P_{\pi,q}^\infty$. Each projector is nonnegative and satisfies $\Gamma_{\pi,q}(P_{\pi,q}-I)=0$ \cite{puterman2014markov}. Multiplying each inequality by its own projector therefore gives
\[
 \eta^{\bar\pi,q}=\Gamma_{\bar\pi,q}c^{\bar\pi}\ge0,
 \qquad
 \eta^{\pi,\bar q}=\Gamma_{\pi,\bar q}c^\pi\le0,
\]
which are (i) and (ii). For (iii), fix a reply $q$ with $\eta^{\bar\pi,q}=0$, and let $P:=P_{\bar\pi,q}$ and $d:=c^{\bar\pi}+(P-I)h\ge0$. Since $P^\infty d=\eta^{\bar\pi,q}=0$, the residual $d$ vanishes on recurrent states, so $Z_Pd=\sum_{t\ge0}P^td\ge0$. The identity $Z_P(I-P)=I-P^\infty$ then gives
\[
 w^{\bar\pi,q}=Z_P\bigl((I-P)h+d\bigr)=h-P^\infty h+Z_Pd\ge h-P^\infty h\ge-\spn(h)\one,
\]
because each row of $P^\infty$ is a probability vector. Thus (iii) holds with $B=\spn(h)$ for every zero-gain reply.

\emph{(i)--(iii) $\Rightarrow$ finite barriers.} Conditions (i) and (ii) give $\eta^{\bar\pi,\bar q}=0$. Fixing $\bar\pi$ leaves nature a one-action minimizing problem with rows $F_{i,\bar\pi(i)}(g)$ and rewards $c^{\bar\pi}$. By (i) its optimal centered gain is zero, and $\bar q$ attains it. Condition (iii) is exactly the canonical-bias bound of Theorem~\ref{m:fixed-existence} for this problem, which therefore has a finite bias $\ell$ with $\ell_i=\min_{p\in F_{i,\bar\pi(i)}(g)}\{c_{i,\bar\pi(i)}+p^\top\ell\}\le(F\ell)_i$; the inequality holds because $F$ maximizes over a set containing $\bar\pi(i)$. Fixing the full plan $\bar q$ instead leaves the controller a finite-action MDP whose optimal centered gain is zero by (ii) and is attained by $\bar\pi$. The canonical biases of its finitely many deterministic policies have a common finite bound, so it has a finite bias $u$ \cite[Chapter~9]{puterman2014markov} satisfying
\[
 (Fu)_i\le\max_{a\in A_g(i)}\{c_{ia}+\bar q_{ia}^\top u\}=u_i,
\]
because the full plan supplies a feasible row $\bar q_{ia}$ for every active action.

\emph{Finite barriers $\Rightarrow$ bounded orbit $\Rightarrow$ solvability.} Additive homogeneity permits constant shifts of $\ell$ and $u$ so that $\ell\le0\le u$; monotonicity and the barrier inequalities then give $\ell\le F^N\ell\le F^N0\le F^Nu\le u$ for every $N\ge0$. A monotone, additively homogeneous map with a bounded orbit has a fixed point (Lemma~\ref{sol:topical-orbits}, the zero-growth case of \citealp[Lemma~3]{gaubert2004perron}), so $Fh=h$ for a finite $h$.
\end{proof}

\textbf{Concrete sufficient regimes for solvability.}
The following conditions ensure finite Bellman solvability while allowing different recurrent classes to have different gains (beyond constant gains). They provide concrete model classes covered by the preceding criterion and the planning result later.  

\begin{proposition}[Sufficient regimes for Bellman solvability] \label{m:sufficient-regimes}
Under the standing assumptions, each of the following conditions ensures a finite solution $(g^\star,h)$ to \eqref{m:bellman}:

(A) \texttt{Polytopic ambiguity:}
Every row set $\mathcal U_{ia}$ is a polytope.

(B) \texttt{Uniform support gap:}
There exists $\delta>0$ such that every feasible row satisfies
$p_j=0$ or $p_j\ge\delta$, for all $i,a$, $p\in\mathcal U_{ia}$,
and $j\in S$.

(C) \texttt{Continuous stationary projections:}
For every $\pi\in\Pi_D$, the map $P\mapsto P^\infty$ is
continuous on the compact induced-kernel family
$\mathcal P^\pi=\{P^{\pi,q}:q\in\mathcal Q_S\}$.
\end{proposition}
Polytopic rows make $T$ piecewise affine, so Kohlberg's invariant half-line theorem gives a finite pair \cite{kohlberg1980invariant}. Under (C), continuity of the Ces\`aro projections gives a uniform bound on the discounted resolvent remainder over each compact induced-kernel family. Taking nature's infimum and the maximum over the finite deterministic-policy class gives bounded vector centering, from which Lemma~\ref{sol:bounded-centering} extracts a finite Bellman pair. Condition (B) implies (C), as proved in Appendix~\ref{app:supporting}. These regimes accommodate a nonconstant optimal gain. In that case, $\epsilon\operatorname{sp}(V_\epsilon)\to \operatorname{sp}(g^\star)>0$, so the uncentered discounted span grows as the discount vanishes. The existence analysis in Appendix \ref{sol:structural-section}  controls the finite remainder $V_\epsilon-g^\star/\epsilon$ in these regimes. This vector centering isolates transient rewards while preserving the distinct long-run gains of recurrent classes. It extends the role played by bounded discounted span in constant-gain robust theory \cite[Theorem~4]{wang2025bellman}.

\begin{remark}[Solvability under concrete distributional uncertainty sets]
Our results also imply solvability under several distributional uncertainty sets that are extensively studied in robust RL with additional unichain assumption \cite{roch2025reduction,roch2026finite}, defined as $\mathcal{U}_{ia}:=\{q\in\Delta_S: D(q\|p^0_{ia})\leq \rho_{ia}\}$, where $D$ is some probability divergence, like total variation or KL-divergence, $p^0$ is some nominal kernel, and $\rho_{ia}$ is the radius. Since total-variation balls are polytopes, Proposition~\ref{m:sufficient-regimes}(A) applies at every radius. For forward KL balls $D_{\mathrm{KL}}(p\|p^0_{ia})\le\rho_{ia}$, condition~(B) applies when $\rho_{ia}<\min_{j:p^0_{ia,j}>0}-\log(1-p^0_{ia,j})$ for every row, with $-\log0:=+\infty$. Support-restricted reverse KL balls satisfy condition~(B) at every finite radius. Table~\ref{tab:solvability-sets} in Appendix~\ref{sol:divergence-balls} adds $\chi^2$, Hellinger, and $\ell_q$ balls and finite scenario sets, and two KL examples show that transition removal can preserve solvability, while solvability can fail at the first forward-KL support-loss threshold.
\end{remark}

\begin{remark}\label{m:rem-obstructions}
Although the robust gain and an all-state optimal stationary controller exist under our standing assumptions (Section~\ref{m:foundations}), the vector Bellman system need not admit a finite solution. This is different from non-robust cases: general compact ambiguity can obstruct stationary worst-gain attainment \cite[Proposition~3.2]{grand2023beyond} or a uniform lower bound on canonical biases \cite[Example~2, after reversing rewards]{schweitzer1985undiscounted}. In contrast, nominal MDPs with finite state and action spaces always admit solutions to their multichain gain-bias system \cite{puterman2014markov,schweitzer1985undiscounted}. This nominal solvability also follows from Proposition~\ref{m:sufficient-regimes}(A), since every nominal transition row corresponds to a singleton ambiguity set, thus a polytope. Our vector Bellman system thus provides a finite certificate with an exact solvability criterion of the multichain robust average-reward MDPs, generalizing the non-robust Bellman theory.
\end{remark}

\section{Planning from Finite Bellman Certificates}
\label{plan:section}
A finite Bellman certificate supplies asymptotic comparison points for undiscounted planning with a state-dependent gain. Algorithm~\ref{alg:robust_halpern} is inspired by the approximately shifted Halpern update of \cite[Algorithm~1]{zurek2025faster}, with the robust operator $T$. Phase~I computes $x_N=T^N0$ and uses it both to estimate the gain as $\widehat g_N=x_N/N$ and to initialize Phase~II. The second phase anchors at $x_N$ and shifts Bellman updates by this estimate. We assume $T$ can be exactly computed and applied.

\begin{algorithm}[!htb]
\caption{Approximately Shifted Robust Halpern Iteration}
\label{alg:robust_halpern}\label{alg:robust-halpern}\label{ashi:algorithm}
\begin{algorithmic}[1]
\State \textbf{Input:}  $N\ge 1$ and an exact robust Bellman oracle $T$.
\State \textbf{Initialize:} $x_0\gets 0$.
\Statex \texttt{Phase I: Gain estimation and warm start.}
\For{$j=0,1,\ldots,N-1$}
    \State $x_{j+1}\gets T x_j$.
\EndFor
\State $\widehat g_N\gets x_N/N$ and $z_0\gets x_N$.

\Statex \texttt{Phase II: Approximately shifted Halpern iteration.}
\For{$t=0,1,\ldots,N-1$}
    \State
    $\displaystyle
    z_{t+1}\gets
    \frac{2}{t+3}z_0+
    \frac{t+1}{t+3}\bigl(Tz_t-\widehat g_N\bigr)$.
\EndFor
\State $Z_N\gets z_N$.
\State Evaluate the action scores at $Z_N$ to obtain $TZ_N$ and choose
\Statex \hspace{\algorithmicindent}
$\displaystyle \pi_N(i)\in\arg\max_{a\in A(i)}
 \{r_{ia}+\min_{p\in\U_{ia}}p^\top Z_N\}$ for every $i\in S$.
\State \textbf{Output:} $\widehat g_N$, $Z_N$, $TZ_N-Z_N$, and $\pi_N$.
\end{algorithmic}
\end{algorithm}

\begin{theorem}
\label{ashi:main}
Under the standing finite-state, finite-action, compact post-action $(s,a)$-rectangular model, suppose \eqref{m:bellman} has a finite solution. Then Algorithm~\ref{ashi:algorithm} satisfies, in supremum norm,
\begin{equation}
 \widehat g_N\longrightarrow g^\star,\qquad
 TZ_N-Z_N\longrightarrow g^\star,\qquad
 Z_N/N\longrightarrow g^\star.
 \label{ashi:limits}
\end{equation}
Moreover, there is a finite $N_0$ such that every output $\pi_N$ at any $N\ge N_0$ satisfies $g^{\pi_N}=g^\star$, and each such controller is optimal from every initial state against history-dependent nature.
\end{theorem}

\begin{proof}[Proof sketch]
Fix a finite solution $(g,h)$; Theorem~\ref{m:verification} gives $g=g^\star$, and Appendix~\ref{plan:appendix} gives the full proof. Lemma~\ref{fv:lem:tangent} shows that $h+tg$ is an approximate affine trajectory of $T$: its affine defect $\omega_h(t):=\|T(h+tg)-h-(t+1)g\|_\infty$ tends to zero. Under general compact ambiguity the defect need not vanish at any finite $t$, because rows outside the gain face can remain preferable when their bias advantage offsets a small increase in continuation gain (Proposition~\ref{rev:slow-gain-estimation}). The proof therefore keeps $\omega_h$ explicit.

\emph{1. Phase~I tracks the trajectory.} By nonexpansiveness, each update adds at most one defect to the distance from $h+tg$, so $\|x_N-h-Ng\|_\infty\le A_N:=\|h\|_\infty+\sum_{j<N}\omega_h(j)=o(N)$. In particular, $\widehat g_N\to g^\star$.

\emph{2. Phase~II controls displacement and direction.}
Set $\varepsilon_N:=(A_N+\|h\|_\infty)/N$, which bounds the gain-estimation error and tends to zero.
For the nonexpansive shifted map $S_N(v):=Tv-\widehat g_N$, use $h+Ng$ as a comparison point.
Its distance from the anchor $x_N$ is at most $A_N$, and its residual is at most $e_N:=\omega_h(N)+\varepsilon_N$.
The approximate-fixed-point estimate in Lemma~\ref{ashi:approx-halpern} gives
\begin{align*}
 \|TZ_N-Z_N-g\|_\infty
 &\le\frac{8A_N}{N+3}+\frac43e_N+\varepsilon_N,\\
 \|Z_N-h-Ng\|_\infty&\le A_N+\frac N3e_N.
\end{align*}
The first bound tends to zero. Dividing the second by $N$ proves $Z_N/N\to g$, since $A_N/N\to0$ and $e_N\to0$.

\emph{3. From vector limits to optimal controllers.} The direction limit eventually excludes every action with $m_{ia}(g^\star)<g_i^\star$ from the greedy rule, since there are finitely many actions. At every remaining action, every feasible nature row satisfies $p^\top g^\star\ge g_i^\star$, so the expected gain $\E\,g^\star(S_t)$ never decreases under any nature strategy, and telescoping the reward inequality, with budget $N$ fixed while the payoff horizon tends to infinity, bounds the statewise robust gain loss: $0\le g^\star-g^{\pi_N}\le\|TZ_N-Z_N-g^\star\|_\infty\one$. There are finitely many deterministic controllers, so vanishing loss makes every sufficiently late greedy controller exactly optimal from all states. This gain-active identification is the additional step required when the optimal gain is a vector.
\end{proof}
The defect also sets the rate: if $\omega_h(t)\le C(1+t)^{-\alpha}$, the gain-estimation and displacement errors are $O(N^{-\alpha})$, $O(\log N/N)$, or $O(1/N)$ for $0<\alpha<1$, $\alpha=1$, or $\alpha>1$, respectively. For polytopes, the defect of every finite solution vanishes after finite time (Theorem~\ref{sol:polytope}), so these two errors are $O(1/N)$. No fixed algebraic rate holds across solvable compact models (Proposition~\ref{rev:slow-gain-estimation}), and Corollary~\ref{ashi:budget} gives a computable budget when a vanishing defect bound and the required instance constants are known.

\begin{remark}[Relation to discounted and constant-gain planning]
    Discounted reductions \cite{wang2023robust,grand2023reducing,roch2025reduction,yang2026robust,grand2023beyond} solve a robust discounted MDP with a large \cite{roch2025reduction,grand2023reducing} or increasing \cite{wang2023model,grand2023beyond} discount factor. Direct anchored and robust Bellman methods \cite{wang2023robust,wang2023model,roch2026finite,xu2025efficient,xu2025finite} assume a constant gain, where $T(x+c\mathbf 1)=Tx+c\mathbf 1$ turns $Th=h+\rho\mathbf 1$ into a fixed-point equation modulo constant vectors. A state-dependent gain keeps its drift after this scalar normalization, so our planner instead controls the vector displacement $TZ_N-Z_N\to g^\star$ through an estimated shift.
\end{remark} 

\textbf{Numerical verification.} We further numerically evaluate Algorithm~\ref{alg:robust_halpern} against a nominal counterpart using the same anchored updates \cite{zurek2025faster} on three multichain models: boundary leakage, periodic recurrent classes, and a safe-risky decision. Figure~\ref{fig:main-paired-experiments} in Appendix~\ref{sec:numerical} supports the theoretical convergence of Algorithm~\ref{alg:robust_halpern} and shows that its extracted controllers attain the optimal robust gain in all three models. The nominal method converges for its reference model but selects controllers with strictly smaller worst-case gains. The models have three to six states and exact values of $g^\star$ and of every controller's robust gain, so they test the statements of Theorem~\ref{ashi:main} rather than performance at scale. Appendix~\ref{num:no-solution} also runs it on Example~\ref{sol:ex-kl-fail}, which has no finite solution.

\section{Conclusion}
\label{m:discussion}
We developed a vector Bellman certificate for multichain robust average-reward MDPs. Our coupled gain-bias system preserves gain priority for both players and certifies stationary saddle strategies against history-dependent opponents from every initial state. Its solvability criterion connects stationary gain attainment to uniform control of canonical transient corrections, with concrete sufficient conditions permitting distinct recurrent-class gains. This certificate also provides the asymptotic structure needed for direct undiscounted planning. Under finite solvability, our approximately shifted Halpern iteration recovers the vector gain through its estimates and Bellman displacements, and its greedy controllers are eventually average-optimal. \textbf{Limitations.} The results assume post-action $(s,a)$-rectangular ambiguity; with coupled ($s$-rectangular) sets, average-optimal policies may fail to exist or may need history dependence \cite[Propositions~3.8--3.9]{grand2023beyond}, so no stationary certificate of this form exists in general. The budget $N_0$ is existential unless a vanishing defect bound and the required instance constants are known (Corollary~\ref{ashi:budget}), and the planner assumes an exact robust Bellman oracle.

%% file: rebuttal-app-polished.tex
\appendix

\noindent\textbf{Appendix organization.} Appendices~\ref{app:related-work} and~\ref{sec:numerical} discuss related work and give the details of the experiments. Appendix~\ref{app:preliminaries} recalls the strategic results used from prior work, and Appendix~\ref{app:finite-horizon-foundations} proves the finite-horizon extension and records its stationary-attainment consequences. Appendix~\ref{fv:sec:finite} verifies the vector Bellman system, and Appendix~\ref{sol:finite-bias-section} characterizes its solvability. Appendix~\ref{plan:appendix} uses the certificate's asymptotically affine trajectory to prove convergence of the planning algorithm. Appendix~\ref{app:supporting-results} collects sufficient conditions for solvability and their applications to distributional uncertainty sets.

\section{Related work}
\label{app:related-work}
The paper connects robust dynamic programming, multichain average-reward theory, and the analysis of nonexpansive operators. The central distinction is between the existence of a long-run strategic value and its representation by a finite gain-bias pair. We organize the literature around this distinction and its consequences for planning.

\subsection{Transition ambiguity, rectangularity, and dynamic consistency}

\textbf{Robust and distributionally robust MDPs.}
Robust MDPs optimize a policy against a family of plausible transition models. Classical robust dynamic programming identifies rectangularity assumptions under which local worst-case transition choices yield a recursive description of the value \cite{iyengar2005robust,nilim2004robustness,wiesemann2013robust}. Distributionally robust formulations also model uncertainty through distributions over model parameters and allow statistical information to enter the ambiguity description \cite{xu2010distributionally}. The precise uncertainty object and the information available to nature matter: an uncertainty set over transition rows, a distribution over kernels, and a single unknown kernel chosen at the outset need not define the same control problem. Here nature selects a transition row after observing the current state and action, with independent admissibility constraints across state-action pairs. This post-action $(s,a)$-rectangular structure determines the order of optimization in our Bellman operator.

\textbf{Dynamic consistency and risk-averse control.}
Rectangularity has a broader interpretation in sequential decision theory. \cite{epstein2003recursive} connect rectangular sets of priors to recursive multiple-priors preferences and dynamic consistency. In Markov control, \cite{ruszczynski2010risk} develop dynamic programming with Markov risk measures for finite-horizon and discounted problems. The connection to our operator is visible at the one-step level: the lower expectation $\min_{p\in\mathcal U_{sa}}p^\top v$ equals the negative of the upper expectation $\max_{p\in\mathcal U_{sa}}p^\top(-v)$. Thus worst-case reward evaluation has a natural risk-averse interpretation after reversing signs. These connections explain the recursive structure of robust evaluation; the existence of a finite bias for an undiscounted, state-dependent gain requires additional long-run analysis.

\textbf{Coupled uncertainty and the timing of nature's choices.}
Rectangularity can also be imposed on a representation of uncertainty rather than directly on individual transition rows. \cite{goyal2023beyond} study factor-matrix uncertainty that couples transitions across states while retaining tractability under rectangularity in the factor representation. \cite{li2025rectangularity} examine the relationship between static and game formulations of distributionally robust MDPs and the role of rectangularity in their equivalence and duality. For average reward, \cite{wang2026non} study nonrectangular uncertainty with a stationary kernel chosen by nature and history-dependent controller policies. These models address different forms of dependence and information. Our results concern the stagewise post-action model; this specification is essential to the gain-restricted row sets and the controller-nature comparisons used below.

\subsection{Multichain average reward and stochastic-game values}

\textbf{Classical multichain optimality equations.}
Average-reward MDPs model continuing decisions without an exogenous discount factor \cite{puterman2014markov}. The distinction between long-run gain and transient bias is classical. \cite{blackwell1962discrete} establish the connection between discounted optimization near discount factor one and undiscounted optimality in finite models. \cite{denardo1968multichain} develop multichain Markov renewal programming, while \cite{schweitzer1978functional} study the solution structure of undiscounted functional equations, including the degrees of freedom associated with optimal recurrent behavior. In general multichain models the gain can depend on the initial state. Gain-bias equations then compare continuation gains first and rewards and biases among gain-optimal actions second. Under a fixed policy and kernel, recurrent-class rewards and absorption probabilities determine the gain vector. Transition ambiguity adds a second optimization that can change those classes and probabilities. Our vector Bellman system uses the same gain-first principle while requiring compatible comparisons for both players.

\textbf{Finite stochastic games and mean payoff.}
The long-run value problem also belongs to the theory of zero-sum stochastic games. For finite state and action spaces, \cite{bewley1976asymptotic} establish a common asymptotic limit of normalized finite-horizon and discounted values, and \cite{mertens1981stochastic} prove existence of the uniform value. The latter is a strategic guarantee across sufficiently long horizons; it does not in general imply that both players have stationary optimal strategies. Perfect-information games have additional structure, with classical stationary-strategy results for time-average payoff \cite{liggett1969stochastic}. Their multichain gain-bias structure and policy iteration are developed further by \cite{akian2012policy}. These are direct precedents for robust models with finitely many effective nature actions.

For polytopic $(s,a)$-rectangular ambiguity, minimizing a linear continuation value can be reduced to the finitely many extreme rows. \cite{chatterjee2023solving} exploit the resulting connection to finite turn-based stochastic games to obtain long-run robust planning and complexity results. Consequently, direct average-reward planning beyond scalar-gain assumptions already has precedents in the polytopic case. Our analysis also permits compact curved row sets, where a finite reduction need not be available and finite Bellman solvability must be examined separately.

\textbf{Compact-action games and asymptotic values.}
Finiteness of the state space alone does not replace assumptions on the action sets in general stochastic-game value theory. \cite{bolte2015definable} use definability and additional structural conditions to establish uniform values for classes of compact-action games, including definable perfect-information games. Definability includes semialgebraic examples and supplies regularity beyond compactness. A complementary operator approach relates convergence of normalized discounted and finite-horizon values through Tauberian theorems \cite{ziliotto2016tauberian}. We use the latter connection after establishing discounted convergence for the present robust model. The relevant conclusion is convergence of normalized values for arbitrary compact row sets, without imposing definability. It should be distinguished from both a finite gain-bias representation and stronger uniform-strategy conclusions for general compact-action games.

\textbf{Robust average-reward theory.}
Robust average-reward Bellman equations and algorithms have been developed under unichain assumptions on the policy-kernel family \cite{wang2023robust,wang2024robust}. More recent theory gives conditions for scalar robust Bellman solvability under broader communication and information structures, including one-sided weak communication \cite{wang2025bellman}. Such assumptions can allow multiple recurrent classes for some choices while still producing an optimal gain independent of the initial state. Thus the relevant distinction for this paper is state dependence of the optimal gain, rather than simply whether any multichain transition matrix is admissible.

For compact $(s,a)$-rectangular sets, \cite{grand2023beyond} establish deterministic stationary controller optimality, strong duality, equivalence of the principal average-payoff conventions, and normalized discounted convergence without a unichain assumption. They also show that nature's stationary worst case need not be attained. These results provide the strategic foundation for our study. We use these results directly, recording their all-state consequences, and extend normalized finite-horizon convergence beyond definable ambiguity using the Tauberian connection. The subsequent questions are whether the value admits a finite vector gain-bias certificate, which stationary choices such a certificate supports, and how to recover the gain and a controller policy by direct iteration.

\subsection{Finite biases and nonlinear operators}

\textbf{Compact-action Bellman solvability.}
The fixed-policy criterion comes from classical compact-action MDP theory. After reversing the reward sign, \cite[Theorem~1]{schweitzer1985undiscounted} characterizes finite one-player gain-bias solvability through stationary gain attainment and a uniform lower bound on canonically normalized biases. The bound controls transient corrections as transition kernels vary; pointwise finiteness for each kernel is insufficient. We apply this criterion to the nature problem induced by a fixed controller policy. For robust optimal control, the proof combines a lower barrier obtained from this one-player result with an upper barrier supplied by a full nature plan. The two-player formulation makes explicit which policies, replies, and common bias bounds must be compatible.

\textbf{Nonlinear spectral theory.}
Undiscounted Bellman operators are monotone, additively homogeneous with respect to scalar constants, and nonexpansive in the sup norm. Nonlinear Perron-Frobenius theory provides bounded-orbit criteria for additive eigenvectors and fixed points \cite{gaubert2004perron}.

\subsection{Direct planning and anchored iterations}

\textbf{Invariant half-lines and approximate affine behavior.}
For finite-action perfect-information games and polytopic robust MDPs, the Bellman operator is piecewise affine. \cite{kohlberg1980invariant} show that a nonexpansive piecewise-linear map admits an invariant half-line, which describes an eventual affine trajectory with a fixed growth direction. This structure underlies multichain game algorithms \cite{akian2012policy}. For general compact row sets, a finite Bellman solution can instead yield an asymptotically affine trajectory: its one-step defect tends to zero, but the trajectory need not become exactly invariant after a finite threshold. The distinction matters for planning because an argument based on eventual exact equality does not automatically cover curved ambiguity sets. Our convergence analysis tracks this vanishing defect explicitly.

\textbf{Halpern iteration and multichain planning.}
Anchored fixed-point methods originate in the iteration of \cite{halpern1967fixed}. Quantitative analyses include sharp residual bounds for nonexpansive maps in Hilbert spaces \cite{lieder2021convergence}. Those results explain the general anchoring mechanism, while Bellman planning requires estimates in the operator's relevant norm and may involve a nonzero growth direction rather than an ordinary fixed point. Recent nominal MDP work develops anchored and shifted methods that address these issues \cite{lee2025optimal,zurek2025faster}. In particular, our update follows the approximately shifted iteration of \cite{zurek2025faster}. The robust analysis controls the additional affine defect and connects gain and displacement estimates to controller-policy extraction. The iteration's origin, its robust convergence argument, and the structural conditions ensuring a finite bias are therefore separate parts of the comparison.

\textbf{Computing robust Bellman updates.}
Iteration complexity and the cost of each inner minimization are complementary questions. \cite{ho2018fast} develop efficient exact Bellman updates for $\ell_1$ ambiguity, and \cite{ho2021partial} combine efficient updates with partial policy iteration for discounted robust MDPs. Such methods can serve as computational components when the row sets in our model have the corresponding structure. They do not by themselves provide an undiscounted multichain convergence argument. Conversely, an operator-level convergence result for arbitrary compact row sets does not imply a uniformly efficient implementation of every inner optimization problem; its computational use depends on how the ambiguity sets are represented.

\subsection{Average-reward and robust reinforcement learning}

\textbf{Nominal average-reward learning.}
Average-reward reinforcement learning includes differential value estimation, temporal-difference methods, and $Q$-learning \cite{mahadevan1996average,mahadevan_average-reward_1996,abounadi2001learning,zhang2021average,ma2021average}. Regret-based work such as \cite{jaksch2010near} studies exploration in unknown communicating MDPs using a diameter parameter. Other analyses develop convergence under weak communication and finite-sample guarantees governed by bias span or related structural quantities \cite{wan2021learning,wan2022convergence,wan2024convergence,zurek2024span}. These results show why communication, recurrent structure, and bias size are central to both learning and planning. Our setting isolates the deterministic robust planning and solvability questions with access to the Bellman operator; a statistical learning guarantee would additionally need to control how transition-estimation errors affect those quantities.

\textbf{Robust average-reward learning.}
Existing methods include relative-value TD and $Q$-learning, policy optimization, discounted reductions, anchored procedures, and stochastic approximation \cite{wang2023model,sun2024policy,roch2025reduction,chen2025sample,xu2025finite,xu2025efficient,roch2026finite,yang2026robust,wang2025provable}. Their assumptions vary, with many guarantees using unichain, irreducibility, or uniform ergodicity conditions that yield a state-independent robust gain. These conditions provide ways to control long-run sensitivity and transient behavior. Our finite-bias characterization addresses the structural question that arises when the gain can vary across initial states: which robust models still admit a finite certificate on which a direct planning analysis can be based?

\textbf{Discounted and finite-horizon robust learning.}
A large literature studies statistical estimation of worst-case values under discounted or finite-horizon criteria. Early sample-based approaches include robust temporal-difference learning, approximate dynamic programming, and linear policy evaluation \cite{lim2013reinforcement,tamar2014scaling,badrinath2021robust,wang2021online}. Subsequent tabular analyses cover model-based estimation and plug-in planning, as well as model-free procedures, for several ambiguity families and data-access models \cite{yang2022toward,panaganti2022sample,xu2023improved,shi2023curious,zhou2021finite,wang2023finite,liang2023single,liu2022distributionally,wang2023sample,wang2024modelfree,wang2023bring,kumar2023efficient,derman2021twice}. Online robust learning additionally treats exploration \cite{wang2021online,lu2024distributionally,ghosh2026orvit,he2025sample}. Related model-free methods use robust $Q$-learning, multilevel Monte Carlo, or variance reduction \cite{liu2022distributionally,wang2023finite,wang2024modelfree,wang2024sample,ghosh2026scaling}, while policy-based approaches analyze robust policy-gradient and actor-critic methods \cite{wang2022policy,kumar2023policy}.

\textbf{Offline learning and function approximation.}
Robust offline methods combine ambiguity-aware pessimism with tabular or fitted value iteration and structured function approximation \cite{zhou2021finite,panaganti2022robust,shi2022distributionally,blanchet2023double,liu2024minimax,panaganti2024model,wang2024sample,wang2024unified}. Other work considers distributional robustness with function approximation or simulator access in large or continuous state spaces \cite{ramesh2024distributionally}. These literatures address estimation, coverage, and approximation errors. Their discounted contraction or finite-horizon recursion controls the propagation of those errors. In the undiscounted multichain problem studied here, finite-bias solvability and state-dependent growth must first be understood to obtain an analogous foundation for algorithmic analysis.

\section{Numerical experiments}
\label{sec:numerical}

We examine convergence of the computed gain and Bellman displacement,
and robust average-reward performance of the extracted controller.
The three models are designed stress tests in which midpoint nominal
parameters favor an action with a smaller worst-case gain. The first two
have curved, nonpolytopic ambiguity and isolate boundary leakage and
periodicity, respectively. The third is a polytopic safe-risky decision.
Each model has a finite gain-bias certificate, specified below, so the
solvability hypothesis of Theorem~\ref{ashi:main} is satisfied.
Appendix~\ref{num:no-solution} runs Algorithm~\ref{alg:robust_halpern} on Example~\ref{sol:ex-kl-fail}, where this hypothesis fails.

\subsection{Methods, budgets, and evaluation}
\label{num:paired-protocol}

\textbf{Methods and initialization.}
Let $T$ denote the robust Bellman operator and let $T_0$ replace each
uncertain row set by the nominal reference row specified below. We compare
Algorithm~\ref{alg:robust_halpern}, using $T$, with a nominal ablation using $T_0$. The latter is
the approximately shifted multichain iteration of \cite{zurek2025faster}, applied to the fixed nominal model.
For either planning operator $\mathcal T\in\{T,T_0\}$, the run starts at zero,
forms $x_N:=\mathcal T^N0$, sets $\widehat g_N:=x_N/N$ and $z_0:=x_N$, and
performs
\[
 z_{t+1}:=\frac{2z_0+(t+1)(\mathcal Tz_t-\widehat g_N)}{t+3},
 \qquad t=0,\ldots,N-1.
\]
The output vector is $v_N:=z_N$ (denoted $Z_N$ in the planning theorem). One final evaluation of $\mathcal T v_N$
extracts a controller greedy for that method's own operator. Both methods
therefore use $2N+1$ planning evaluations. Tied action values are resolved
in favor of $c$ in Models~I and II and the risky action in Model~III.
The nominal controller is evaluated under the robust model with its
selected action fixed.

\textbf{Budget convention.}
We evaluate every integer $N=1,\ldots,16384$, including both parities.
Each point represents a complete run from zero with that budget; the anchor
and estimated gain depend on $N$. Thus $2N+1$ measures the planning budget
of a run. The implementation shares common first-phase iterates and batches
the independent second-phase calculations, producing the same outputs as
individually initialized runs. Batched execution time is recorded separately
from the per-run oracle count.
The nominal method uses one additional robust operator evaluation to measure
its robust displacement. This external diagnostic is excluded from its
nominal planning budget and is never fed into its updates or policy selection.

\textbf{Convergence errors.}
All methods are evaluated against the exact robust gain $g^\star$. Define
\begin{equation}
 E_g(N):=\|\widehat g_N-g^\star\|_\infty,\qquad
 E_d(N):=\|Tv_N-v_N-g^\star\|_\infty,\qquad
 E_N:=\max\{E_g(N),E_d(N)\}.
 \label{num:paired-errors}
\end{equation}
The main figure uses $E_N$ to show both convergence targets in one panel per
model. The top and middle rows of Figure~\ref{fig:appendix-experiments}
display the two components separately. For Algorithm~\ref{alg:robust_halpern}, Theorem~\ref{ashi:main} gives $E_N\to0$ under
finite Bellman solvability. This is an asymptotic statement and does not
require the errors to decrease at every finite budget. For the nominal method, this robust-target error includes disagreement
between the nominal and robust objectives. We separately measure each
solver's own-objective error:
\[
 E_N^{\mathrm{own}}:=
 \max\bigl\{\|\widehat g_N-g_{\mathcal T}^\star\|_\infty,
 \|\mathcal T v_N-v_N-g_{\mathcal T}^\star\|_\infty\bigr\},
\]
where $g_{\mathcal T}^\star:=g^\star$ for $\mathcal T=T$ and
$g_{\mathcal T}^\star:=g_0^\star$ for $\mathcal T=T_0$, with $g_0^\star$
the optimal gain of the nominal reference model.
The bottom row of Figure~\ref{fig:appendix-experiments} reports this diagnostic.
The data retain full output vectors and statewise robust policy gains.

\textbf{Robust evaluation of the output controller.}
Each model has one decision state $x$. For the controller $\pi_N$ selected by
each method, we plot its actual worst-case average reward
\begin{equation}
 G_N:=g_x^{\pi_N}=\inf_{q\in\mathcal Q_S}\eta_x^{\pi_N,q},
 \qquad G^\star:=g_x^\star.
 \label{num:paired-policy-gain}
\end{equation}
This is evaluated analytically from the selected controller, rather than
estimated by its value iterate or by a finite simulated rollout.
In the first two models, state $y$ has the same gain as $x$, and all remaining
states have policy-independent gains. The latter property also holds in the
third model. Consequently, in every experiment,
\[
 \|g^\star-g^{\pi_N}\|_\infty=G^\star-G_N.
\]
Reaching the horizontal reference $G^\star$ therefore verifies optimality
from every initial state for these models. Policy gains are plotted on linear axes, including negative values.

\textbf{Exact row minimization and implementation.}
Linear minimization over the convex hull of a curve has the same value as
minimization over the generating curve. In Model~I, the row objective is
$v_x+u(v_y-v_x)+u^3(v_z-v_x)$; we compare both endpoints and any real stationary
point in $[0,1/2]$. In Model~II it is quadratic, so both endpoints and any
interior minimizing vertex suffice. The alternative-action row objectives and the Model~III objective are affine,
so their minima occur at interval endpoints. All calculations use these analytic
minimizers in double precision, with no discretization of the ambiguity set
and no Monte Carlo policy evaluation. The computations are deterministic.
The nominal reference rows are specified as part of each model.

\subsection{Model I: boundary leakage}
\label{num:paired-boundary}

The ordered state space is $(x,y,z,w,H)$. State $x$ has actions $c$
(curved) and $t$ (risky); every other state has one action. The rewards are
$r(x,c):=0$, $r(x,t):=10$, $r(y):=-1$, $r(z):=1$, $r(w):=-2$, and $r(H):=5$.
The uncertain rows are
\[
\begin{aligned}
 \U_{x,c}&:=\operatorname{co}\{(1-u-u^3,u,u^3,0,0):0\le u\le1/2\},\\
 \U_{x,t}&:=\{(0,0,0,1-\rho,\rho):0\le\rho\le1\}.
\end{aligned}
\]
State $y$ returns to $x$, and $z,w,H$ are absorbing. These remaining rows
are known. The nominal parameters are the interval midpoints
$u_0:=1/4$ and $\rho_0:=1/2$.

The exact robust gain and a gain-face bias are
\[
 g^\star=(0,0,1,-2,5),\qquad h:=(0,-1,0,0,0).
\]
To verify the certificate, first compare continuation gains. The curved
continuation gain is $u^3$, minimized at $u=0$, whereas the
risky continuation gain is $-2+7\rho$, minimized at $\rho=0$.
Thus only $c$ is gain-active at $x$, and its minimizing gain face consists
of the self-loop at $x$. The restricted reward-bias value there is
$r(x,c)+h_x=0=g_x^\star+h_x$. At $y$, the bias equation is
$g_y^\star+h_y=-1=-1+h_x$; at each absorbing state it reduces to
$g_i^\star=r_i$. Hence both equations hold at every state, and
Theorem~\ref{m:verification} identifies $g^\star$ as the robust gain of the
curved controller and as the optimal robust gain. The risky controller has gain
$(-2,-2,1,-2,5)$. Hence $G_N\in\{0,-2\}$.

Under the nominal model, action $c$ eventually reaches $z$ and has gain
$1$ at $x,y$. Action $t$ has nominal gain
$(1-\rho_0)(-2)+5\rho_0=3/2$, so it is strictly preferable nominally.
The nominal optimal gain is $g_0^\star=(3/2,3/2,1,-2,5)$, and its optimal controller
has robust loss $2$. The immediate reward $10$ affects transient decisions
but contributes zero to the long-run average after absorption.

The curved rows create long transients near $u=0$: every fixed $u>0$
eventually leads to $z$, while $u=0$ leaves $x$ recurrent. This change in
recurrent structure makes the model a test of planning near the boundary
of the uncertainty set.

\subsection{Model II: periodic basins}
\label{num:paired-periodic}

The ordered states are $(x,y,h_0,h_1,m,\ell)$. State $x$ has actions $c$
(curved) and $f$ (risky), both with reward zero. Every other state has one
action. The rewards at $(y,h_0,h_1,m,\ell)$ are $(-0.3,1,-1,-1,1)$.
For the curved action,
\[
 \U_{x,c}:=\operatorname{co}\{p(u):0\le u\le1\},
\]
where, in the stated coordinate order,
\[
 p(u):=(0,0,0.55+0.10u,0,0.25-0.05u^2,0.20-0.10u+0.05u^2).
\]
The risky row set is
\[
 \U_{x,f}:=\{(0,0,0,0,1-\rho,\rho):0\le\rho\le1\}.
\]
State $y$ goes to $x$, and $h_0,h_1$ alternate deterministically. States
$m,\ell$ are absorbing. The nominal parameters are $u_0:=1/2$ and $\rho_0:=1/2$.

The exact robust certificate is
\[
 g^\star=(-3/40,-3/40,0,0,-1,1),\qquad
 h:=(27/40,9/20,1,0,0,0).
\]
For the gain equation, expand the contribution of the two absorbing states:
$p(u)^\top g^\star=-0.05-0.10u+0.10u^2
=-3/40+(u-1/2)^2/10$. Its unique minimizing row is therefore
$p(1/2)$, with continuation gain $-3/40$. The risky continuation gain is
$2\rho-1$, whose minimum is $-1$, so only $c$ is gain-active at $x$.
For the bias equation at $x$, the minimizing gain face gives
$p(1/2)^\top h=3/5=g_x^\star+h_x$. At $y$,
$g_y^\star+h_y=3/8=-3/10+h_x$; on the two-cycle,
$h_{h_0}=1+h_{h_1}$ and $h_{h_1}=-1+h_{h_0}$; and at $m,\ell$ the
gain equals the absorbing reward. These identities verify the complete
certificate. The curved and risky controllers consequently have
robust gains $g^\star$ and $(-1,-1,0,0,-1,1)$, respectively.
Consequently, $G_N\in\{-3/40,-1\}$.

Nominally, the risky action has gain $2\rho_0-1=0$, which exceeds the curved
action's $-3/40$. The nominal optimal gain is $g_0^\star=(0,0,0,0,-1,1)$, and the
nominal optimal controller has robust loss $37/40$.
The deterministic two-cycle has alternating rewards and zero average
gain. This model tests convergence with periodic recurrent dynamics and
different gains across recurrent classes.

\subsection{Model III: a safe-risky decision}
\label{num:paired-tradeoff}

This model gives a direct safe-risky comparison. There are three
states $(x,L,H)$. States $L,H$ are absorbing with rewards $0,1$, respectively.
At $x$, both available actions have reward zero. The safe action has known
row $(0,0.6,0.4)$, while the risky action has row set
\[
 \U_{x,\mathrm{risky}}:=\{(0,1-p,p):0.1\le p\le0.9\}.
\]
The nominal risky row uses $p_0:=0.5$; the safe row is unchanged.
The robust gain is $g^\star=(0.4,0,1)$, achieved by the safe action, with bias
$h:=(-0.4,0,0)$. Indeed, the safe continuation gain is $0.4$, whereas the
minimum risky continuation gain is $0.1$. Only the safe action is
gain-active at $x$, and its reward-bias value is
$0=g_x^\star+h_x$. The absorbing-state equations reduce to their fixed
rewards. These vectors therefore satisfy the gain-first system; the
row sets are also polytopes. The risky controller has robust gain
$(0.1,0,1)$, so its robust loss is $0.3$.
In contrast, nominal planning prefers the risky action and has optimal
nominal gain $g_0^\star=(0.5,0,1)$. Thus the nominal objective changes the selected
controller in this model. This polytopic example complements the two
nonpolytopic constructions by making the robustness distinction explicit.

\subsection{Results and interpretation}
\label{num:paired-results}

\begin{figure}[!htb]
\centering\includegraphics[width=\linewidth]{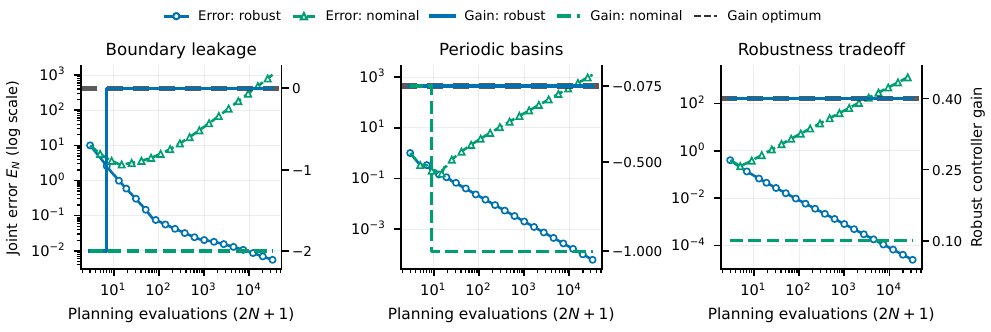}\vspace{-0.2in}
  \caption{Robust and nominal planning against iterations. Left logarithmic axis: robust-target joint error  $E_N=\max\{\|\widehat g_N-g^\star\|_\infty, \|TZ_N-Z_N-g^\star\|_\infty\}$ (lower is better). Right linear axis: exact robust gain for a specific initial state $g^{\pi_N}_x$   of the extracted controller (higher is better); the black dashed line marks the optimum obtained by solving Bellman equations. }
\label{fig:main-paired-experiments}
\end{figure}

Figure~\ref{fig:main-paired-experiments} pairs each model's joint error with
robust controller gain, and the top and middle rows of Figure~\ref{fig:appendix-experiments}
separate the gain and displacement errors. Algorithm~\ref{alg:robust_halpern}'s errors are consistent with
the predicted convergence to zero, and its controllers attain the optimal
robust gain in all three models. In Model~I it selects $t$ at $N=1,2$ and
$c$ at every tested $N\ge3$. It selects an optimal controller at every tested
budget in Models~II and III.

The nominal baseline selects the risky controller at every tested budget in
Models~I and III. In Model~II it selects $c$ at $N=1,2,3$ and $f$ at every
tested $N\ge4$. Its final robust gains are therefore $-2$, $-1$, and $0.1$,
compared with the optimal gains $0$, $-3/40$, and $0.4$.
These gaps follow from the different objectives: the nominal controller
optimizes its reference model, whose favorable outcomes are less reliable
under worst-case transition evaluation.

At $N=16384$, Algorithm~\ref{alg:robust_halpern} has
joint errors approximately $5.51\times10^{-3}$, $6.10\times10^{-5}$, and
$2.44\times10^{-5}$ in Models~I-III, respectively.
The bottom row of Figure~\ref{fig:appendix-experiments} shows both methods approaching zero
error for their own objectives. The nominal method's robust displacement
can nevertheless grow because its output follows the nominal gain direction.
Its robust policy losses are $2$, $37/40$, and $0.3$, respectively. Thus
convergence for the reference model and robust controller performance are
distinct properties.

\begin{figure}[!htb]
  \centering
  \includegraphics[width=\linewidth]{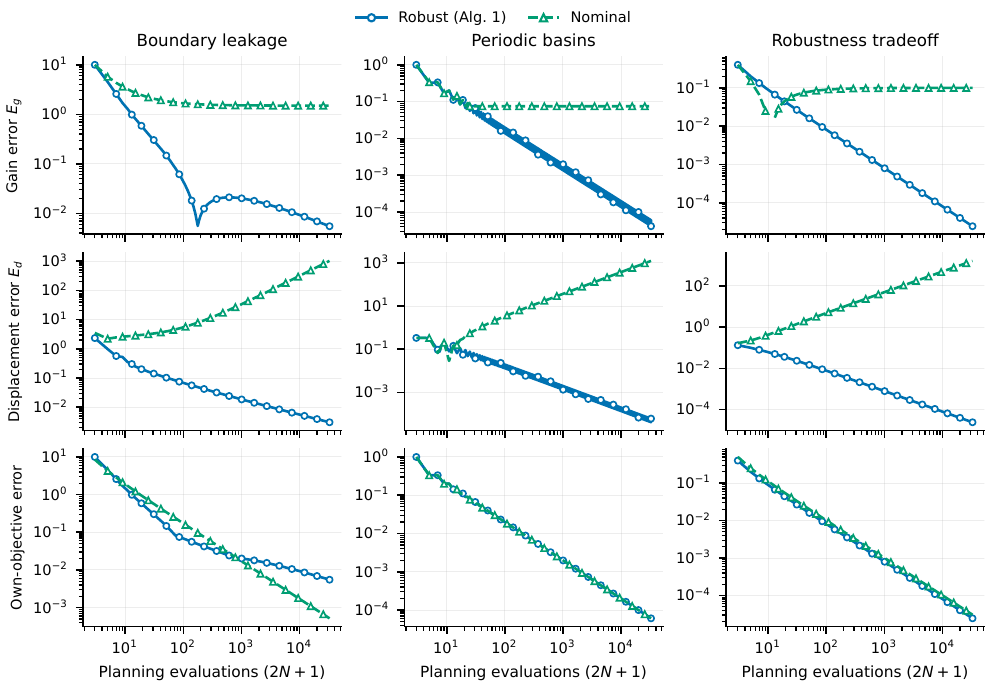}
  \caption{Additional convergence diagnostics at every integer budget.
  Columns correspond to Models~I-III. Top: robust gain-estimation error
  $E_g(N)$. Middle: robust displacement error $E_d(N)$.
  Bottom: joint error $E_N^{\mathrm{own}}$ against each method's own
  operator and optimal gain. Both methods converge for their own planning
  objectives, while the nominal controllers retain the robust policy losses
  shown in Figure~\ref{fig:main-paired-experiments}.}
  \label{fig:appendix-experiments}
\end{figure}

\subsection{Behavior without a finite solution}
\label{num:no-solution}

Theorem~\ref{ashi:main} assumes that \eqref{m:bellman} has a finite solution. To see how Algorithm~\ref{alg:robust_halpern} behaves when this assumption fails, we run it on Example~\ref{sol:ex-kl-fail} in Appendix~\ref{sol:divergence-balls}. There $g^\star=0$ and every stationary row has gain zero, but the canonical biases of the rows $p_t$ in its proof equal $-1/t$ at $x$, so no finite solution exists. As a control, we use the same model with the radius reduced to $\log(3/2)-0.1$. This ball keeps the full support, condition~(B) of Proposition~\ref{m:sufficient-regimes} holds, a finite solution exists, and again $g^\star=0$.

\textbf{Computation.} Since
$D_{\mathrm{KL}}(p\|(1/3,1/3,1/3))=\log3-H(p)$, where $H$ is the Shannon
entropy, the ball of radius $\rho$ is
$\{p\in\Delta(S):H(p)\ge\log3-\rho\}$. We enforce this entropy constraint
without discretizing the ball. For a vector $v$, first check whether the
uniform row on its minimizing coordinates is feasible. If so, the minimum
is $\min_jv_j$. Otherwise the entropy constraint is active, and the
Lagrange multiplier equations give
$p_j=e^{-\beta v_j}/\sum_ke^{-\beta v_k}$, where $\beta>0$ is chosen to
make the entropy constraint an equality. We compute $\beta$ by a
safeguarded Newton method. Because the instance sits exactly at the
solvability threshold, all calculations use 40-digit arithmetic (mpmath).
Algorithm~\ref{alg:robust_halpern} is run as in
Appendix~\ref{num:paired-protocol}, for $N=2^k$ with $k=1,\ldots,16$.
The table below reports the errors \eqref{num:paired-errors} and $T^N0(x)$.
Its last row gives least-squares slopes of the logarithms of the errors
against $\log N$ over $N\in[2^{10},2^{16}]$; for $T^N0(x)$ it uses
$\log|T^N0(x)|$.

\begin{center}
\small
\begin{tabular}{@{}rcccccc@{}}
\toprule
 & \multicolumn{3}{c}{No finite solution: radius $\log(3/2)$} & \multicolumn{3}{c}{Control: radius $\log(3/2)-0.1$}\\
\cmidrule(lr){2-4}\cmidrule(l){5-7}
$N$ & $E_g(N)$ & $E_d(N)$ & $T^N0(x)$ & $E_g(N)$ & $E_d(N)$ & $T^N0(x)$\\
\midrule
$2^{8}$ & $5.60\times10^{-2}$ & $3.89\times10^{-2}$ & $-13.4$ & $1.55\times10^{-2}$ & $1.23\times10^{-2}$ & $-2.979$ \\
$2^{10}$ & $3.19\times10^{-2}$ & $2.11\times10^{-2}$ & $-31.7$ & $3.89\times10^{-3}$ & $3.65\times10^{-3}$ & $-2.981$ \\
$2^{12}$ & $1.77\times10^{-2}$ & $1.13\times10^{-2}$ & $-71.6$ & $9.72\times10^{-4}$ & $9.57\times10^{-4}$ & $-2.981$ \\
$2^{14}$ & $9.66\times10^{-3}$ & $6.06\times10^{-3}$ & $-157.2$ & $2.43\times10^{-4}$ & $2.42\times10^{-4}$ & $-2.981$ \\
$2^{16}$ & $5.19\times10^{-3}$ & $3.22\times10^{-3}$ & $-339.2$ & $6.07\times10^{-5}$ & $6.07\times10^{-5}$ & $-2.981$ \\
\midrule
slope & $-0.44$ & $-0.45$ & $0.57$ & $-1.00$ & $-0.99$ & -- \\
\bottomrule
\end{tabular}
\end{center}

\textbf{Divergence of finite-horizon values.}
In the control, $T^N0(x)$ settles at $-2.981$ and $Z_N$ stays bounded.
Without a finite solution, $T^N0(x)$ and $Z_N(x)$ keep decreasing in our
runs, to $-339.2$ and $-294.2$ at $N=2^{16}$.

The divergence of $T^N0(x)$ can also be proved directly. For sufficiently
large $N$, take the feasible row $p_t$ from the proof of
Example~\ref{sol:ex-kl-fail} with $t=N^{-1/2}$, and use this row at every
visit to $x$. Write $a_k:=\Pr_x(S_k=x)$ and $b_k:=\Pr_x(S_k=y)$ under this
fixed row, and let $u:=1/2+t$. Since absorption at $z$ has probability
$1/N$ on each visit to $x$, the probability of remaining in $\{x,y\}$ at
time $k$ is $a_k+b_k\ge1-k/N$. Also $b_0=0$ and
$b_{k+1}=u a_k$. For large $N$ we have $u\le1$, so
\[
 2\sum_{k=0}^{N-2}a_k
 \ge\sum_{k=0}^{N-2}(a_k+b_k)
 \ge\sum_{k=0}^{N-2}(1-k/N)
 \ge\frac{N-1}{2}.
\]
The expected $N$-stage reward of this admissible stationary row is therefore
\[
 \frac12\sum_{k=0}^{N-1}a_k-\sum_{k=0}^{N-1}b_k
 =-t\sum_{k=0}^{N-2}a_k+\frac12a_{N-1}
 \le\frac12-\frac{N-1}{4\sqrt N}.
\]
Nature's finite-horizon minimum is no larger, proving $T^N0(x)\to-\infty$.
This argument concerns $T^N0(x)$; the decrease of $Z_N(x)$ above is an
experimental observation. For every tested $N\ge2^{10}$, evaluation of
the same row gives the sharper numerical bound
$T^N0(x)\le-0.47\sqrt N$. Over $N\in[2^{10},2^{16}]$, the ratio
$|T^N0(x)|/\sqrt{N\log N}$ stays between $0.37$ and $0.40$, while
$|T^N0(x)|/\sqrt N$ grows from $0.99$ to $1.33$.

\textbf{The gain estimate converges, and the displacement error decreases in the runs.} Theorem~\ref{m:uniform-value} gives $\widehat g_N\to g^\star$ for every compact model, with or without a finite solution. The displacement error also decreases in our runs, although Theorem~\ref{ashi:main} does not cover this case. The fitted slopes are $-0.44$ and $-0.45$, against $-1.00$ and $-0.99$ in the control. Over $N\in[2^{10},2^{16}]$, $E_g(N)\sqrt{N/\log N}$ stays between $0.38$ and $0.40$ and $E_d(N)\sqrt{N/\log N}$ between $0.24$ and $0.26$, so both errors are consistent with order $\sqrt{\log N/N}$.

\textbf{Policy selection at the tested budgets.} Example~\ref{sol:ex-kl-fail}
has one action per state. To test the controller, we add a decision state
$s$ with two actions of reward zero: $\mathsf{enter}$ moves to $x$, and
$\mathsf{exit}$ moves to a new absorbing state $a$ with reward $-c<0$.
Then $g^\star_s=\max\{g^\star_x,-c\}=0$, so $\mathsf{enter}$ is optimal and
$\mathsf{exit}$ has robust loss $c$. The new states do not change the
iterates at $x,y,z$. At $a$, Phase~I gives $x_N(a)=-cN$ and
$\widehat g_N(a)=-c$; the shifted update leaves this coordinate unchanged,
so $Z_N(a)=-cN$ exactly. Consequently, when the two action scores differ,
$\pi_N(s)=\mathsf{enter}$ if and only if $Z_N(x)/N>-c$. This is a one-sided
threshold relative to the optimal gain $g^\star_x=0$.
On the tested grid, for $c=0.05$ the controller exits for
$4\le N\le128$ and enters at $N=2$ and at every $N\ge256$.
For $c=0.02$, it exits for every $N\le2048$ and enters at every
$N\ge4096$. No ties occurred. The smaller gap thus requires a larger
budget in these runs.
\section{Preliminaries and cited strategic results}
\label{app:preliminaries}
\label{sec:model}

We first record the one-step inequalities that connect Bellman operators to trajectory rewards, and then recall the finite-chain identities used for gains and biases. The remaining argument transfers the strategic results of \cite{grand2023beyond} to a common choice of strategies for all initial states. The main point in this transfer is that a single stationary selector must approximate every coordinate of the value vector; separate statewise infima do not by themselves give such a selector.

\textbf{Notation.} We use the model and notation of Section~\ref{m:model}. Write $n:=|S|$, $R:=\max_{i,a}|r_{ia}|$, and $\spn(x):=\max_i x_i-\min_i x_i$. Vector inequalities and extrema are coordinatewise. For $x\in\R^S$ and a full stationary selector $q\in\mathcal Q_S$, define
\[
 m_{ia}(x):=\min_{p\in\U_{ia}}p^\top x,
 \qquad
 (H_qx)_i:=\max_{a\in A(i)}\{r_{ia}+q_{ia}^\top x\}.
\]
Thus $m_{ia}(x)$ is the worst continuation value of action $a$ at state $i$, and $H_q$ is the Bellman operator of the nominal MDP obtained when nature commits to $q$. Compactness of $\U_{ia}$ attains each row minimum. Since every feasible row is a probability vector, each $m_{ia}$ is $1$-Lipschitz in the supremum norm, positively homogeneous ($m_{ia}(\lambda x)=\lambda m_{ia}(x)$ for $\lambda\ge0$), and additively homogeneous ($m_{ia}(x+c\one)=m_{ia}(x)+c$ for $c\in\R$).

\textbf{Applicability of the cited results.} The strategic results below are taken from \cite{grand2023beyond}; theorem numbers refer to arXiv:2312.03618v3, January~14, 2025. Their assumptions are finite state and controller action sets and nonempty compact $(s,a)$-rectangular ambiguity, without convexity. Their successor-dependent rewards include ours. A positive affine rescaling puts any nonconstant bounded reward into $[0,1]$ and preserves the optimal strategies and all the stated value identities; constant rewards are immediate. State-dependent action sets can be represented by duplicating actions. Their perfect-information formulation places nature's choice after the controller's action (Appendix~A.1). Allowing the controller to observe past rows enlarges its strategy class, while the stationary securing strategies remain admissible under our observation convention. The finite-horizon extension is proved in Appendix~\ref{app:finite-horizon-foundations}.

\subsection{Rectangular dynamic programming and one-step comparisons}
\label{app:one-step}

\textbf{Information structure.} Fix a strategy pair. Let $\mathcal F_t$ be the $\sigma$-field generated by the history
\[
 (S_0,A_0,p_0,\ldots,S_{t-1},A_{t-1},p_{t-1},S_t)
\]
and by the randomization both players realized before stage $t$; the lotteries of stage $t$ itself are excluded. Let $\mathcal G_t$ additionally contain the action $A_t$, nature's row $p_t\in\U_{S_tA_t}$, and the stage-$t$ lotteries that produced them, but not $S_{t+1}$. Then $\mathcal F_t\subseteq\mathcal G_t\subseteq\mathcal F_{t+1}$, and the transition rule reads
\begin{equation}
 \E[f(S_{t+1})\mid\mathcal G_t]=p_t^\top f
 \qquad(f\in\R^S).
 \label{app:eq:transition-rule}
\end{equation}
The fresh random seeds used to implement a stage's lotteries are independent of the past, while the action and row distributions themselves may depend on the observed history. In particular, if the controller uses $\pi\in\Pi_S$, then $\Prob(A_t=a\mid\mathcal F_t)=\pi(a\mid S_t)$. The enlarged $\sigma$-fields serve only to condition expectations in the proofs; each player's strategy still uses only that player's permitted observations.

\textbf{One-step comparisons.} The order of play enters the analysis only through the following two inequalities. For every $v\in\R^S$ and every $t\ge0$,
\begin{align}
 \E\bigl[r_{S_tA_t}+v(S_{t+1})\bigm|\mathcal F_t\bigr]
   &\ge (T^\pi v)(S_t)
   &&\text{if the controller uses }\pi\in\Pi_S,
 \label{app:eq:one-step-lower}\\
 \E\bigl[r_{S_tA_t}+v(S_{t+1})\bigm|\mathcal G_t\bigr]
   &=r_{S_tA_t}+q_{S_tA_t}^\top v\le (H_qv)(S_t)
   &&\text{if nature uses }q\in\mathcal Q_S.
 \label{app:eq:one-step-upper}
\end{align}
To prove \eqref{app:eq:one-step-lower}, we condition in two stages. Given $\mathcal G_t$, the transition rule \eqref{app:eq:transition-rule} and feasibility $p_t\in\U_{S_tA_t}$ give the conditional expectation $r_{S_tA_t}+p_t^\top v\ge r_{S_tA_t}+m_{S_tA_t}(v)$. Conditioning this bound on $\mathcal F_t$ and using $\Prob(A_t=a\mid\mathcal F_t)=\pi(a\mid S_t)$ gives $\sum_a\pi(a\mid S_t)\{r_{S_ta}+m_{S_ta}(v)\}=(T^\pi v)(S_t)$. The identity in \eqref{app:eq:one-step-upper} is \eqref{app:eq:transition-rule} with $p_t=q_{S_tA_t}$, and the inequality holds because $A_t\in A(S_t)$ is one of the actions in the maximum defining $H_q$. Both displays hold for arbitrary randomized history-dependent strategies of the opponent.

\begin{lemma}
\label{lem:basic}
The following statements hold for every $\pi\in\Pi_S$ and $0<\eps<1$.
\begin{enumerate}[label=(\roman*)]
\item The operators $T$, $T^\pi$, and $H_q$ ($q\in\mathcal Q_S$) are order preserving, additively homogeneous, and nonexpansive in $\|\cdot\|_\infty$. Explicitly, each such operator $F$ satisfies
\[
 x\le y\ \Longrightarrow\ Fx\le Fy,\qquad
 F(x+c\one)=F(x)+c\one,\qquad
 \|Fx-Fy\|_\infty\le\|x-y\|_\infty
\]
for all $x,y\in\mathbb R^S$ and $c\in\mathbb R$.
\item For $F\in\{T,T^\pi\}$, the map $x\mapsto F((1-\eps)x)$ is a $(1-\eps)$-contraction. Its unique fixed point is $V_\eps$ or $V_\eps^\pi$, respectively, and has norm at most $R/\eps$. Moreover, $V_\eps^\pi\le V_\eps$.
\item The vectors $V_\eps$ and $V_\eps^\pi$ are the discounted values against history-dependent opponents. In the control problem, both players have deterministic stationary discounted-optimal strategies that are optimal from every initial state; in particular, $V_\eps^{\pi_\eps}=V_\eps$ for some $\pi_\eps\in\Pi_D$. For fixed $\pi$, nature has such a strategy in $\mathcal Q_S$.
\item The $N$-horizon game has value $T^N0$: for every $i\in S$,
\[
 (T^N0)_i=\sup_{\sigma\in\Pi_H}\inf_{\tau\in\mathcal Q_H}NJ_N(i;\sigma,\tau)
         =\inf_{\tau\in\mathcal Q_H}\sup_{\sigma\in\Pi_H}NJ_N(i;\sigma,\tau).
\]
Likewise, $[(T^\pi)^N0]_i=\inf_{\tau\in\mathcal Q_H}NJ_N(i;\pi,\tau)$ and $[H_q^N0]_i=\sup_{\sigma\in\Pi_H}NJ_N(i;\sigma,q)$ for every $q\in\mathcal Q_S$. In particular, $NJ_N(i;\pi,\tau)\ge[(T^\pi)^N0]_i$ for every $\tau\in\mathcal Q_H$, and $NJ_N(i;\sigma,q)\le[H_q^N0]_i$ for every $\sigma\in\Pi_H$.
\end{enumerate}
\end{lemma}

\begin{proof}
These are the standard rectangular dynamic-programming properties \cite{iyengar2005robust,nilim2004robustness}. We give the verification to specify where compactness and the observation order enter.

\emph{Operator properties and fixed points.} A probability row preserves order and satisfies $p^\top(x+c\one)=p^\top x+c$. Taking minima, maxima, or the fixed convex combination in $T^\pi$ preserves these properties. If $d:=\|x-y\|_\infty$, then $y-d\one\le x\le y+d\one$; order preservation and additive homogeneity give $Fy-d\one\le Fx\le Fy+d\one$. This proves nonexpansiveness and hence (i). It also gives
\[
 \|F((1-\eps)x)-F((1-\eps)y)\|_\infty
 \le(1-\eps)\|x-y\|_\infty.
\]
The contraction mapping theorem gives a unique fixed point $V$. Since $\|F0\|_\infty\le R$, its equation implies $\|V\|_\infty\le(1-\eps)\|V\|_\infty+R$, and therefore $\|V\|_\infty\le R/\eps$. Iterating the two discounted maps from $0$, the inequality $T^\pi x\le Tx$ and monotonicity give the same ordering at every iterate. Passing to their limits proves $V_\eps^\pi\le V_\eps$.

\emph{Discounted verification.} Write $\beta:=1-\eps$ and let $V$ be the fixed point for $T$. Choose a maximizing action in $T(\beta V)$ at each state and, for every state-action pair, a row minimizing $p^\top V$. These choices exist because action sets are finite and row sets are compact. Under the resulting controller, \eqref{app:eq:one-step-lower} with continuation vector $\beta V$ gives $\E[r_{S_tA_t}+\beta V(S_{t+1})\mid\mathcal F_t]\ge V(S_t)$ against every nature strategy. Multiplying by $\beta^t$, taking expectations, and summing from $t=0$ to $N-1$ yields
\[
 \E_i\sum_{t=0}^{N-1}\beta^t r_{S_tA_t}
 \ge V_i-\beta^N\E_i V(S_N).
\]
The selected full nature selector gives the reverse inequality against every controller, because its row is minimizing for each possible sampled action. The terminal term vanishes as $N\to\infty$, since $V$ is bounded and $0<\beta<1$. Thus the two stationary strategies secure $V$ from every state. For a fixed randomized controller, use its equation $V=T^\pi(\beta V)$ and the same rowwise minima. This proves (iii) and identifies the fixed points in (ii) with the discounted values.

\emph{Finite horizons.} For the fixed-policy problem put $w_k:=(T^\pi)^k0$, so $w_0=0$ and $w_{k+1}=T^\pi w_k$. At stage $t<N$, \eqref{app:eq:one-step-lower} gives
\[
 \E[r_{S_tA_t}+w_{N-t-1}(S_{t+1})\mid\mathcal F_t]
 \ge w_{N-t}(S_t).
\]
Taking expectations and summing cancels the intermediate continuation values, leaving $NJ_N(i;\pi,\tau)\ge(w_N)_i$. Choosing minimizing rows for each remaining horizon attains equality. With $w_k:=H_q^k0$, \eqref{app:eq:one-step-upper} gives the reversed bound, attained by maximizing actions. Finally, with $w_k:=T^k0$, choose maximizing actions and a full tuple of minimizing rows for each remaining horizon. These two Markov strategies secure opposite inequalities at the same vector $w_N$, proving both game values and all assertions in (iv). All row optimizations are performed on the original compact sets, so convexity is not needed.
\end{proof}

\paragraph{Nature choosing before the controller.}
If nature first chooses an action-indexed tuple $(p_a)_{a\in A(i)}\in\prod_{a\in A(i)}\U_{ia}$, the one-step operator is unchanged:
\[
 \min_{(p_a)_a\in\prod_a\U_{ia}}\max_{a\in A(i)}\{r_{ia}+p_a^\top v\}=\max_{a\in A(i)}\Bigl\{r_{ia}+\min_{p\in\U_{ia}}p^\top v\Bigr\}=(Tv)_i.
\]
Indeed, every tuple gives a maximum at least as large as the right-hand side, and choosing a minimizing row for each action attains equality. The stationary nature selector in Theorem~\ref{m:verification} can therefore be implemented by committing to this full tuple at each state. A shared uncertainty choice that couples different actions changes the feasible set and lies outside this equivalence. For such $s$-rectangular sets, average-optimal policies may fail to exist or may need history dependence \cite[Propositions~3.8--3.9]{grand2023beyond}, so a stationary certificate of the form \eqref{m:bellman} cannot hold in general.

\textbf{Effective rows of a stationary controller.} For $\pi\in\Pi_S$, put $r_i^\pi:=\sum_a\pi(a\mid i)r_{ia}$ and
\begin{equation}
 \U_i^\pi:=\Bigl\{\sum_{a\in A(i)}\pi(a\mid i)p_a:\ p_a\in\U_{ia}\text{ for all }a\Bigr\},
 \qquad
 (T^\pi x)_i=r_i^\pi+\min_{\bar p\in\U_i^\pi}\bar p^\top x .
\label{eq:aggregate}
\end{equation}
The set $\U_i^\pi$ is compact because it is the image of the compact product $\prod_a\U_{ia}$ under a continuous map. The identity for $T^\pi$ holds because the minimum of $\sum_a\pi(a\mid i)p_a^\top x$ over tuples $(p_a)_a$ is attained by minimizing each positive-weight component separately. A full selector $q\in\mathcal Q_S$ induces the effective rows $\sum_a\pi(a\mid i)q_{ia}\in\U_i^\pi$. Conversely, every choice of effective rows is induced by some $q\in\mathcal Q_S$: choose a representing tuple at each state and assign arbitrary feasible rows to actions of probability zero. Stationary quantities of the fixed-policy problem, such as $\eta^{\pi,q}$ and canonical biases, therefore range over the same set of transition matrices whether nature is described by full selectors or by effective rows.

\subsection{Finite-chain facts used by the Bellman arguments}
\label{app:finite-chain}

We recall the following finite-chain identities and limit results from \cite[Section~2]{schweitzer1985undiscounted} and \cite[Chapters~8--9]{puterman2014markov}. They hold for periodic multichain matrices. For a finite stochastic matrix $P$ and reward vector $c$, define
\begin{equation}
 P^\infty:=\lim_{N\to\infty}\frac1N\sum_{t=0}^{N-1}P^t,
 \quad Z_P:=(I-P+P^\infty)^{-1},\quad
 \eta:=P^\infty c,\quad w:=Z_P(c-\eta).
\label{sol:gain-bias-def}
\end{equation}
Here $\eta$ is the gain vector of the chain with reward $c$, and $w$ is its canonical bias.

\begin{lemma}
\label{sol:markov-algebra}
The projector $P^\infty$ is stochastic, $PP^\infty=P^\infty P=(P^\infty)^2=P^\infty$, and $Z_P$ exists. The canonical bias $w$ is the unique solution of
\begin{equation}
 (I-P)w=c-P^\infty c,\qquad P^\infty w=0.
\label{sol:poisson-normalization}
\end{equation}
Moreover, $Z_P(I-P)=I-P^\infty$. If $d\ge0$ and $P^\infty d=0$, then $d$ vanishes on recurrent states and
\begin{equation}
 Z_Pd=\sum_{t=0}^\infty P^td\ge0.
\label{sol:transient-series}
\end{equation}
\end{lemma}

\begin{proof}
The projector identities, invertibility of $I-P+P^\infty$, and the normalized Poisson equation are the cited finite-chain facts. We verify the two consequences used later. First, $(I-P+P^\infty)P^\infty=P^\infty$ gives $Z_PP^\infty=P^\infty$, and hence
\[
 Z_P(I-P)=Z_P(I-P+P^\infty)-Z_PP^\infty=I-P^\infty.
\]
For the series representation, the key point is that $d$ is supported on transient states. Indeed, if $i$ is recurrent, its diagonal entry $P^\infty_{ii}$ is positive, so
$0=(P^\infty d)_i\ge P^\infty_{ii}d_i\ge0$ forces $d_i=0$.
Write $Q$ for the transient-to-transient block of $P$. Recurrent classes are closed, so $P^td$ vanishes on recurrent states and equals $Q^td_{\mathrm{tr}}$ on transient states. Since the spectral radius of $Q$ is less than one, $y:=\sum_{t\ge0}P^td$ converges and is nonnegative. Its partial sums satisfy $(I-P)\sum_{t=0}^{N-1}P^td=d-P^Nd$, with $P^Nd\to0$, and $P^\infty\sum_{t=0}^{N-1}P^td=0$. Passing to the limit gives $(I-P)y=d$ and $P^\infty y=0$. Therefore $(I-P+P^\infty)y=d$, which proves $y=Z_Pd$.
\end{proof}

\begin{lemma}
\label{fd:lem:chain-limits}
For each fixed finite chain, $\eps(I-(1-\eps)P)^{-1}c\to P^\infty c$ as $\eps\downarrow0$. Under a stationary pair $(\pi,q)$, $X_N$ converges almost surely to the invariant mean reward of the recurrent class eventually entered. From initial state $i$, the expectation of this limiting random variable is $\eta_i^{\pi,q}$. All four payoffs in \eqref{m:four-payoffs} equal this number.
\end{lemma}

The Abel limit and the almost-sure ergodic theorem are the cited finite-chain results. For randomized $\pi$, we must also relate the sampled rewards to the state reward $r^\pi$. Put $D_t:=r_{S_tA_t}-r^\pi_{S_t}$. Then $\E[D_t\mid\mathcal F_t]=0$, $D_t$ is $\mathcal F_{t+1}$-measurable, and $|D_t|\le2R$. The martingale strong law applies because $\sum_{t\ge0}\E[D_t^2]/(t+1)^2<\infty$, giving $N^{-1}\sum_{t<N}D_t\to0$ almost surely. Thus the sampled rewards have the same limit as the state rewards. The limiting recurrent class has the usual absorption probabilities, so the expectation of its invariant mean is $(P^{\pi,q})^\infty r^\pi=\eta^{\pi,q}$. Since $|X_N|\le R$, bounded convergence also gives $\E_iX_N\to\eta_i^{\pi,q}$, proving equality of all four payoffs.

\subsection{Strategic values and common stationary choices}
\label{app:cited-values}
\label{app:policy-evaluation}

The cited results are stated for one initial law at a time. The reduction to all initial states uses the following averaging observation twice. Let $\mu$ be the uniform law, $\mu_i:=1/n$. If $x\in\R^S$ satisfies
\begin{equation}
 x\ge0\quad\text{and}\quad\mu^\top x\le\nu/n,
 \qquad\text{then}\quad x\le\nu\one,
 \label{app:eq:averaging}
\end{equation}
because each coordinate of a nonnegative vector is at most the sum $n\mu^\top x$ of all coordinates. Thus a tolerance $\nu/n$ for the uniform initial law becomes a tolerance $\nu$ from every state.

\textbf{Fixed-policy values.} For $\pi\in\Pi_S$, \cite[Lemma~4.7]{grand2023beyond}, applied with initial law $e_i$, gives
\begin{equation}
 \eps V_\eps^\pi\longrightarrow g^\pi,
 \qquad g_i^\pi:=\inf_{q\in\mathcal Q_S}\eta_i^{\pi,q}.
 \label{gc:eq:policy-limit}
\end{equation}
Convergence is in supremum norm because $S$ is finite. The stationary reduction and the equivalence of all four payoff conventions follow from their Lemma~3.3 and Corollary~3.7, applied to the one-action model with row sets $\U_i^\pi$ from \eqref{eq:aggregate}:
\begin{equation}
 \inf_{\tau\in\mathcal Q_H}\Psi_i(\pi,\tau)
 =\inf_{q\in\mathcal Q_S}\Psi_i(\pi,q)=g_i^\pi,
 \qquad \Psi\in\{I^-,J^-,J^+,I^+\}.
 \label{app:eq:cited-policy-values}
\end{equation}
Here the reduction to one action must preserve nature's ability to react to the sampled controller action. At any pre-stage history ending at state $i$, let $K_a$ be nature's conditional distribution on $\U_{ia}$ after action $a$. Before sampling the controller action, sample a tuple $(p_a)_a$ with marginal distributions $K_a$, independently of the controller's fresh lottery, and use its $A_t$ component. This gives the same conditional law of $(A_t,p_t,S_{t+1})$ as the original post-action response. Conditional on the tuple, the state transition row is $\bar p=\sum_a\pi(a\mid i)p_a\in\U_i^\pi$. Forgetting the sampled action and conditioning future choices on the remaining history gives an admissible randomized nature strategy in the one-action model. Conversely, a row in $\U_i^\pi$ has a representing tuple by definition, which nature can implement after observing $A_t$. The reward difference $r_{S_tA_t}-r^\pi_{S_t}$ is a bounded martingale difference, as above: its average vanishes almost surely and has zero expectation at every horizon. Therefore this reduction preserves all four average-payoff criteria, including the order of expectation and limit.

\cite[Theorem~4.3]{grand2023beyond} also gives a common approximate stationary reply. Apply it with the uniform initial law $\mu$ and tolerance $\nu/n$. It supplies a fixed selector $q^{\pi,\nu}$ whose normalized discounted excess $\eps\mu^\top(V_\eps^{\pi,q^{\pi,\nu}}-V_\eps^\pi)$ is at most $\nu/n$ for every sufficiently small $\eps$, where $V_\eps^{\pi,q}$ denotes the discounted value of the stationary pair. Letting $\eps\downarrow0$ and using \eqref{gc:eq:policy-limit} for the second term and Lemma~\ref{fd:lem:chain-limits} for the first gives $\mu^\top(\eta^{\pi,q^{\pi,\nu}}-g^\pi)\le\nu/n$. Every coordinate of $\eta^{\pi,q^{\pi,\nu}}-g^\pi$ is nonnegative by the definition of $g^\pi$ as an infimum, so the averaging observation \eqref{app:eq:averaging} yields
\begin{equation}
 g^\pi\le\eta^{\pi,q^{\pi,\nu}}\le g^\pi+\nu\one.
 \label{gc:eq:stationary-approx}
\end{equation}

\textbf{Optimal values.} The normalized discounted limit and stationary controller optimality are \cite[Lemma~4.8 and Theorem~3.4]{grand2023beyond}. In our notation,
\[
 \eps V_\eps\longrightarrow g^\star,
 \qquad g_i^\star=\max_{\pi\in\Pi_D}g_i^\pi.
\]
A single controller attains every coordinate of this maximum. To see this, choose deterministic stationary discounted-optimal controllers along a sequence $\eps_k\downarrow0$ (Lemma~\ref{lem:basic}(iii)). Since $\Pi_D$ is finite, one controller $\pi^\star$ occurs infinitely often; along that subsequence $V_{\eps_k}^{\pi^\star}=V_{\eps_k}$, and multiplying by $\eps_k$ and passing to the limit gives $g^{\pi^\star}=g^\star$ at every state.

\textbf{Nominal values for a fixed selector.} Let $V_\eps^q$ be the unique fixed point of $x\mapsto H_q((1-\eps)x)$. The following facts concern this finite nominal MDP and follow from classical multichain MDP theory \cite[Chapters~8--9]{puterman2014markov}.

\begin{lemma}
\label{fd:lem:fixed-q}
Fix $q\in\mathcal Q_S$ and define $d_i(q):=\max_{\pi\in\Pi_D}\eta_i^{\pi,q}$, the optimal gain of the nominal MDP obtained by fixing $q$.
\begin{enumerate}[label=(\roman*)]
\item $\eps V_\eps^q\to d(q)$, and one policy $\pi_q\in\Pi_D$ satisfies $\eta^{\pi_q,q}=d(q)$ at every state.
\item For every payoff $\Psi\in\{I^-,J^-,J^+,I^+\}$ and every state, $\sup_{\sigma\in\Pi_H}\Psi_i(\sigma,q)=d_i(q)$. In particular, $\max_{\pi\in\Pi_D}\mu^\top\eta^{\pi,q}=\mu^\top d(q)$ for every $\mu\in\Delta(S)$.
\item $H_q^N0/N\to d(q)$. Consequently, for every $\delta>0$ there is $N_\delta(q)<\infty$ such that $J_N(i;\sigma,q)\le d_i(q)+\delta$ for all $i\in S$, $\sigma\in\Pi_H$, and $N\ge N_\delta(q)$.
\end{enumerate}
\end{lemma}

The consequence in (iii) combines the limit with Lemma~\ref{lem:basic}(iv): $NJ_N(i;\sigma,q)\le[H_q^N0]_i$ for every $\sigma$, and $[H_q^N0]_i/N\le d_i(q)+\delta$ for all $i$ once $N$ is large, because $S$ is finite.

\textbf{Stationary duality and a common nature approximation.} We invoke \cite[Theorem~3.5, equation~(3.6)]{grand2023beyond} directly. For any initial law $\mu$, it gives the second equality in
\begin{equation}
 \mu^\top g^\star
 =\max_{\pi\in\Pi_D}\inf_{q\in\mathcal Q_S}\mu^\top\eta^{\pi,q}
 =\inf_{q\in\mathcal Q_S}\mu^\top d(q).
 \label{foundation:weighted-duality}
\end{equation}
The first equality is ours. For each $\pi$, the approximate replies \eqref{gc:eq:stationary-approx} give $\inf_q\mu^\top\eta^{\pi,q}=\mu^\top g^\pi$, and maximizing over $\pi$ gives $\mu^\top g^{\pi^\star}=\mu^\top g^\star$ by the common controller above. In the second expression, $\max_\pi\mu^\top\eta^{\pi,q}=\mu^\top d(q)$ is Lemma~\ref{fd:lem:fixed-q}(ii). Every selector satisfies $d(q)\ge\eta^{\pi^\star,q}\ge g^{\pi^\star}=g^\star$. Taking $\mu$ uniform and choosing $q_\alpha$ with excess $\mu^\top(d(q_\alpha)-g^\star)\le\alpha/n$ in \eqref{foundation:weighted-duality}, the averaging observation \eqref{app:eq:averaging} gives
\begin{equation}
 g^\star\le d(q_\alpha)\le g^\star+\alpha\one
 \qquad(\alpha>0).
 \label{fd:eq:dual-approx}
\end{equation}
This does not require the infimum in \eqref{foundation:weighted-duality} to be attained.

\subsection{Payoff conventions and strategy classes}
\label{app:strategic-duality}
The common values under stationary or history-dependent opponents and all four payoff conventions are established in \cite[Theorems~3.5--3.6 and Corollary~3.7]{grand2023beyond}. In the present all-state notation, for $\Pi_D\subseteq\mathcal C\subseteq\Pi_H$ and $\mathcal Q_S\subseteq\mathcal N\subseteq\mathcal Q_H$, they give
\[
 \underline v_i^\Psi(\mathcal C,\mathcal N)
 =\overline v_i^\Psi(\mathcal C,\mathcal N)=g_i^\star,
 \qquad \Psi\in\{I^-,J^-,J^+,I^+\}.
\]
The all-state form follows from the stationary strategies already constructed, which belong to every displayed pair of classes. By \eqref{app:eq:cited-policy-values}, the controller $\pi^\star$ gives
$\underline v_i^\Psi\ge\inf_{\tau\in\mathcal Q_H}\Psi_i(\pi^\star,\tau)=g_i^{\pi^\star}=g_i^\star$.
For every $\alpha>0$, Lemma~\ref{fd:lem:fixed-q}(ii) and \eqref{fd:eq:dual-approx} give
$\overline v_i^\Psi\le\sup_{\sigma\in\Pi_H}\Psi_i(\sigma,q_\alpha)=d_i(q_\alpha)\le g_i^\star+\alpha$.
Together with weak duality, these bounds yield
$g_i^\star\le\underline v_i^\Psi\le\overline v_i^\Psi\le g_i^\star+\alpha$; letting $\alpha\downarrow0$ proves the claim.

For later use, boundedness of $X_N$ and Fatou's lemma give the ordering
\begin{equation}
 I_i^-\le J_i^-\le J_i^+\le I_i^+.
 \label{fd:eq:payoff-order}
\end{equation}
Hence a controller guarantee for $I^-$ and a nature guarantee for $I^+$ imply the corresponding guarantees for all four conventions. This also proves the value formula for any initial law $\mu$. The expectation defining $I^-$ is linear in the initial law, so $\pi^\star$ secures at least $\mu^\top g^\star$; likewise, $q_\alpha$ bounds $I^+$ above by $\mu^\top g^\star+\alpha$. Applying \eqref{fd:eq:payoff-order} and then letting $\alpha\downarrow0$ gives the value $\mu^\top g^\star$ for every convention. No exchange of a liminf with a sum over initial states is required. The corresponding discounted-policy characterizations and the discount-to-average optimality threshold are given by \cite[Theorems~4.6 and~4.10]{grand2023beyond}.

\section{Finite-horizon consequences and stationary attainment}
\label{app:finite-horizon-foundations}
\label{app:optimal-control}

The strategic inputs are recalled in Appendix~\ref{app:cited-values}. This appendix adds the finite-horizon limits and the uniform guarantees in Theorems~\ref{m:policy-value}--\ref{m:uniform-value}. It extends \cite[Theorem~5.2]{grand2023beyond}, whose finite-horizon convergence result assumes definable ambiguity. The extension rests on one observation: the norm Tauberian theorem of \cite{ziliotto2016tauberian} needs only operator properties that hold for every compact row set, and these properties transfer the known discounted limits to finite horizons.

\subsection{Normalized finite-horizon values}

\begin{lemma}
\label{app:lem:tauberian}
Let $F\in\{T,T^\pi,H_q\}$ with $\pi\in\Pi_S$ and $q\in\mathcal Q_S$, and let $V_\eps^F$ denote the fixed point of $x\mapsto F((1-\eps)x)$. If $\eps V_\eps^F$ converges in $\|\cdot\|_\infty$ as $\eps\downarrow0$, then $F^N0/N$ converges to the same limit.
\end{lemma}

\begin{proof}
We apply the norm Tauberian theorem of \cite[Theorem~1.2]{ziliotto2016tauberian} on the Banach space $(\R^S,\|\cdot\|_\infty)$. The theorem requires two properties of $F$. The first, nonexpansiveness, is Lemma~\ref{lem:basic}(i). The second is a uniform bound on how rescaling changes the operator:
\begin{equation}
 \|\lambda F(x/\lambda)-\mu F(x/\mu)\|_\infty
       \le R|\lambda-\mu|,
 \qquad \lambda,\mu\in(0,1],\quad x\in\R^S.
\label{foundation:scaling}
\end{equation}
To verify \eqref{foundation:scaling} for $F=T$, use positive homogeneity of $m_{ia}$:
\[
 \lambda(T(x/\lambda))_i
 =\max_{a\in A(i)}\{\lambda r_{ia}+\lambda m_{ia}(x/\lambda)\}
 =\max_{a\in A(i)}\{\lambda r_{ia}+m_{ia}(x)\}.
\]
Replacing $\lambda$ by $\mu$ changes each expression inside the maximum by $|\lambda-\mu||r_{ia}|\le R|\lambda-\mu|$, and a maximum of finitely many terms changes by at most the largest individual change. The same calculation applies to $H_q$ with $q_{ia}^\top x$ in place of $m_{ia}(x)$. For $T^\pi$, the rescaled operators differ exactly by $(\lambda-\mu)r^\pi$. This proves \eqref{foundation:scaling}, uniformly in $x$. If $R=0$, the bound also holds with any positive constant, as the cited theorem requires.

It remains to match the two families of normalized values. In the notation of the cited theorem, the normalized discounted vector solves $v_\lambda=\lambda F((1-\lambda)v_\lambda/\lambda)$. Writing $v_\lambda=\lambda V$ turns this equation into $V=F((1-\lambda)V)$, so $v_\eps=\eps V_\eps^F$ in our notation. The normalized finite-horizon vectors are $v_N=F^N0/N$. The theorem states that these two families converge in norm together and to the same limit. The assumed discounted convergence therefore proves the assertion. Only operator properties and a scaling estimate were verified, so no definability assumption is needed.
\end{proof}

\subsection{Completion of Theorems~\ref{m:policy-value} and~\ref{m:uniform-value}}

\begin{proof}
\emph{Limits.} By \eqref{gc:eq:policy-limit}, $\eps V_\eps^\pi\to g^\pi$ for every $\pi\in\Pi_S$. Lemma~\ref{app:lem:tauberian} with $F=T^\pi$ therefore gives $(T^\pi)^N0/N\to g^\pi$. Together with the value identities \eqref{app:eq:cited-policy-values} and the approximate replies \eqref{gc:eq:stationary-approx}, this completes Theorem~\ref{m:policy-value}. Lemma~\ref{app:lem:tauberian} with $F=T$ and the optimal discounted limit $\eps V_\eps\to g^\star$ give $T^N0/N\to g^\star$; with the common controller $\pi^\star$ of Appendix~\ref{app:cited-values}, this completes \eqref{m:value-formula}.

\emph{Uniform guarantees.} Fix $\delta>0$. For the controller, $g^{\pi^\star}=g^\star$ and the fixed-policy limit give $N_1<\infty$ with $[(T^{\pi^\star})^N0]_i/N\ge g_i^\star-\delta$ for all $i$ and $N\ge N_1$, because $S$ is finite. Lemma~\ref{lem:basic}(iv) then yields
\[
 J_N(i;\pi^\star,\tau)
 \ge \frac{[(T^{\pi^\star})^N0]_i}{N}
 \ge g_i^\star-\delta
 \qquad(i\in S,\ \tau\in\mathcal Q_H,\ N\ge N_1).
\]
For nature, choose $q_{\delta/2}$ from \eqref{fd:eq:dual-approx}. Lemma~\ref{fd:lem:fixed-q}(iii) gives $N_2<\infty$ such that
\[
 J_N(i;\sigma,q_{\delta/2})
 \le d_i(q_{\delta/2})+\delta/2
 \le g_i^\star+\delta
 \qquad(i\in S,\ \sigma\in\Pi_H,\ N\ge N_2).
\]
Taking $N_\delta:=\max\{N_1,N_2\}$ and letting $q_{\delta/2}$ play the role of $q_\delta$ in Theorem~\ref{m:uniform-value} proves \eqref{m:uniform-foundation}. Both bounds hold simultaneously over initial states, all later horizons, and history-dependent opponents.
\end{proof}

\subsection{Exact nature attainment}
\label{app:nature-attainment}

The following consequence of the cited strategic results expresses exact attainment for nature through the stationary gains used later in the Bellman analysis.

\begin{corollary}[Stationary gain test]
\label{m:nature-attainment}
A selector $q\in\mathcal Q_S$ is exactly optimal for nature from every state against all history-dependent controllers, for every payoff in \eqref{m:four-payoffs}, if and only if
\begin{equation}
 d(q)=g^\star.
 \label{m:nature-attainment-test}
\end{equation}
\end{corollary}

\begin{proof}
By Lemma~\ref{fd:lem:fixed-q}(ii), the controller's best payoff against $q$ is $d_i(q)$ for each payoff convention. Exact optimality of $q$ means that this best payoff equals the value $g_i^\star$ at every state. Since $d(q)\ge g^\star$ for every selector (Appendix~\ref{app:cited-values}), exact optimality is equivalent to \eqref{m:nature-attainment-test}.
\end{proof}

\textbf{A worst reply can leave a profitable deviation.} The test \eqref{m:nature-attainment-test} checks every controller action, not only those used by an optimal controller. Consider states $s,L,H$, with $L,H$ absorbing and rewards $0,0,1$, respectively. At $s$, action $a$ moves to $L$, and action $b$ has row set $\{e_L,e_H\}$. Both deterministic policies satisfy $g^{\pi_a}=g^{\pi_b}=g^\star=(0,0,1)$, since nature answers $b$ with $e_L$. Choose the full selector $q$ with $q_{sb}=e_H$. Then $\eta^{\pi_a,q}=(0,0,1)$ but $\eta^{\pi_b,q}=(1,0,1)$. Thus $q$ is exactly worst against the optimal controller $\pi_a$, yet $d(q)=(1,0,1)\ne g^\star$: the row at the unused action $b$ permits a profitable deviation. The same example works with row set $\operatorname{co}\{e_L,e_H\}$.
\section{Proof of Theorem~\ref{m:verification}: vector Bellman verification}
\label{fv:sec:finite}

This appendix proves that every finite solution of the vector Bellman system identifies the robust value and supplies stationary optimal strategies. The argument uses three lemmas. Lemma~\ref{fv:lem:tangent} computes the limit of $T(tg+h)-tg$ as $t\to\infty$; it explains why gains must be compared before biases and yields the asymptotically affine trajectory used in Appendix~\ref{plan:appendix}. Lemma~\ref{fv:lem:nearface} converts a bias inequality that holds on a gain face into an approximate inequality on every feasible row. Lemma~\ref{fv:lem:decomposition} writes the expected total reward of any strategy pair in terms of the potentials $g$ and $h$ plus two error sums. Theorem~\ref{fv:thm:verification} then controls these error sums for each player's selected strategy. We finally specialize to fixed policies and give two examples showing why both gain restrictions are necessary.

For a vector of candidate long-run gains $g$, the recession map retains only the continuation-gain terms in the Bellman operator:
\begin{equation}
 (\widehat T g)_i:=\max_{a\in A(i)}\min_{p\in\U_{ia}}p^\top g.
 \label{fv:eq:recession}
\end{equation}
For every vector $g\in\R^S$, define the row minima and their minimizing sets by
\[
 m_{ia}(g):=\min_{p\in\U_{ia}}p^\top g,
 \qquad F_{ia}(g):=\{p\in\U_{ia}:p^\top g=m_{ia}(g)\}.
\]
If $g=\widehat Tg$, additionally put
\begin{equation}
 A_g(i):=\{a\in A(i):m_{ia}(g)=g_i\},
 \qquad (L_gh)_i:=\max_{a\in A_g(i)}
          \min_{p\in F_{ia}(g)}\{r_{ia}+p^\top h\}.
 \label{fv:eq:faces}
\end{equation}
The continuous function $p\mapsto p^\top g$ attains its minimum on the compact set $\U_{ia}$, so $F_{ia}(g)$ is nonempty and compact. Finiteness of $A(i)$ and $g_i=\max_a m_{ia}(g)$ make $A_g(i)$ nonempty. We call $a\in A_g(i)$ a gain-active action and call $F_{ia}(g)$ a gain face, even when $\U_{ia}$ is nonconvex. We define $m_{ia}(g)$ and $F_{ia}(g)$ for arbitrary $g$ because fixed-policy evaluation uses them when $g$ does not solve the optimal-control gain equation.

A finite vector gain-bias certificate is a pair $(g,h)\in\R^S\times\R^S$ satisfying
\begin{equation}
       g=\widehat Tg,\qquad g+h=L_gh.
       \label{fv:eq:gb}
\end{equation}
The first equation requires the candidate gain to be unchanged by gain optimization. The second compares immediate reward and continuation bias only among actions and rows that optimize the first comparison. The face $F_{ia}(g)$ is defined relative to $m_{ia}(g)$ for every action, including inactive actions, for which $m_{ia}(g)<g_i$.

\subsection{The limiting Bellman operator along a gain direction}

The first lemma shows that the restricted operator $L_g$ is exactly what the unrestricted operator $T$ sees along the direction $g$ at large scales.

\begin{lemma}
\label{fv:lem:tangent}
If $g=\widehat Tg$, then, for every $h\in\R^S$,
\begin{equation}
       T(tg+h)-tg\longrightarrow L_gh
       \quad\text{as }t\longrightarrow\infty.
       \label{fv:eq:tangent}
\end{equation}
Consequently, \eqref{fv:eq:gb} is equivalent to
\begin{equation}
       T(tg+h)=(t+1)g+h+o(1)\qquad(t\to\infty).
       \label{fv:eq:asymray}
\end{equation}
\end{lemma}

\begin{proof}
We first find the limit of a single row minimization with a gain penalty, then take the maximum over actions, and finally deduce the equivalence.

\emph{Step 1: one penalized row minimization.} Fix a state-action pair $(i,a)$ and a bias vector $h$. Write $m:=m_{ia}(g)$, $F:=F_{ia}(g)$, and
\[
 b:=\min_{p\in F}(r_{ia}+p^\top h),\qquad
 b_t:=\min_{p\in\U_{ia}}
       \{t(p^\top g-m)+r_{ia}+p^\top h\},\quad t\ge0.
\]
The nonnegative term $t(p^\top g-m)$ penalizes a row's departure from the gain face, and it vanishes on $F$. We show $b_t\to b$. Evaluating the minimum defining $b_t$ at a row of $F$ attaining $b$ gives the upper bound $b_t\le b$ for every $t$. For the lower bound, let $p_t$ attain $b_t$, which is possible by compactness and continuity. Rearranging $b_t\le b$ gives, for $t>0$,
\[
 0\le p_t^\top g-m
 \le \frac{b-r_{ia}-p_t^\top h}{t}
 \le\frac{2\|h\|_\infty}{t},
\]
because $b-r_{ia}=\min_{p\in F}p^\top h\le\|h\|_\infty$ and $-p_t^\top h\le\|h\|_\infty$. Now choose $t_k\to\infty$ with $b_{t_k}\to\liminf_{t\to\infty}b_t$. By compactness, a further subsequence has $p_{t_k}\to p_0\in\U_{ia}$, and the last display gives $p_0^\top g=m$, that is, $p_0\in F$. Dropping the nonnegative penalty,
\[
 \liminf_{t\to\infty}b_t=\lim_kb_{t_k}
 \ge\lim_k\bigl(r_{ia}+p_{t_k}^\top h\bigr)=r_{ia}+p_0^\top h\ge b.
\]
Together with $b_t\le b$, this proves $b_t\to b$.

\emph{Step 2: the maximum over actions.} Denote the quantities of Step~1 by $b_{ia,t}$ and $b_{ia}$ when all actions are considered. Factoring the minimal continuation gain $tm_{ia}(g)$ out of each row minimum gives
\[
 [T(tg+h)-tg]_i
 =\max_{a\in A(i)}
       \{t(m_{ia}(g)-g_i)+b_{ia,t}\}.
\]
For $a\in A_g(i)$, the first term vanishes and the expression tends to $b_{ia}$ by Step~1. For $a\notin A_g(i)$, the gain equation $g=\widehat Tg$ gives $m_{ia}(g)<g_i$, so $t(m_{ia}(g)-g_i)\to-\infty$, while $b_{ia,t}\le b_{ia}$ stays bounded above; that expression tends to $-\infty$. Since $A(i)$ is finite and $A_g(i)$ is nonempty, the maximum tends to $\max_{a\in A_g(i)}b_{ia}=(L_gh)_i$. Finiteness of $S$ upgrades coordinatewise convergence to supremum-norm convergence, proving \eqref{fv:eq:tangent}.

\emph{Step 3: the equivalence.} If \eqref{fv:eq:gb} holds, the limit in \eqref{fv:eq:tangent} is $L_gh=g+h$, which is \eqref{fv:eq:asymray}. Conversely, suppose \eqref{fv:eq:asymray} holds, without assuming the gain equation in advance. For $t>0$,
\[
 \frac{(T(tg+h))_i}{t}=\max_{a\in A(i)}\min_{p\in\U_{ia}}\Bigl\{p^\top g+\frac{r_{ia}+p^\top h}{t}\Bigr\},
\]
and every objective inside differs from $p^\top g$ by at most $(R+\|h\|_\infty)/t$. Minima and maxima preserve a common error bound, so
\[
 \left\|\frac{T(tg+h)}{t}-\widehat Tg\right\|_\infty
 \le\frac{R+\|h\|_\infty}{t}.
\]
Dividing \eqref{fv:eq:asymray} by $t$ shows that $T(tg+h)/t\to g$; hence $g=\widehat Tg$. The limit \eqref{fv:eq:tangent} now applies, and comparing it with \eqref{fv:eq:asymray} gives $L_gh=g+h$.
\end{proof}

\subsection{A compactness estimate for rows close to a gain face}

At a gain-active action, the bias equation gives a nonnegative residual only on the gain face. Nature can also use rows near the face, where the residual may be negative. The next lemma shows that such rows cost either a small amount $\theta$ or a definite multiple of their gain drift.

\begin{lemma}
\label{fv:lem:nearface}
Let $X$ be compact, let $d,e:X\to\R$ be continuous, and suppose that $d\ge0$ on $X$ and $e\ge0$ on $\{x:d(x)=0\}$. For every $\theta>0$ there is $C_\theta\in[0,\infty)$ such that
\begin{equation}
             e(x)\ge-\theta-C_\theta d(x)
             \qquad(x\in X).
             \label{fv:eq:nearface}
\end{equation}
\end{lemma}

\begin{proof}
Fix $\theta>0$. The assertion is vacuous if $X$ is empty, so assume $X\ne\varnothing$. Only points at which $e$ is at most $-\theta$ need attention, so define
\[
 M:=\max_{x\in X}[-e(x)]_+,\qquad
 B_\theta:=\{x\in X:e(x)\le-\theta\}.
\]
By continuity, $M$ is finite and $B_\theta$ is compact. If $B_\theta$ is empty, take $C_\theta=0$. Otherwise, $d$ is strictly positive throughout $B_\theta$, since a zero of $d$ would have $e\ge0$. Compactness therefore gives
\[
 \delta_\theta:=\min_{x\in B_\theta}d(x)>0.
\]
Choose $C_\theta:=M/\delta_\theta$. On $B_\theta$, we have $C_\theta d(x)\ge C_\theta\delta_\theta=M\ge-e(x)$, so $e(x)\ge-C_\theta d(x)$. Off $B_\theta$, we have $e(x)>-\theta$. Since $d\ge0$, both cases imply \eqref{fv:eq:nearface}.
\end{proof}

The additive $\theta$ cannot be removed in general. On $X:=[0,1]$, the functions $d(x):=x^2$ and $e(x):=-x$ satisfy the hypotheses, but no finite $C_0$ satisfies $-x\ge-C_0x^2$ for all small $x>0$. This matters when the estimate is summed over $N$ stages: compactness gives an error that is $o(N)$, but that error need not be bounded independently of $N$.

\subsection{An expected reward decomposition}

The finite-horizon verification bounds below come from one accounting identity. The gain vector $g$ acts as a potential whose expected increments are the gain drifts, and the bias vector $h$ acts as a potential whose increments telescope.

\begin{lemma}[Expected reward decomposition]
\label{fv:lem:decomposition}
Fix vectors $g,h\in\R^S$, an initial state $i$, and any strategy pair $(\sigma,\tau)\in\Pi_H\times\mathcal Q_H$. Let $p_t\in\U_{S_tA_t}$ be nature's realized row at stage $t$, and define the gain drift and the bias residual of that stage by
\[
 d_t:=p_t^\top g-g(S_t),\qquad e_t:=r_{S_tA_t}+p_t^\top h-g(S_t)-h(S_t).
\]
Then, for every $N\ge1$,
\begin{align}
 \sum_{t=0}^{N-1}\E_i d_t&=\E_i g(S_N)-g_i,\label{fv:eq:totalupdrift}\\
 \E_i\sum_{t=0}^{N-1}r_{S_tA_t}&=\sum_{t=0}^{N-1}\E_i g(S_t)+h_i-\E_i h(S_N)+\sum_{t=0}^{N-1}\E_i e_t.\label{fv:eq:reward-identity}
\end{align}
Moreover, $|\E_ig(S_N)-g_i|\le\spn(g)$ and $|h_i-\E_ih(S_N)|\le\spn(h)$.
\end{lemma}

\begin{proof}
By the transition rule \eqref{app:eq:transition-rule}, $\E[f(S_{t+1})\mid\mathcal G_t]=p_t^\top f$ for every $f\in\R^S$. With $f=g$, this says that $d_t=\E[g(S_{t+1})-g(S_t)\mid\mathcal G_t]$. Taking expectations and summing over $t<N$ telescopes to \eqref{fv:eq:totalupdrift}. For \eqref{fv:eq:reward-identity}, rearrange the definition of $e_t$ as
\[
 r_{S_tA_t}=g(S_t)+h(S_t)-p_t^\top h+e_t,
\]
take expectations, and use $\E[p_t^\top h]=\E h(S_{t+1})$. Summing over $t<N$, the terms $\E h(S_t)-\E h(S_{t+1})$ telescope to $h_i-\E_ih(S_N)$. The final bounds hold because $g_i$ and $\E_ig(S_N)$ both lie in $[\min_jg_j,\max_jg_j]$, and similarly for $h$.
\end{proof}

The identity holds for arbitrary $g$ and $h$. A certificate is useful because, under each player's selected strategy, the drifts have one sign and the residuals are either nonnegative or controlled by the drifts.

\subsection{Verification against history-dependent opponents}

Nature's stationary strategy must specify a row for every state and every controller action. Rows specified only for the actions used by the selected controller would not define a strategy against a different controller.

\begin{theorem}
\label{fv:thm:verification}
Suppose \eqref{fv:eq:gb} has a finite solution $(g,h)$. At every state choose
\begin{align}
 \pi^*(i)&\in\mathop{\rm argmax}_{a\in A_g(i)}
       \min_{p\in F_{ia}(g)}\{r_{ia}+p^\top h\},
       \label{fv:eq:controller}\\
 q^*_{ia}&\in
 \begin{cases}
   \mathop{\rm argmin}_{p\in F_{ia}(g)}p^\top h,
                    &a\in A_g(i),\\
   F_{ia}(g),       &a\notin A_g(i),
 \end{cases}
       \label{fv:eq:nature}
\end{align}
which specify a deterministic stationary policy and a stationary kernel. Then there exist a constant $C<\infty$ and, for every $\theta>0$, a constant $C_\theta<\infty$ such that, for every initial state $i$, every horizon $N\ge1$, and all randomized history-dependent strategies $\sigma,\tau$,
\begin{align}
 \E_i^{\pi^*,\tau}\sum_{t=0}^{N-1}r_{S_tA_t}
 &\ge N(g_i-\theta)-\spn(h)
                         -C_\theta\spn(g),
                         \label{fv:eq:lower}\\
 \E_i^{\sigma,q^*}\sum_{t=0}^{N-1}r_{S_tA_t}
 &\le Ng_i+\spn(h)+C\spn(g).
                         \label{fv:eq:upper}
\end{align}
The constants can be chosen uniformly over all selectors satisfying \eqref{fv:eq:controller}--\eqref{fv:eq:nature}, with dependence
\begin{equation}
 C=C(r,\U,g,h)<\infty,
 \qquad
 C_\theta=C_\theta(r,\U,g,h,\theta)<\infty.
 \label{fv:eq:constant-dependence}
\end{equation}
They do not depend on the initial state, the horizon, or the opposing strategies. In general, $C_\theta$ need not remain bounded as $\theta\downarrow0$; see the paragraph after the proof. In particular,
\begin{equation}
 \liminf_{N\to\infty}J_N(i;\pi^*,\tau)\ge g_i,
 \qquad
 \limsup_{N\to\infty}J_N(i;\sigma,q^*)\le g_i.
 \label{fv:eq:guarantees}
\end{equation}
For each $\Psi\in\{I^-,J^-,J^+,I^+\}$, both the max-min and the min-max values equal $g_i$, and $(\pi^*,q^*)$ is an all-state stationary saddle against history-dependent opponents. These equalities remain valid when either or both strategy classes are restricted to stationary strategies. Moreover,
\begin{equation}
 \E_i^{\pi^*,q^*}\sum_{t=0}^{N-1}r_{S_tA_t}
       =Ng_i+h_i-\E_i^{\pi^*,q^*}h(S_N),
       \label{fv:eq:saddleidentity}
\end{equation}
and
\begin{equation}
             \frac{T^N0}{N}\longrightarrow g.
             \label{fv:eq:cesaro}
\end{equation}
Thus the gain component is the same for every finite certificate, and it equals $g^\star$.
\end{theorem}

\begin{proof}
Fix the certificate $(g,h)$ and selectors satisfying \eqref{fv:eq:controller}--\eqref{fv:eq:nature}. By Lemma~\ref{fv:lem:decomposition}, the expected total reward of any strategy pair consists of the gain terms $\sum_t\E_ig(S_t)$, a bias term bounded by $\spn(h)$, and the residual sum $\sum_t\E_ie_t$. The proof therefore reduces to controlling the drifts $d_t$ and residuals $e_t$ at the rows that each player's selected strategy allows the opponent to induce. Steps~1--2 do this for the controller and prove \eqref{fv:eq:lower}; Steps~3--4 do it for nature and prove \eqref{fv:eq:upper}. Step~5 extends both guarantees to all payoff conventions, and Step~6 identifies the gain. We use the filtrations $\mathcal F_t\subseteq\mathcal G_t\subseteq\mathcal F_{t+1}$ of Appendix~\ref{app:preliminaries}.

\emph{Step 1: a controller inequality on every feasible row.}
The bias equation controls the residual only on the gain face, while nature may use any feasible row. We extend the control to every row. Define the set of bias-maximizing active actions
\[
 A_{g,h}^\star(i):=\mathop{\rm argmax}_{a\in A_g(i)}
       \min_{p\in F_{ia}(g)}\{r_{ia}+p^\top h\},
\]
and, for $a\in A_{g,h}^\star(i)$ and $p\in\U_{ia}$, the one-step drift and residual
\[
 d_{ia}(p):=p^\top g-g_i,
 \qquad e_{ia}(p):=r_{ia}+p^\top h-g_i-h_i .
\]
Since $a$ is gain active, $m_{ia}(g)=g_i$, so $d_{ia}\ge0$ on $\U_{ia}$ and its zero set is exactly $F_{ia}(g)$. Since $a$ attains the maximum defining $(L_gh)_i=g_i+h_i$, we have $\min_{p\in F_{ia}(g)}e_{ia}(p)=0$, so $e_{ia}\ge0$ on the zero set of $d_{ia}$. Lemma~\ref{fv:lem:nearface} with $X=\U_{ia}$ therefore gives a constant for each such pair. Let $C_\theta$ be the maximum of these constants over the finitely many pairs $(i,a)$ with $a\in A_{g,h}^\star(i)$. For every selector satisfying \eqref{fv:eq:controller},
\begin{equation}
 e_{i,\pi^*(i)}(p)\ge-\theta-C_\theta d_{i,\pi^*(i)}(p)
 \qquad(p\in\U_{i,\pi^*(i)}).
 \label{fv:eq:controwineq}
\end{equation}
The constant does not depend on the tie-breaking in \eqref{fv:eq:controller}.

\emph{Step 2: the controller's finite-horizon guarantee.}
Fix a nature strategy $\tau$ and let the controller use $\pi^*$. In the notation of Lemma~\ref{fv:lem:decomposition}, $A_t=\pi^*(S_t)$, so $d_t=d_{S_tA_t}(p_t)$ and $e_t=e_{S_tA_t}(p_t)$. Step~1 gives, pointwise,
\[
 d_t\ge0,\qquad e_t\ge-\theta-C_\theta d_t .
\]
The first inequality says that $g(S_t)$ is a submartingale. By \eqref{fv:eq:totalupdrift}, $\E_ig(S_t)-g_i=\sum_{s<t}\E_id_s\ge0$ for every $t$, and $\sum_{t<N}\E_id_t=\E_ig(S_N)-g_i\le\spn(g)$. Substituting into \eqref{fv:eq:reward-identity} gives
\begin{align*}
 \E_i^{\pi^*,\tau}\sum_{t=0}^{N-1}r_{S_tA_t}
 &=\sum_{t=0}^{N-1}\E_ig(S_t)+\bigl(h_i-\E_ih(S_N)\bigr)+\sum_{t=0}^{N-1}\E_ie_t\\
 &\overset{(i)}{\ge}Ng_i-\spn(h)-N\theta-C_\theta\sum_{t=0}^{N-1}\E_id_t\\
 &\overset{(ii)}{\ge}N(g_i-\theta)-\spn(h)-C_\theta\spn(g),
\end{align*}
where (i) uses $\E_ig(S_t)\ge g_i$, the bound $|h_i-\E_ih(S_N)|\le\spn(h)$ of Lemma~\ref{fv:lem:decomposition}, and the residual inequality, and (ii) uses $\sum_t\E_id_t\le\spn(g)$ and $C_\theta\ge0$. This is \eqref{fv:eq:lower}. Dividing by $N$, taking the limit inferior at fixed $\theta$, and then letting $\theta\downarrow0$ proves the controller half of \eqref{fv:eq:guarantees}. This order of limits is legitimate because neither $\pi^*$ nor $\tau$ depends on $\theta$.

\emph{Step 3: a nature inequality at every controller action.}
Nature's guarantee must hold against deviations to any action, including actions not used by $\pi^*$. For every state-action pair, define the drift and residual of nature's selected row,
\[
 \bar d_{ia}:=(q^*_{ia})^\top g-g_i,
 \qquad \bar e_{ia}:=r_{ia}+(q^*_{ia})^\top h-g_i-h_i .
\]
For an active action, $q^*_{ia}\in F_{ia}(g)$ gives $\bar d_{ia}=m_{ia}(g)-g_i=0$. Since $q^*_{ia}$ minimizes $p^\top h$ over $F_{ia}(g)$ and $(L_gh)_i$ is the maximum over active actions of these minima,
\[
 \bar e_{ia}
 =\min_{p\in F_{ia}(g)}(r_{ia}+p^\top h)-(L_gh)_i\le0 .
\]
For an inactive action, $q^*_{ia}\in F_{ia}(g)$ gives $\bar d_{ia}=m_{ia}(g)-g_i<0$, while $\bar e_{ia}$ may be positive. Because this drift is strictly negative and there are finitely many actions, a finite constant bounds the residual by a multiple of the drift's magnitude. Set
\[
 C:=\max_{(i,a):\,a\notin A_g(i)}
 \frac{\max_{p\in F_{ia}(g)}[r_{ia}+p^\top h-g_i-h_i]_+}
      {g_i-m_{ia}(g)},
\]
with an empty maximum equal to zero. The numerators are finite by compactness of the faces, and the finitely many denominators are positive. Thus $C$ is finite and does not depend on nature's tie-breaking, and
\begin{equation}
 \bar d_{ia}\le0,\qquad
 \bar e_{ia}\le C(-\bar d_{ia})
 \quad\text{for every }(i,a).
 \label{fv:eq:naturerowineq}
\end{equation}

\emph{Step 4: nature's finite-horizon guarantee.}
Let the controller use any randomized history-dependent $\sigma$, and let nature use $q^*$. Then $p_t=q^*_{S_tA_t}$, so $d_t=\bar d_{S_tA_t}$ and $e_t=\bar e_{S_tA_t}$, and \eqref{fv:eq:naturerowineq} holds at the realized action:
\[
 d_t\le0,\qquad e_t\le C(-d_t).
\]
Now $g(S_t)$ is a supermartingale. By \eqref{fv:eq:totalupdrift}, $\E_ig(S_t)\le g_i$ for every $t$, and $\sum_{t<N}\E_i(-d_t)=g_i-\E_ig(S_N)\le\spn(g)$. Substituting into \eqref{fv:eq:reward-identity},
\[
 \E_i^{\sigma,q^*}\sum_{t<N}r_{S_tA_t}
 \le Ng_i+\spn(h)+C\sum_{t<N}\E_i(-d_t)
 \le Ng_i+\spn(h)+C\spn(g),
\]
which is \eqref{fv:eq:upper}. Dividing by $N$ gives the nature half of \eqref{fv:eq:guarantees}.

\emph{Step 5: the selected pair and all four payoff conventions.}
When both selected strategies are used, the drift and the residual vanish identically:
\[
 P^{\pi^*,q^*}g=g,
 \qquad r^{\pi^*}+P^{\pi^*,q^*}h=g+h .
\]
The first identity holds because $\pi^*(i)$ is active and $q^*_{i,\pi^*(i)}\in F_{i,\pi^*(i)}(g)$. The second holds because $\pi^*(i)$ attains the maximum in $(L_gh)_i$ and $q^*_{i,\pi^*(i)}$ attains the corresponding minimum. Hence $d_t=e_t=0$, $\E_ig(S_t)=g_i$ for every $t$, and \eqref{fv:eq:reward-identity} reduces to \eqref{fv:eq:saddleidentity}.

We now pass to stationary gains. By Lemma~\ref{fd:lem:chain-limits}, the expected average reward of a stationary pair converges to its gain vector, so \eqref{fv:eq:saddleidentity} gives $\eta^{\pi^*,q^*}=g$. Applying \eqref{fv:eq:lower} to an arbitrary stationary selector $q$, dividing by $N$, and letting $N\to\infty$ and then $\theta\downarrow0$ gives $\eta^{\pi^*,q}\ge g$ for every $q\in\mathcal Q_S$. Since $q^*$ attains equality, Theorem~\ref{m:policy-value} yields
\[
 g^{\pi^*}=\inf_{q\in\mathcal Q_S}\eta^{\pi^*,q}=g.
\]
Symmetrically, \eqref{fv:eq:upper} gives $\eta^{\pi,q^*}\le g$ for every $\pi\in\Pi_D$, and $\pi^*$ attains equality, so
\[
 d(q^*)=\max_{\pi\in\Pi_D}\eta^{\pi,q^*}=g.
\]
For every $\Psi\in\{I^-,J^-,J^+,I^+\}$, Theorem~\ref{m:policy-value} and Lemma~\ref{fd:lem:fixed-q}(ii) now give
\[
 \inf_{\tau\in\mathcal Q_H}\Psi_i(\pi^*,\tau)=g_i,
 \qquad
 \sup_{\sigma\in\Pi_H}\Psi_i(\sigma,q^*)=g_i.
\]
These are the two one-sided saddle inequalities. The lower value is at least the first quantity and the upper value is at most the second, and a lower value never exceeds an upper value; hence both values equal $g_i$, and $(\pi^*,q^*)$ is an all-state saddle. The same argument applies to restricted strategy classes containing $\pi^*$ and $q^*$, because the one-sided inequalities hold against all history-dependent opponents.

\emph{Step 6: identify the gain through finite-horizon values.}
By Lemma~\ref{lem:basic}(iv), $(T^N0)_i$ is the value of the $N$-stage
game. It is at least the guarantee of $\pi^*$ and at most the guarantee
of $q^*$. Therefore \eqref{fv:eq:lower}--\eqref{fv:eq:upper} imply
\[
 g_i-\theta-\frac{\spn(h)+C_\theta\spn(g)}N
 \le\frac{(T^N0)_i}{N}
 \le g_i+\frac{\spn(h)+C\spn(g)}N.
\]
Letting $N\to\infty$ with $\theta$ fixed, and then
$\theta\downarrow0$, proves \eqref{fv:eq:cesaro}. Because there are
finitely many states, the convergence is in the supremum norm.
Theorem~\ref{m:uniform-value} identifies this limit as $g^\star$.
Thus every finite certificate has the same gain component $g=g^\star$.
\end{proof}

\textbf{Why $C_\theta$ can diverge as $\theta\downarrow0$.} Suppose a sequence $\theta_n\downarrow0$ admitted constants bounded by some $\bar C<\infty$. For each fixed $N$, letting $n\to\infty$ in the lower bound of the last display would give $(T^N0)_i\ge Ng_i-\spn(h)-\bar C\spn(g)$. Together with the upper bound, this would give $\|T^N0/N-g\|_\infty=O(1/N)$. Proposition~\ref{rev:slow-gain-estimation} constructs a model with a finite certificate whose gain-estimation error $\|T^N0/N-g^\star\|_\infty$ instead has exact order $N^{-1/(k-1)}$ with $k\ge3$. In that model, therefore, every admissible choice of constants has $C_\theta\to\infty$ as $\theta\downarrow0$.

\subsection{Fixed-policy verification with action-contingent nature}
\label{fv:sec:fixed}

Let $\pi(a\mid i)$ be a fixed stationary randomized policy, and recall that nature observes the realized action. As in \eqref{eq:aggregate}, the expected one-step reward and the set of effective transition rows are
\begin{equation}
 r_i^\pi:=\sum_{a\in A(i)}\pi(a\mid i)r_{ia},
 \qquad
 \U_i^\pi:=\left\{
    \sum_{a\in A(i)}\pi(a\mid i)p_a:
               p_a\in\U_{ia}\text{ for all }a\right\}.
 \label{fv:eq:aggregate}
\end{equation}
The set $\U_i^\pi$ is nonempty and compact. Actions of probability zero can be omitted from the product without changing $\U_i^\pi$.

\begin{proposition}
\label{fv:prop:fixedrandom}
For the post-action nature model, the fixed-policy Bellman operator is
\begin{equation}
 (T^\pi x)_i
  :=\sum_a\pi(a\mid i)
       \left(r_{ia}+\min_{p\in\U_{ia}}p^\top x\right)
  =r_i^\pi+\min_{\bar p\in\U_i^\pi}\bar p^\top x.
 \label{fv:eq:randomoperator}
\end{equation}
Its finite vector evaluation certificate is
\begin{align}
 g_i&=\sum_a\pi(a\mid i)m_{ia}(g),
                                          \label{fv:eq:randomgain}\\
 g_i+h_i&=\sum_a\pi(a\mid i)
       \left(r_{ia}+\min_{p\in F_{ia}(g)}p^\top h\right).
                                          \label{fv:eq:randombias}
\end{align}
Every finite solution $(g,h)$ satisfies $g=g^\pi$ and, for every initial state $i$,
\begin{equation}
 \begin{gathered}
 \inf_{\tau\in\mathcal Q_H}\liminf_{N\to\infty}J_N(i;\pi,\tau)
 =\inf_{\tau\in\mathcal Q_H}\limsup_{N\to\infty}J_N(i;\pi,\tau)=g_i,\\
 \frac{(T^\pi)^N0}{N}\longrightarrow g
 \quad\text{in }\|\cdot\|_\infty.
 \end{gathered}
 \label{fv:eq:fixedvalue}
\end{equation}
Both infima have the same value when nature is restricted to stationary strategies. A single stationary nature strategy attaining both infima at every initial state is obtained by choosing, separately for every positive-probability action,
\begin{equation}
 q^*_{ia}\in\mathop{\rm argmin}_{p\in F_{ia}(g)}p^\top h .
 \label{fv:eq:randomselector}
\end{equation}
Complete $q^*$ by choosing arbitrary feasible rows at zero-probability actions. Under this selected stationary pair, for every $i$ and $N\ge1$,
\[
 J_N(i;\pi,q^*)
 =g_i+\frac{h_i-\E_i^{\pi,q^*}h(S_N)}N.
\]
In this statement $g$ may be nonconstant, and the individual quantities $m_{ia}(g)$ need not equal $g_i$.
\end{proposition}

\begin{proof}
The proof transfers the argument of Theorem~\ref{fv:thm:verification} to the one-action model whose row set at state $i$ is $\U_i^\pi$. Step~1 identifies the gain and bias equations of that model with \eqref{fv:eq:randomgain}--\eqref{fv:eq:randombias}. Step~2 proves a lower bound against every history-dependent nature strategy by applying the near-face estimate separately to each action and then averaging with the policy probabilities. Step~3 shows that the selected rows attain the bound, and Step~4 identifies the limit.

\emph{Step 1: the effective-row equations.}
At state $i$, action $a$ occurs with probability $\pi(a\mid i)$. Nature observes that action and can minimize its continuation value separately. Thus the expected immediate reward is $r_i^\pi$, and the smallest expected continuation value is $\sum_a\pi(a\mid i)\min_{p\in\U_{ia}}p^\top x$. This proves the first expression in \eqref{fv:eq:randomoperator}; minimizing a weighted sum over independent row choices gives the second. Applied to $g$, the same identity shows that \eqref{fv:eq:randomgain} is precisely the gain equation of the one-action model.

To identify the effective gain face, assume \eqref{fv:eq:randomgain}. For any tuple of feasible rows $(p_a)_a$, set $\bar p:=\sum_a\pi(a\mid i)p_a$. Then
\[
 \bar p^\top g-g_i
 =\sum_a\pi(a\mid i)\bigl(p_a^\top g-m_{ia}(g)\bigr),
\]
and every term in parentheses is nonnegative. Consequently, $\bar p^\top g=g_i$ if and only if $p_a\in F_{ia}(g)$ for every action with positive probability, and this holds for every representation of $\bar p$. The effective gain face therefore consists exactly of the weighted sums of component gain faces, and
\[
 \min_{\bar p\in\U_i^\pi:\,\bar p^\top g=g_i}\bar p^\top h
 =\sum_a\pi(a\mid i)\min_{p\in F_{ia}(g)}p^\top h.
\]
Thus \eqref{fv:eq:randombias} is the bias equation of the one-action model, and the choices \eqref{fv:eq:randomselector} attain its minimum.

\emph{Step 2: a lower bound against every nature strategy.}
The gain and bias equations hold after averaging over the controller's
action. Individual actions need not have continuation gain $g_i$, so we
first measure each row relative to its own action's gain face. For every
state-action pair define
\[
 b_{ia}:=\min_{p\in F_{ia}(g)}p^\top h,\qquad
 \delta_{ia}(p):=p^\top g-m_{ia}(g),\qquad
 \xi_{ia}(p):=p^\top h-b_{ia}.
\]
On the compact set $\U_{ia}$, $\delta_{ia}\ge0$ and its zero set is
$F_{ia}(g)$. On that face, $\xi_{ia}\ge0$ by the definition of $b_{ia}$.
For fixed $\theta>0$, Lemma~\ref{fv:lem:nearface} therefore supplies a
constant for each state-action pair. Taking their finite maximum gives
$C_\theta\ge0$ such that
\begin{equation}
 \xi_{ia}(p)\ge-\theta-C_\theta\delta_{ia}(p)
 \qquad(i\in S,\ a\in A(i),\ p\in\U_{ia}).
 \label{fv:eq:effective-nearface}
\end{equation}

Now fix any history-dependent nature strategy $\tau$, and let $d_t,e_t$
be the realized drift and residual of Lemma~\ref{fv:lem:decomposition}.
Conditional on $\mathcal F_t$, the controller still samples $A_t$ from
$\pi(\cdot\mid S_t)$. The gain and bias equations give
\[
 \sum_a\pi(a\mid S_t)m_{S_ta}(g)=g(S_t),\qquad
 \sum_a\pi(a\mid S_t)(r_{S_ta}+b_{S_ta})=g(S_t)+h(S_t).
\]
These identities cancel the action-dependent reference terms when we
average $d_t$ and $e_t$. Consequently, with sums over actions of positive
probability,
\begin{align*}
 \E[d_t\mid\mathcal F_t]
 &=\sum_a\pi(a\mid S_t)
       \E[\delta_{S_ta}(p_t)\mid\mathcal F_t,A_t=a]\ge0,\\
 \E[e_t\mid\mathcal F_t]
 &=\sum_a\pi(a\mid S_t)
       \E[\xi_{S_ta}(p_t)\mid\mathcal F_t,A_t=a]\\
 &\ge-\theta-C_\theta\E[d_t\mid\mathcal F_t].
\end{align*}
The last inequality averages \eqref{fv:eq:effective-nearface}, which
holds at every realized row even when nature randomizes.
Thus no convexity of the row sets is needed.

Taking expectations gives $\E_i d_t\ge0$ and
$\E_i e_t\ge-\theta-C_\theta\E_i d_t$.
By \eqref{fv:eq:totalupdrift},
$\E_i g(S_t)\ge g_i$ and
$\sum_{t<N}\E_i d_t=\E_i g(S_N)-g_i\le\spn(g)$.
Substituting these bounds into \eqref{fv:eq:reward-identity} yields
\[
 \E_i^{\pi,\tau}\sum_{t<N}r_{S_tA_t}
 \ge N(g_i-\theta)-\spn(h)-C_\theta\spn(g)
 \qquad(N\ge1).
\]
This is the required bound, uniform over the initial state, horizon, and
nature strategy.

\emph{Step 3: the selected rows attain the bound.}
For the stationary pair $(\pi,q^*)$, the componentwise minimizations in Step~1 give
\[
 P^{\pi,q^*}g=g,
 \qquad r^\pi+P^{\pi,q^*}h=g+h.
\]
The first equality makes $\E_i^{\pi,q^*}g(S_t)=g_i$ for every $t$, and the second makes the expected residual zero. Lemma~\ref{fv:lem:decomposition} therefore gives
\[
 \E_i^{\pi,q^*}\sum_{t<N}r_{S_tA_t}
 =Ng_i+h_i-\E_i^{\pi,q^*}h(S_N),
\]
which is the displayed formula for $J_N(i;\pi,q^*)$. The last two terms differ by at most $\spn(h)$, so $J_N(i;\pi,q^*)\to g_i$. By Step~2, every nature strategy gives $\liminf_NJ_N(i;\pi,\tau)\ge g_i-\theta$ for every $\theta>0$, hence $\ge g_i$. Since a limit inferior never exceeds the corresponding limit superior, both infima in \eqref{fv:eq:fixedvalue} equal $g_i$, and the stationary $q^*$ attains them from every state.

\emph{Step 4: the finite-horizon limit and $g=g^\pi$.}
By Lemma~\ref{lem:basic}(iv), $[(T^\pi)^N0]_i$ is nature's $N$-stage value against $\pi$. It is bounded below by the lower bound of Step~2 and above by the payoff of $q^*$ in Step~3. Dividing by $N$, letting $N\to\infty$, and then letting $\theta\downarrow0$ proves $(T^\pi)^N0/N\to g$ in supremum norm. Theorem~\ref{m:policy-value} identifies this limit as $g^\pi$, so $g=g^\pi$, as required in Proposition~\ref{fv:prop:fixed-policy}.
\end{proof}

\subsection{Why both gain restrictions are necessary}

The following two counterexamples show separately that nature's minimizing face and the controller's active action set are necessary. They use the same transition geometry: every row at $s$ sends equal mass to $m$ and to $z$, so the probability of ever reaching $m$ is at most $1/2$.

\begin{example}[Nature's gain restriction cannot be omitted]
\label{er:ex-faces}
There are states $(s,m,z)$ with rewards $(0,-1,0)$, and $m,z$ are absorbing. At $s$, let
\[
 p^\lambda:=(1-\lambda/2,\lambda/4,\lambda/4),
 \qquad \U_s:=\{p^\lambda:0\le\lambda\le1\}.
\]
The robust gain is $g:=(-1/2,-1,0)$, and $h:=(0,-2,0)$ is a Bellman bias. However, the unrestricted equations $\widehat T\widetilde g=\widetilde g$ and $T\widetilde h=\widetilde g+\widetilde h$ also accept the incorrect pair $\widetilde g:=(-3/4,-1,0)$, $\widetilde h:=(0,-3,0)$.
\end{example}

\begin{proof}
We first compute the robust gain under arbitrary nature strategies and then check the two candidate pairs.

\emph{The robust gain.} Define the hitting times $\tau_m:=\inf\{t\ge0:S_t=m\}$ and $\tau_z:=\inf\{t\ge0:S_t=z\}$, with $\inf\varnothing:=\infty$. Let $\lambda_t\in[0,1]$ be nature's parameter at time $t$, possibly history dependent and randomized. Each row sends equal mass $\lambda_t/4$ to $m$ and to $z$, so the first-entry probabilities at time $t+1$ coincide:
\[
 \Prob_s(\tau_m=t+1)
 =\E_s[\mathbf1_{\{S_t=s\}}\lambda_t/4]
 =\Prob_s(\tau_z=t+1).
\]
Summing over $t$ gives $\Prob_s(\tau_m<\infty)=\Prob_s(\tau_z<\infty)$, and since the two events are disjoint, each probability is at most $1/2$. On a trajectory that reaches $m$, the sample reward average tends to $-1$; on every other trajectory the rewards are zero. Thus the sample average converges to $-\mathbf1_{\{\tau_m<\infty\}}$, and bounded convergence shows that all four payoff conventions equal $-\Prob_s(\tau_m<\infty)\ge-1/2$. Taking $\lambda_t=1$ at every visit to $s$ attains $-1/2$: the chain stays at $s$ for $t$ steps with probability $2^{-t}$, so it is absorbed almost surely, with probability $1/2$ at each absorbing state.

\emph{The proposed certificate.} Every $\lambda\in[0,1]$ satisfies
\[
 (p^\lambda)^\top g=-1/2,\qquad
 (p^\lambda)^\top h=-\lambda/2 .
\]
Every row is therefore gain minimizing, and the gain-face bias equation at $s$ reads $-1/2+0=0+\min_\lambda(-\lambda/2)$, which holds. The equations at the absorbing states are identities.

\emph{The incorrect pair.} Direct substitution gives
\[
 (p^\lambda)^\top\widetilde g=-3/4+\lambda/8,
 \qquad (p^\lambda)^\top\widetilde h=-3\lambda/4 .
\]
The gain minimum is $-3/4=\widetilde g_s$, attained only at $\lambda=0$, so $\widehat T\widetilde g=\widetilde g$. The unrestricted bias minimum is $-3/4$, attained at $\lambda=1$, so $T\widetilde h=\widetilde g+\widetilde h$ at $s$; the absorbing states again give identities. The unrestricted equations thus combine two incompatible rows. On the actual gain face $\{p^0\}$, the bias equation would require $-3/4+0=0+0$, which is false.
\end{proof}

\begin{example}[The controller's gain restriction cannot be omitted]
\label{er:ex-actionfaces}
Use states $(s,m,z)$ with rewards $(0,1,0)$, and make $m,z$ absorbing. At $s$, the controller chooses between the nominal rows $p^0:=(1,0,0)$ and $p^1:=(1/2,1/4,1/4)$. The true gain is $(1/2,1,0)$, but the unrestricted equations accept $\widetilde g:=(3/4,1,0)$ and $\widetilde h:=(0,3,0)$.
\end{example}

\begin{proof}
\emph{The true gain.} The two rows are $p^\lambda$ of Example~\ref{er:ex-faces} with $\lambda\in\{0,1\}$, now chosen by the controller. The first-entry calculation of that example, with $\lambda_t$ equal to the realized action, gives $\Prob_s(\tau_m<\infty)=\Prob_s(\tau_z<\infty)\le1/2$ under every controller strategy, including randomized history-dependent strategies. The sample average converges to $\mathbf1_{\{\tau_m<\infty\}}$, and always using $p^1$ attains absorption probability $1/2$. This proves that the gain is $(1/2,1,0)$, also against history-dependent control.

\emph{The incorrect pair.} Direct substitution gives
\[
 (p^0)^\top\widetilde g=3/4,\qquad
 (p^1)^\top\widetilde g=5/8,\qquad
 (p^0)^\top\widetilde h=0,\qquad
 (p^1)^\top\widetilde h=3/4 .
\]
The gain maximum $3/4=\widetilde g_s$ is attained only by $p^0$, whereas the unrestricted bias maximum $3/4=\widetilde g_s+\widetilde h_s$ uses $p^1$. Hence the unrestricted equations hold, but the gain-active bias equation at $s$ requires $3/4+0=0+0$, which is false. The controller restriction is therefore necessary independently of nature's restriction.
\end{proof}
\section{Proofs of Theorems~\ref{m:fixed-existence} and~\ref{m:optimal-existence}: finite Bellman solvability}
\label{sol:finite-bias-section}
\label{sol:section}

For a prescribed vector $g$ satisfying $g=\widehat Tg$, the bias equation $g+h=L_gh$ is the fixed-point equation $K_gh=h$, where $K_g:=L_g-g$. We first specialize the compact-action MDP criterion of \cite{schweitzer1985undiscounted} to fixed-policy evaluation. We then prove the optimal-control criterion by constructing two finite vectors: a lower barrier from the controller that secures the candidate gain, and an upper barrier from a full nature plan. Monotone iteration between these barriers produces the common bias. Theorems~\ref{sol:generic} and~\ref{sol:tangent-saddle} also give the bounded-orbit formulation in Theorem~\ref{m:optimal-existence}. Appendix~\ref{app:supporting-results} gives structural sufficient conditions.

\subsection{The classical criterion and fixed-policy specialization}

Consider a finite-state minimizing MDP. At state $i$, its action set $B_i$ is a nonempty compact metric space. Action $b\in B_i$ gives reward $c_i(b)$ and transition row $p_i(b)$, both continuous in $b$. A stationary deterministic policy is an element $q\in\mathcal Q:=\prod_i B_i$. Write $P_q$ for its transition matrix, with row $p_i(q_i)^{\mathsf T}$, and $c_i^q:=c_i(q_i)$ for its reward vector. Using the finite-chain quantities in Lemma~\ref{sol:markov-algebra}, define
\begin{equation}
 \begin{aligned}
 \eta(q)&:=P_q^\infty c^q,&
 \gamma_i&:=\inf_{q\in\mathcal Q}\eta_i(q),\\
 \mathcal Q_*&:=\{q\in\mathcal Q:\eta(q)=\gamma\},&
 w(q)&:=Z_{P_q}\bigl(c^q-\eta(q)\bigr).
 \end{aligned}
 \label{sol:compact-min-def}
\end{equation}
The vector $\gamma$ takes an infimum separately at each initial state, whereas $\mathcal Q_*$ requires a single policy to attain every coordinate. Compactness makes the rewards bounded, so $\gamma$ is finite, but it need not make $\mathcal Q_*$ nonempty. The canonical bias $w(q)$ is normalized by $P_q^\infty w(q)=0$; using the same normalization for every policy is essential when comparing their biases.

\begin{theorem}[Discrete-time minimization form of \cite{schweitzer1985undiscounted}]
\label{sol:schweitzer}
The coupled equations
\begin{align}
 z_i&=\min_{b\in B_i}p_i(b)^{\mathsf T}z,
 \label{sol:min-gain}\\
 z_i+h_i&=\min_{b\in B_i:\,p_i(b)^{\mathsf T}z=z_i}
       \{c_i(b)+p_i(b)^{\mathsf T}h\}
 \label{sol:min-bias}
\end{align}
have a finite solution if and only if
\begin{equation}
 \mathcal Q_*\ne\varnothing
 \quad\text{and}\quad
 \exists B<\infty:\quad w(q)\ge-B\one
       \quad\text{for every }q\in\mathcal Q_*.
 \label{sol:schweitzer-condition}
\end{equation}
Every solution satisfies $z=\gamma$. The bias bound is componentwise and uniform over all simultaneously gain-optimal stationary deterministic policies.
\end{theorem}

\begin{proof}
We obtain the existence criterion from \cite[Theorem~1]{schweitzer1985undiscounted} by reversing rewards and taking unit holding times. We check this reduction explicitly, then verify directly that the gain in any solution is $\gamma$.

\emph{Step 1: match the cited criterion and equations.}
Use rewards $\widetilde c:=-c$ in the cited maximizing semi-Markov model. Unit holding times give holding-time matrix $H_q=P_q$ and mean holding-time vector $\one$. Thus the holding times are positive and continuous, $H_q$ has the support required by $P_q$, and the compactness and continuity assumptions are exactly those stated above.

For each policy, reward reversal gives $\widetilde\eta(q)=-\eta(q)$. Since $P_q\eta(q)=\eta(q)$, its canonical bias in the cited notation is
\[
 \widetilde w(q)
 =Z_{P_q}\bigl(\widetilde c^q-H_q\widetilde\eta(q)\bigr)
 =Z_{P_q}\bigl(-c^q+P_q\eta(q)\bigr)
 =-w(q).
\]
The transformed componentwise maximal gain is $-\gamma$, and its simultaneously optimal policies are exactly $\mathcal Q_*$. The cited requirement that their canonical biases be uniformly bounded above is therefore the lower bound in \eqref{sol:schweitzer-condition}.

It remains to check that the functional equations also transform correctly. Write $(\widetilde z,\widetilde h)$ for their unknowns. The first equation is
$\widetilde z_i=\max_{b\in B_i}p_i(b)^{\mathsf T}\widetilde z$.
Setting $z:=-\widetilde z$ gives \eqref{sol:min-gain}, and the maximizing action set becomes
$L_i=\{b\in B_i:p_i(b)^{\mathsf T}z=z_i\}$.
The second equation, with $h:=-\widetilde h$, becomes
\[
 -h_i
 =\max_{b\in L_i}
   \{-c_i(b)+p_i(b)^{\mathsf T}z-p_i(b)^{\mathsf T}h\}
 =z_i-\min_{b\in L_i}\{c_i(b)+p_i(b)^{\mathsf T}h\}.
\]
The last equality uses $p_i(b)^{\mathsf T}z=z_i$ on $L_i$, and rearrangement gives \eqref{sol:min-bias}. The cited theorem now yields both directions of the existence criterion.

\emph{Step 2: identify the gain of every solution.}
Let $(z,h)$ satisfy \eqref{sol:min-gain}--\eqref{sol:min-bias}. The gain-minimizing action set at each state is nonempty and compact, and the bias expression is continuous there. Choose an action attaining the second minimum within that set. The resulting stationary policy $q$ satisfies
\[
 P_qz=z,\qquad (I-P_q)h=c^q-z.
\]
The first identity implies $P_q^\infty z=z$. Multiplying the second by $P_q^\infty$ eliminates $(I-P_q)h$ and gives $z=P_q^\infty c^q=\eta(q)$. Hence $z\ge\gamma$.

To prove the reverse inequality, fix any stationary policy $q'$, and set $P:=P_{q'}$ and $\Gamma:=P^\infty$. The gain equation gives $d:=Pz-z\ge0$. Applying the nonnegative matrix $P$ repeatedly shows $z\le P^tz$ for every $t\ge0$; Ces\`aro averaging therefore gives $z\le\Gamma z$. Moreover, $\Gamma d=0$, so Lemma~\ref{sol:markov-algebra} implies that $d$ vanishes at every recurrent state of $P$. At such a state, $q'_i$ is gain-minimizing and is consequently feasible in \eqref{sol:min-bias}. Thus
\[
 e:=c^{q'}+Ph-h-z
\]
is nonnegative on all recurrent states. The transient-state columns of $\Gamma$ are zero, so $\Gamma e\ge0$. Using $\Gamma P=\Gamma$, this inequality reads
$\eta(q')\ge\Gamma z\ge z$.
Since $q'$ was arbitrary, taking a separate infimum in each coordinate gives $\gamma\ge z$. We have proved $z=\gamma$, and the policy selected at the start of this step belongs to $\mathcal Q_*$.
\end{proof}

\begin{corollary}
\label{sol:fixed-policy}
Fix $\pi\in\Pi_S$. Apply Theorem~\ref{sol:schweitzer} with $B_i:=\U_i^\pi$ from \eqref{eq:aggregate}, $p_i(b):=b$, and $c_i(b):=r_i^\pi$. Its equations are the fixed-policy certificate \eqref{fv:eq:randomgain}--\eqref{fv:eq:randombias}. For deterministic $\pi$, they reduce to
\[
 g_i^\pi=\min_{p\in\mathcal U_{i,\pi(i)}}p^{\mathsf T}g^\pi,
 \qquad
 g_i^\pi+h_i^\pi=r_{i,\pi(i)}+
       \min_{p\in F_{i,\pi(i)}(g^\pi)}p^{\mathsf T}h^\pi.
\]
These equations are solvable exactly when one stationary nature kernel attains the entire worst-gain vector and all such kernels have a common componentwise lower bound on their canonical biases. Every solution has gain $g^\pi$.
\end{corollary}

\begin{proof}
We check that aggregation preserves both the stationary quantities and the two nested minimizations.

\emph{Stationary kernels and biases.}
The effective row set $\U_i^\pi$ is the continuous image of the nonempty compact product $\prod_a\U_{ia}$ under the map $(p_a)_a\mapsto\sum_a\pi(a\mid i)p_a$. It is therefore nonempty and compact, and its transition and reward data are continuous. By its definition, the map
\[
 q\longmapsto\Bigl(\sum_a\pi(a\mid i)q_{ia}\Bigr)_i
\]
from full nature selectors onto products of effective rows is surjective. The two descriptions induce exactly the same transition matrices and reward vector $r^\pi$, hence the same gains and canonical biases. In particular, $\gamma$ in \eqref{sol:compact-min-def} is $\gamma^\pi$ from Theorem~\ref{m:fixed-existence}, which equals $g^\pi$ by \eqref{gc:eq:policy-limit}.

\emph{Gain and bias equations.}
Independent minimization over the actionwise rows gives
\[
 \min_{b\in\U_i^\pi}b^{\mathsf T}z
 =\sum_a\pi(a\mid i)m_{ia}(z),
\]
so \eqref{sol:min-gain} is \eqref{fv:eq:randomgain}. Suppose this equality holds. For any representation $b=\sum_a\pi(a\mid i)p_a$, the gain gap is
\[
 b^{\mathsf T}z-z_i
 =\sum_a\pi(a\mid i)\bigl(p_a^{\mathsf T}z-m_{ia}(z)\bigr).
\]
Every summand is nonnegative. Thus $b^{\mathsf T}z=z_i$ holds exactly when $p_a\in F_{ia}(z)$ for every action of positive probability. Actions of probability zero impose no restriction and affect neither minimum. Minimizing the bias over this effective gain face consequently gives
\[
 \min_{b\in\U_i^\pi:\,b^{\mathsf T}z=z_i}
       \{r_i^\pi+b^{\mathsf T}h\}
 =r_i^\pi+\sum_a\pi(a\mid i)
                   \min_{p\in F_{ia}(z)}p^{\mathsf T}h.
\]
Hence \eqref{sol:min-bias} is \eqref{fv:eq:randombias}, and Theorem~\ref{sol:schweitzer} gives the claimed criterion. For deterministic $\pi$, $\U_i^\pi=\U_{i,\pi(i)}$.
\end{proof}

\subsection{Fixed points, bounded orbits, and barriers}

For $z=\widehat Tz$, the map $K_z(x):=L_zx-z$ is monotone and additively homogeneous: $x\le y$ implies $K_zx\le K_zy$, and $K_z(x+c\one)=K_zx+c\one$ for every scalar $c$. The next lemma explains why bounded iterates or two finite barriers suffice to solve its fixed-point equation. An orbit is a sequence $x,Fx,F^2x,\ldots$, with $F^0x:=x$.

\begin{lemma}
\label{sol:topical-orbits}
Let $F:\mathbb R^n\to\mathbb R^n$ be monotone and additively homogeneous. The following statements are equivalent:
\begin{enumerate}
\item There exists $h\in\mathbb R^n$ such that $Fh=h$.
\item $\sup_{N\ge0}\|F^N0\|_\infty<\infty$.
\item There exist $\ell,u\in\mathbb R^n$ such that $\ell\le F\ell$ and $Fu\le u$.
\end{enumerate}
They are also equivalent to boundedness of the orbit of any one initial vector in the supremum norm.
\end{lemma}

The bounded-orbit criterion is the zero-growth case of \cite[Lemma~3]{gaubert2004perron}. We give a direct proof that also establishes the barrier formulation.

\begin{proof}
First, monotonicity and additive homogeneity imply nonexpansiveness. Indeed, for $\delta:=\|x-y\|_\infty$, applying $F$ to $y-\delta\one\le x\le y+\delta\one$ gives
$Fy-\delta\one\le Fx\le Fy+\delta\one$.
Thus $\|Fx-Fy\|_\infty\le\|x-y\|_\infty$. In particular, $F$ is continuous and every iterate is nonexpansive.

\emph{A fixed point bounds every orbit.}
If $Fh=h$, then
\[
 \|F^N0-h\|_\infty
 =\|F^N0-F^Nh\|_\infty
 \le\|h\|_\infty
 \qquad(N\ge0),
\]
which proves (1)$\Rightarrow$(2).

\emph{A bounded orbit yields a fixed point.}
Assume (2) and write $x_N:=F^N0$. The orbit need not be monotone, so we first take its coordinatewise tail infima,
$(v_N)_i:=\inf_{k\ge N}(x_k)_i$.
Boundedness makes these vectors finite, and $v_N$ increases coordinatewise to a finite vector $v$. Since $v_N\le x_k$ for every $k\ge N$, monotonicity gives $Fv_N\le x_{k+1}$. Taking the infimum over $k\ge N$ yields $Fv_N\le v_{N+1}$. Finite dimensionality gives norm convergence $v_N\to v$, and continuity of $F$ gives $Fv\le v$.

The sequence $v,Fv,F^2v,\ldots$ is now coordinatewise decreasing. It is bounded because
\[
 \|F^Nv-x_N\|_\infty
 =\|F^Nv-F^N0\|_\infty\le\|v\|_\infty
\]
and $(x_N)$ is bounded. Hence $F^Nv$ converges in norm to a finite vector $h$. Passing to the limit in $F(F^Nv)=F^{N+1}v$ gives $Fh=h$, proving (2)$\Rightarrow$(1).

\emph{Two barriers yield a fixed point.}
A fixed point gives (3) by taking $\ell=u=h$. Conversely, suppose (3) holds. The barriers need not initially be ordered. Set $s:=\max_i(\ell_i-u_i)$ and $u':=u+s\one$. Then $\ell\le u'$ and $Fu'=Fu+s\one\le u'$. Starting at $h_0:=\ell$, define $h_{k+1}:=Fh_k$. The lower-barrier inequality gives $h_0\le h_1$, and monotonicity propagates this order. The upper-barrier inequality gives $h_k\le u'$ by induction. Thus
\[
 \ell=h_0\le h_1\le\cdots\le u',
\]
so $(h_k)$ converges in norm to a finite $h$, and continuity again gives $Fh=h$. This proves (3)$\Rightarrow$(1).

Finally, for every initial vector $x$, nonexpansiveness gives
$\|F^Nx-F^N0\|_\infty\le\|x\|_\infty$ for all $N$.
Consequently its orbit is bounded if and only if the orbit of $0$ is bounded.
\end{proof}

\begin{theorem}
\label{sol:generic}
For a prescribed $g$ satisfying $g=\widehat Tg$, the following statements are equivalent:
\begin{enumerate}
\item There exists a finite $h$ such that $g+h=L_gh$.
\item $\sup_{N\ge0}\|K_g^N0\|_\infty<\infty$.
\item There exist finite $\ell,u$ such that $\ell\le K_g\ell$ and $K_gu\le u$.
\end{enumerate}
\end{theorem}

\begin{proof}
The bias equation is $K_gh=h$. Each transition row is a probability vector, so every affine expression $r_{ia}-g_i+p^{\mathsf T}x$ is monotone in $x$ and increases by $c$ when $x$ is replaced by $x+c\one$. Row minima and action maxima preserve both properties. Hence $K_g$ satisfies the hypotheses of Lemma~\ref{sol:topical-orbits}, which gives the three equivalent conditions.
\end{proof}

\subsection{Stationary tangent saddle and canonical-bias envelope}

Fix $g=\widehat Tg$ and center the rewards by $c_{ia}:=r_{ia}-g_i$. Restrict controller actions to $A_g(i)$ and nature rows after action $a$ to $F_{ia}(g)$. Write
\[
 \Pi_g:=\prod_i A_g(i),\qquad
 \mathcal Q_g:=\prod_i\prod_{a\in A_g(i)}F_{ia}(g),\qquad
 \mathcal Q_g^\pi:=\prod_iF_{i,\pi(i)}(g).
\]
A full plan $q\in\mathcal Q_g$ supplies a row for every active action; a reply $q\in\mathcal Q_g^\pi$ supplies only the rows used by a fixed controller $\pi$. Let $P_{\pi,q}$ be the induced matrix, with row $q_{i,\pi(i)}^{\mathsf T}$ for a full plan and $q_i^{\mathsf T}$ for a reply, and define
\[
 c_i^\pi:=c_{i,\pi(i)},\qquad
 \eta^{\pi,q}:=P_{\pi,q}^\infty c^\pi,\qquad
 w^{\pi,q}:=Z_{P_{\pi,q}}c^\pi
       \quad\text{when }\eta^{\pi,q}=0.
\]
Gain activity gives $P_{\pi,q}g=g$ and hence $P_{\pi,q}^\infty g=g$. Therefore $\eta^{\pi,q}$ is the original gain of this pair minus $g$. When it is zero, $w^{\pi,q}$ is its canonical bias, with $P_{\pi,q}^\infty w^{\pi,q}=0$. All inequalities below are componentwise.

\begin{theorem}
\label{sol:tangent-saddle}
The equation $K_gh=h$ has a finite solution if and only if there exist $\bar\pi\in\Pi_g$, a full plan $\bar q\in\mathcal Q_g$, and $B<\infty$ such that
\begin{align}
 \eta^{\bar\pi,q}&\ge0
       &&\text{for every }q\in\mathcal Q_g^{\bar\pi},
 \label{sol:tangent-lower}\\
 \eta^{\pi,\bar q}&\le0
       &&\text{for every }\pi\in\Pi_g,
 \label{sol:tangent-upper}\\
 w^{\bar\pi,q}&\ge-B\one
       &&\text{for every }q\in\mathcal Q_g^{\bar\pi}
                 \text{ with }\eta^{\bar\pi,q}=0.
 \label{sol:tangent-envelope}
\end{align}
The first condition uses one controller against all nature replies. The second uses one full nature plan against all active controllers. The third bounds the canonical biases of all zero-gain replies to the controller from the first condition.
\end{theorem}

\begin{proof}
We first extract the three conditions from a common bias. For the converse, the first and third conditions give a bias for nature's problem with $\bar\pi$ fixed, while the second gives a bias for the controller's problem with $\bar q$ fixed. These two biases need not agree: their role is to provide lower and upper barriers for $K_g$.

\emph{Step 1: a common bias gives the stationary gain inequalities.}
Suppose $K_gh=h$. At each state choose $\bar\pi(i)$ attaining the action maximum. For every active action choose $\bar q_{ia}$ attaining the row minimum. The action sets are finite and the row sets are nonempty and compact, so these choices exist. At the selected action the minimum is $h_i$, while the minimum at any other active action is at most $h_i$. Consequently,
\[
 c^{\bar\pi}+P_{\bar\pi,q}h-h\ge0
      \quad(q\in\mathcal Q_g^{\bar\pi}),\qquad
 c^\pi+P_{\pi,\bar q}h-h\le0
      \quad(\pi\in\Pi_g).
\]
For each inequality, multiply by the Ces\`aro projector of its own transition matrix. This nonnegative matrix preserves the inequality, and $P^\infty(P-I)=0$ eliminates the bias. We obtain \eqref{sol:tangent-lower} and \eqref{sol:tangent-upper}.

\emph{Step 2: the same bias bounds every zero-gain reply.}
Fix any $q\in\mathcal Q_g^{\bar\pi}$ with $\eta^{\bar\pi,q}=0$, and write
\[
 P:=P_{\bar\pi,q},\qquad \Gamma:=P^\infty,\qquad
 d:=c^{\bar\pi}+Ph-h.
\]
Step 1 gives $d\ge0$, while $\Gamma d=\Gamma c^{\bar\pi}=0$. By Lemma~\ref{sol:markov-algebra}, $d$ vanishes on recurrent states and $Z_Pd=\sum_{t\ge0}P^td\ge0$. Thus, using $c^{\bar\pi}=(I-P)h+d$ and $Z_P(I-P)=I-\Gamma$,
\[
 w^{\bar\pi,q}
 =Z_P\bigl((I-P)h+d\bigr)
 =h-\Gamma h+Z_Pd
 \ge h-\Gamma h
 \ge-\spn(h)\one.
\]
The last inequality uses that each coordinate of $\Gamma h$ is a convex combination of the coordinates of $h$. The same constant $B:=\spn(h)$ works for every zero-gain reply, proving \eqref{sol:tangent-envelope}.

\emph{Step 3: the stationary conditions give a lower barrier.}
Now assume \eqref{sol:tangent-lower}--\eqref{sol:tangent-envelope}. Restrict $\bar q$ to the actions of $\bar\pi$. The resulting reply enters \eqref{sol:tangent-lower}, and $\bar\pi$ enters \eqref{sol:tangent-upper}; therefore
$0\le\eta^{\bar\pi,\bar q}\le0$.
In particular, this pair has zero centered gain simultaneously at every state.

Fixing $\bar\pi$ leaves nature a minimizing MDP with compact row sets $F_{i,\bar\pi(i)}(g)$ and rewards $c_{i,\bar\pi(i)}$. Its stationary gains are nonnegative by \eqref{sol:tangent-lower}, and the restricted $\bar q$ attains zero. Its minimal gain vector is thus zero, and its simultaneously optimal policies are exactly the zero-gain replies in \eqref{sol:tangent-envelope}. That condition supplies the uniform canonical-bias bound in Theorem~\ref{sol:schweitzer}. The theorem gives a finite bias $\ell$. Since the gain is zero, every row remains gain-active, and its bias equation is
\[
 \ell_i=\min_{p\in F_{i,\bar\pi(i)}(g)}
              \{c_{i,\bar\pi(i)}+p^{\mathsf T}\ell\}.
\]
The maximum defining $K_g$ includes action $\bar\pi(i)$, so $\ell\le K_g\ell$.

\emph{Step 4: the full nature plan gives an upper barrier.}
Fixing $\bar q$ leaves the controller a finite-action maximizing MDP with rewards $c_{ia}$ and rows $\bar q_{ia}$. Every deterministic stationary policy has gain at most zero by \eqref{sol:tangent-upper}, and $\bar\pi$ attains zero. To obtain its finite bias, apply Theorem~\ref{sol:schweitzer} after reversing rewards. The resulting minimizing problem has gain zero and an attaining policy. Its canonical-bias condition is automatic: there are only finitely many deterministic stationary policies, and each canonical bias is finite. Reversing the sign of the resulting minimizing bias gives a finite vector $u$ satisfying
\[
 u_i=\max_{a\in A_g(i)}\{c_{ia}+\bar q_{ia}^{\mathsf T}u\}.
\]
Every action is retained here because the centered gain is zero. For each such action, $\bar q_{ia}$ is feasible in the row minimum for $K_g$, and hence
\[
 (K_gu)_i
 =\max_{a\in A_g(i)}\min_{p\in F_{ia}(g)}
       \{c_{ia}+p^{\mathsf T}u\}
 \le\max_{a\in A_g(i)}\{c_{ia}+\bar q_{ia}^{\mathsf T}u\}
 =u_i.
\]
This comparison explains why $\bar q$ must be a full plan: it supplies a feasible row for every action in the maximum.

The vectors constructed in Steps 3 and 4 satisfy $\ell\le K_g\ell$ and $K_gu\le u$. Theorem~\ref{sol:generic} therefore gives a finite $h$ with $K_gh=h$. Selecting maximizing actions and minimizing rows at this $h$ yields a pair satisfying the corresponding Bellman equalities; the originally supplied pair $(\bar\pi,\bar q)$ was used only to construct the barriers.
\end{proof}
\section{Proof of unknown-gain anchored planning}
\label{plan:appendix}

We prove Theorem~\ref{ashi:main} by separating gain estimation, control of the Bellman displacement, and policy extraction. All vector norms are supremum norms unless stated otherwise. Fix a finite Bellman solution $(g,h)$. Theorem~\ref{m:verification} identifies $g$ with $g^\star$, and Lemma~\ref{fv:lem:tangent} gives
\[
 \omega_h(t):=\|T(h+tg)-h-(t+1)g\|_\infty\longrightarrow0
 \qquad(t\to\infty).
\]
Thus applying $T$ to $h+tg$ produces the next point $h+(t+1)g$ up to the error $\omega_h(t)$, called the affine defect. This defect need not vanish at any finite time (Proposition~\ref{rev:slow-gain-estimation}). We therefore compare the iterates with this approximate affine trajectory and retain its accumulated error in the bounds. The updates specialize approximately shifted Halpern iteration \cite[Algorithm~1]{zurek2025faster} to the robust Bellman operator; the property of $T$ used below is its nonexpansiveness from Lemma~\ref{lem:basic}.

The proof has three parts. First, Lemma~\ref{ashi:approx-halpern} bounds Halpern iteration near a point with a small fixed-point residual. Second, Proposition~\ref{ashi:quantitative} applies that bound at $h+Ng$: the output displacement $TZ_N-Z_N$ approaches $g$, and the output direction $Z_N/N$ also approaches $g$. Third, Appendix~\ref{ashi:identification} uses these two limits in order. The direction excludes actions that lose continuation gain; the displacement then bounds the reward loss of the selected controller. Keeping these two roles separate is essential for a vector gain.

\subsection{Halpern iteration near an approximate fixed point}

\begin{lemma}
\label{ashi:approx-halpern}
Let $F:X\to X$ be nonexpansive on a normed vector space. Fix an anchor $z_0$ and a comparison point $z_*$, and set
$D:=\|z_0-z_*\|$ and $e:=\|Fz_*-z_*\|$. Define
\[
 z_{t+1}:=\frac{2}{t+3}z_0+\frac{t+1}{t+3}Fz_t,
 \qquad t\ge0.
\]
Then, for every nonnegative integer $t$ and every nonnegative integer $N$,
\begin{equation}
 \|z_t-z_*\|\le D+\frac t3e,
 \qquad
 \|Fz_N-z_N\|\le\frac{8D}{N+3}+\frac43e.
 \label{ashi:approx-bound}
\end{equation}
In particular, $z_*$ need not be a fixed point; its residual $e$ quantifies the additional error.
\end{lemma}

\begin{proof}
The distance bound follows directly from the update. To obtain the residual bound, we then control the distance between $Fz_t$ and the anchor and use it to bound consecutive iterates. This works even when the comparison point has a nonzero residual.

\emph{Step 1: distance to the comparison point.} Nonexpansiveness gives $\|Fz_t-z_*\|\le\|Fz_t-Fz_*\|+\|Fz_*-z_*\|\le\|z_t-z_*\|+e$. Writing $z_{t+1}-z_*=\frac{2}{t+3}(z_0-z_*)+\frac{t+1}{t+3}(Fz_t-z_*)$ therefore yields
\[
 \|z_{t+1}-z_*\|
 \le\frac{2D}{t+3}
   +\frac{t+1}{t+3}\bigl(\|z_t-z_*\|+e\bigr).
\]
The first bound in \eqref{ashi:approx-bound} holds with equality at $t=0$. If it holds at time $t$, substitution gives
\[
 \|z_{t+1}-z_*\|
 \le \frac{2D}{t+3}
   +\frac{t+1}{t+3}\left(D+\frac t3e+e\right)
 =D+\frac{t+1}{3}e,
\]
which proves the distance bound by induction.

\emph{Step 2: control the distance from the operator output to the anchor.} Fix $N\ge0$. To use one bound throughout the first $N$ updates, put $M:=2D+(N/3+1)e$. For $0\le t\le N$, Step~1 and nonexpansiveness give
\[
 \|Fz_t-z_0\|
 \le \|Fz_t-Fz_*\|+\|Fz_*-z_*\|+\|z_*-z_0\|
 \le 2D+\left(\frac t3+1\right)e\le M.
\]

\emph{Step 3: bound consecutive iterates.} Let $d_t:=\|z_t-z_{t-1}\|$ for $t\ge1$. The first update gives $z_1-z_0=\frac13(Fz_0-z_0)$, so $d_1\le M/3$. For $1\le t\le N$, subtract the update defining $z_t$ from the update defining $z_{t+1}$:
\[
 z_{t+1}-z_t
 =\frac{t+1}{t+3}(Fz_t-Fz_{t-1})
  +\frac{2}{(t+2)(t+3)}(Fz_{t-1}-z_0).
\]
The two terms account for the change in the operator output and the change in the anchor weight, respectively. With Step~2, this gives
\[
 d_{t+1}\le\frac{t+1}{t+3}d_t
       +\frac{2M}{(t+2)(t+3)}.
\]
We claim $d_t\le2M/(t+2)$ for $1\le t\le N+1$. The base case follows from $d_1\le M/3\le 2M/3$. If the claim holds at $t\le N$, the recursion gives
\[
 d_{t+1}
 \le\frac{t+1}{t+3}\cdot\frac{2M}{t+2}
       +\frac{2M}{(t+2)(t+3)}
 =\frac{2M}{t+3}.
\]

\emph{Step 4: recover the fixed-point residual.} Insert $z_{N+1}$ between $Fz_N$ and $z_N$. The update gives $Fz_N-z_{N+1}=\frac{2}{N+3}(Fz_N-z_0)$, whose norm is at most $2M/(N+3)$ by Step~2. The other difference is the increment bounded in Step~3. Hence
\[
 \|Fz_N-z_N\|
 \le\|Fz_N-z_{N+1}\|+d_{N+1}
 \le\frac{4M}{N+3}
 =\frac{8D}{N+3}+\frac43e.
\]
For $N=0$, Step~3 needs no induction step, and the same calculation applies.
\end{proof}

\subsection{Gain, displacement, and direction estimates}

For the fixed Bellman solution $(g,h)$, define
\begin{equation}
 A_N:=\|h\|_\infty+\sum_{j=0}^{N-1}\omega_h(j),\qquad
 \varepsilon_N:=\frac{A_N+\|h\|_\infty}{N},\qquad N\ge1.
 \label{ashi:analysis-quantities}
\end{equation}
The quantity $A_N$ accounts for the initial distance from $h$ and the defects accumulated during Phase~I; $\varepsilon_N$ will bound the gain-estimation error. Since $\omega_h(j)\to0$, its Ces\`aro averages tend to zero, so $A_N/N\to0$ and $\varepsilon_N\to0$. These quantities depend on the unknown Bellman solution and are used only in the analysis.

\begin{proposition}
\label{ashi:quantitative}
For every integer $N\ge1$, the outputs of Algorithm~\ref{alg:robust_halpern} satisfy
\begin{align}
 \|\widehat g_N-g\|_\infty
   &\le\varepsilon_N,\label{ashi:gain-bound}\\
 \eta_N:=\|TZ_N-Z_N-g\|_\infty
   &\le\frac{8A_N}{N+3}+\frac43\omega_h(N)
                           +\frac73\varepsilon_N,\label{ashi:displacement-bound}\\
 \|Z_N-h-Ng\|_\infty
   &\le A_N+\frac N3\bigl(\omega_h(N)+\varepsilon_N\bigr).
   \label{ashi:direction-bound}
\end{align}
Consequently,
\begin{equation}
 a_N:=\left\|\frac{Z_N}{N}-g\right\|_\infty\le\frac43\varepsilon_N+\frac13\omega_h(N),
 \label{ashi:direction-normalized}
\end{equation}
and all three limits in \eqref{ashi:limits} hold.
\end{proposition}

\begin{proof}
We first compare the ordinary Bellman iterates with the affine trajectory. This gives both the gain-estimation bound and an initial-distance bound for Phase~II. We then apply Lemma~\ref{ashi:approx-halpern} to the shifted operator and read off the displacement and direction estimates.

\emph{Phase~I tracks the affine trajectory.} Phase~I starts from $x_0=0$ and applies $x_{j+1}=Tx_j$. We compare $x_j$ with $h+jg$. Inserting $T(h+jg)$ and using nonexpansiveness,
\begin{align*}
 \|x_{j+1}-h-(j+1)g\|_\infty
 &\le\|Tx_j-T(h+jg)\|_\infty
       +\|T(h+jg)-h-(j+1)g\|_\infty\\
 &\le\|x_j-h-jg\|_\infty+\omega_h(j).
\end{align*}
Starting from $\|x_0-h\|_\infty=\|h\|_\infty$, repeated application of this inequality gives
\[
 \|x_N-h-Ng\|_\infty
 \le\|h\|_\infty+\sum_{j=0}^{N-1}\omega_h(j)=A_N.
\]
Since $\widehat g_N=x_N/N$, removing the fixed bias $h$ and dividing by $N$ gives
\[
 \|\widehat g_N-g\|_\infty
 =\frac{\|x_N-Ng\|_\infty}{N}
 \le\frac{\|x_N-h-Ng\|_\infty+\|h\|_\infty}{N}
 \le\varepsilon_N,
\]
which proves \eqref{ashi:gain-bound}.

\emph{Phase~II is a Halpern iteration near $h+Ng$.} For the fixed budget $N$, Phase~II is precisely the iteration of Lemma~\ref{ashi:approx-halpern} with
\[
 F:=S_N,\quad S_N(v):=Tv-\widehat g_N,
 \qquad z_0=x_N,
 \qquad z_*:=h+Ng.
\]
The map $S_N$ is nonexpansive, since subtracting a fixed vector does not change distances between outputs. Phase~I bounds the initial distance by $D\le A_N$. The residual of the comparison point satisfies
\[
 \|S_Nz_*-z_*\|_\infty
 \le\|T(h+Ng)-h-(N+1)g\|_\infty
        +\|g-\widehat g_N\|_\infty
 \le\omega_h(N)+\varepsilon_N,
\]
so the affine defect and the gain-estimation error together control how far $z_*$ is from being fixed by $S_N$.

\emph{The two Phase~II bounds give displacement and direction.} Applying the residual bound in Lemma~\ref{ashi:approx-halpern} after $N$ iterations gives
\[
 \|TZ_N-Z_N-\widehat g_N\|_\infty
 \le\frac{8A_N}{N+3}
      +\frac43\bigl(\omega_h(N)+\varepsilon_N\bigr).
\]
Replacing $\widehat g_N$ by $g$ costs at most another $\varepsilon_N$ by the triangle inequality, which proves \eqref{ashi:displacement-bound}. The distance bound in the same lemma gives \eqref{ashi:direction-bound}. Dividing it by $N$ and adding $\|h\|_\infty/N$ gives
\[
 \left\|\frac{Z_N}{N}-g\right\|_\infty
 \le\frac{A_N+\|h\|_\infty}{N}
       +\frac{\omega_h(N)+\varepsilon_N}{3}
 =\frac43\varepsilon_N+\frac13\omega_h(N),
\]
which is \eqref{ashi:direction-normalized}. Since $A_N/N$, $\varepsilon_N$, and $\omega_h(N)$ tend to zero, the bounds for the gain error, displacement error, and normalized direction error all tend to zero. Together with $g=g^\star$, this establishes the three limits in \eqref{ashi:limits}.
\end{proof}

\subsection{Gain-active action identification and average optimality}
\label{ashi:identification}

\begin{proof}[Completion of the proof of Theorem~\ref{ashi:main}]
Proposition~\ref{ashi:quantitative} establishes the three vector limits. To prove the policy conclusion, we must also control which actions are selected. Step~1 uses the direction limit to exclude every action with $m_{ia}(g)<g_i$. Step~2 shows that the remaining actions cannot decrease expected continuation gain under any nature strategy, so their displacement error bounds their robust gain loss. Step~3 uses the finite number of deterministic stationary controllers to turn vanishing loss into exact optimality.

\emph{Step 1: greedy actions are eventually gain active.} Recall $a_N=\|Z_N/N-g\|_\infty$ from \eqref{ashi:direction-normalized} and $R=\max_{i,a}|r_{ia}|$. Each map $m_{ia}$ is $1$-Lipschitz and positively homogeneous, so $m_{ia}(Z_N)/N=m_{ia}(Z_N/N)$ and every normalized action score satisfies
\[
 \left|\frac{r_{ia}+m_{ia}(Z_N)}{N}-m_{ia}(g)\right|
 \le\frac RN+a_N.
\]
At a state $i$, choose any $a^*\in A_g(i)$, which exists because $g=\widehat Tg$. Thus $m_{ia^*}(g)=g_i$. Let $b_N:=R/N+a_N$. The score comparison defining $\pi_N(i)$ gives
\begin{align*}
 m_{i,\pi_N(i)}(g)
 &\ge\frac{r_{i,\pi_N(i)}+m_{i,\pi_N(i)}(Z_N)}{N}-b_N\\
 &\ge\frac{r_{ia^*}+m_{ia^*}(Z_N)}{N}-b_N
 \ge g_i-2b_N.
\end{align*}
The first and last inequalities use the score approximation, and the middle inequality uses greediness. In particular,
\begin{equation}
 m_{i,\pi_N(i)}(g)\ge g_i-2(R/N+a_N).
 \label{ashi:action-identification}
\end{equation}
Because $g=\widehat Tg$, every action satisfies $m_{ia}(g)\le g_i$. If any action is inactive, the finitely many strictly positive gaps therefore have a positive minimum:
\[
 \Delta_A:=\min_{i,a:\,m_{ia}(g)<g_i}
                \{g_i-m_{ia}(g)\}>0.
\]
If this set is empty, set $\Delta_A:=+\infty$. Since $R/N+a_N\to0$, there is a finite $N_A$ such that $2(R/N+a_N)<\Delta_A$ for every $N\ge N_A$. When $\Delta_A$ is finite, \eqref{ashi:action-identification} gives $m_{i,\pi_N(i)}(g)>g_i-\Delta_A$, which excludes every inactive action, however greedy ties are resolved. When all actions are active, take $N_A=1$. Only the controller's action sets need to be finite in this step.

\emph{Step 2: the robust gain loss is at most the displacement error.} Fix $N\ge N_A$ and any greedy output $\pi_N$. Gain activity of every selected action means that
\[
 p^\top g\ge m_{i,\pi_N(i)}(g)=g_i
 \qquad\text{for every }i\in S\text{ and }p\in\U_{i,\pi_N(i)}.
\]
Greediness and the definition of $\eta_N$ also give
\begin{equation}
 r_{i,\pi_N(i)}+p^\top Z_N-Z_{N,i}
 \ge (TZ_N)_i-Z_{N,i}\ge g_i-\eta_N
 \qquad(p\in\U_{i,\pi_N(i)}).
 \label{ashi:row-lower}
\end{equation}
Fix an initial state $i$ and an arbitrary randomized history-dependent nature strategy $\tau$. The two row inequalities hold at each realized row $p_t$. By the transition rule, $\E_i^{\pi_N,\tau}[p_t^\top f]=\E_i^{\pi_N,\tau}f(S_{t+1})$ for every vector $f$. Taking expectations in the gain inequality therefore gives
\[
 \E_i^{\pi_N,\tau}g(S_{t+1})
 \ge\E_i^{\pi_N,\tau}g(S_t)\ge g_i,
\]
where the second inequality follows by induction from $S_0=i$. Taking expectations in \eqref{ashi:row-lower} and rearranging the bias terms gives
\[
 \E_i^{\pi_N,\tau}r_{S_t,\pi_N(S_t)}
 \ge g_i-\eta_N+
       \E_i^{\pi_N,\tau}Z_N(S_t)-\E_i^{\pi_N,\tau}Z_N(S_{t+1}).
\]
For a payoff horizon $K\ge1$, summing these inequalities cancels the intermediate bias terms:
\begin{align}
 \E_i^{\pi_N,\tau}\sum_{t=0}^{K-1}r_{S_t,\pi_N(S_t)}
 &\ge K(g_i-\eta_N)+Z_{N,i}-\E_i^{\pi_N,\tau}Z_N(S_K)\notag\\
 &\ge K(g_i-\eta_N)-\spn(Z_N).
 \label{ashi:horizon-guarantee}
\end{align}
Both endpoint values of $Z_N$ lie between its smallest and largest coordinates, which explains the span bound in the last line. The estimate is uniform in $i$ and $\tau$. We now fix the algorithm budget $N$ and let the payoff horizon $K$ tend to infinity. Then $\spn(Z_N)/K\to0$, so dividing \eqref{ashi:horizon-guarantee} by $K$ gives
\[
 \liminf_{K\to\infty}J_K(i;\pi_N,\tau)\ge g_i-\eta_N
 \qquad\text{for every }\tau\in\mathcal Q_H.
\]
By Theorem~\ref{m:policy-value}, the infimum of the left-hand side over $\tau$ is $g_i^{\pi_N}$. Since $g=g^\star$ is the optimal robust gain, $g^{\pi_N}\le g$, and therefore
\begin{equation}
 0\le g-g^{\pi_N}\le\eta_N\one.
 \label{ashi:policy-bound}
\end{equation}
The same theorem identifies $g^{\pi_N}$ with the worst-case value under each of the four payoff conventions, so the bound holds for each convention.

\emph{Step 3: vanishing loss implies exact optimality.} Every output belongs to the finite set $\Pi_D$ of deterministic stationary controllers. If this set contains a suboptimal controller, then
\[
 \Delta_\Pi:=\min_{\pi\in\Pi_D:\,g^\pi\ne g}
                      \|g-g^\pi\|_\infty>0.
\]
If there is no suboptimal controller, set $\Delta_\Pi:=+\infty$ and $N_\Pi:=1$. Otherwise, since $\eta_N\to0$, choose $N_\Pi$ so that $\eta_N<\Delta_\Pi$ for every $N\ge N_\Pi$. For $N\ge\max\{N_A,N_\Pi\}$, \eqref{ashi:policy-bound} implies $\|g-g^{\pi_N}\|_\infty\le\eta_N<\Delta_\Pi$ and thus rules out every suboptimal output. Consequently $N_0:=\max\{N_A,N_\Pi\}$ works for every greedy tie choice. Theorem~\ref{m:policy-value} identifies each such controller's value with $g$ under all four payoff conventions, from all initial states and against history-dependent nature.
\end{proof}

\textbf{Rates when a defect modulus is available.}
Suppose a finite Bellman solution satisfies $\omega_h(t)\le C(1+t)^{-\alpha}$ for some $C<\infty$ and $\alpha>0$. Since $u\mapsto u^{-\alpha}$ is decreasing on $[1,\infty)$,
\[
 \sum_{j=0}^{N-1}\omega_h(j)
 \le C\sum_{j=0}^{N-1}(1+j)^{-\alpha}
 \le C\left(1+\int_1^N u^{-\alpha}\,du\right).
\]
Evaluating this integral gives
$\sum_{j<N}\omega_h(j)=O(N^{1-\alpha})$, $O(\log(N+1))$, or $O(1)$ according as $0<\alpha<1$, $\alpha=1$, or $\alpha>1$. Substituting these bounds into \eqref{ashi:analysis-quantities} and \eqref{ashi:gain-bound}--\eqref{ashi:displacement-bound} yields
\begin{equation}
 \|\widehat g_N-g\|_\infty+\eta_N
 =\begin{cases}
 O(N^{-\alpha}),&0<\alpha<1,\\
 O(\log(N+1)/N),&\alpha=1,\\
 O(N^{-1}),&\alpha>1.
 \end{cases}
 \label{eq:rate}
\end{equation}
The constants can depend on the chosen solution, $C$, and $\alpha$, but not on $N$. After gain-active identification, \eqref{ashi:policy-bound} gives the same upper bounds on the controller loss, and the finite-policy argument makes that loss exactly zero eventually. If the affine defect is eventually zero, its sum is finite and the gain and displacement bounds are $O(1/N)$. This is the case for polytopic row sets, for every finite solution (Theorem~\ref{sol:polytope}).

Under compactness and finite Bellman solvability alone, only $\omega_h(t)\to0$ is available. The threshold $N_0$ therefore remains instance dependent and existential. In particular, the observable residual $\|TZ_N-Z_N-\widehat g_N\|_\infty$ alone determines neither the gain-estimation error nor the gain-active threshold required by this proof.

\subsection{Finite solvability does not imply an inverse-budget rate}

\begin{proposition}
\label{rev:slow-gain-estimation}
For every integer $k\ge3$, there is a three-state robust MDP with one controller action, compact convex semialgebraic row uncertainty, and a finite vector Bellman solution, such that Algorithm~\ref{alg:robust_halpern} satisfies
\[
 \|\widehat g_N-g\|_\infty
 =\Theta\bigl(N^{-1/(k-1)}\bigr).
\]
Consequently, finite Bellman solvability does not imply an $O(N^{-1})$ gain-estimation rate, or any fixed positive algebraic exponent throughout this class.
\end{proposition}

\begin{proof}
Fix $k\ge3$. We construct a zero-reward state where nature can either remain forever or make occasional negative-reward excursions. Each excursion comes with a much smaller chance of reaching a positive absorbing state. The first option determines the limiting gain, while the second makes the finite-horizon gain estimate converge slowly. We first check the model and its Bellman certificate, then prove the upper bound through the affine defect and the matching lower bound through an explicit stationary nature strategy.

\emph{Step 1: construct compact convex row uncertainty.} Use states $(x,y,z)$, one action per state, and rewards $(0,-1,1)$. State $y$ moves deterministically to $x$, and $z$ is absorbing. Suppressing the unique action index, set
\[
 p_k(u):=(1-u-u^k,u,u^k),\qquad
 \U_x:=\operatorname{co}\{p_k(u):0\le u\le1/2\}.
\]
Each generating row is a probability vector because $u+u^k\le1/2+1/8<1$. To verify the asserted properties of this ambiguity set, we use the explicit description
\[
 \U_x=
 \bigl\{(1-v-w,v,w):0\le v\le\tfrac12,\quad
                   v^k\le w\le2^{1-k}v\bigr\}.
\]
Indeed, convexity of $u\mapsto u^k$ (Jensen's inequality) bounds the $w$-coordinate of every convex combination below by $v^k$, while each generating point satisfies $u^k\le2^{1-k}u$ on $[0,1/2]$, so the chord through the endpoint values gives the upper bound $2^{1-k}v$. Conversely, at a fixed $v$, the lower endpoint is $p_k(v)$ and the upper endpoint is the convex combination $(1-2v)p_k(0)+2vp_k(1/2)$, and mixing these two rows produces every intermediate $w$. The displayed set is closed, bounded, and defined by polynomial inequalities. Thus $\U_x$ is compact and semialgebraic, as well as convex by construction.

\emph{Step 2: verify a finite certificate.} Take $g:=(0,0,1)$ and $h:=(0,-1,0)$. At $x$, a feasible row has $p^\top g=w\ge v^k\ge0$, with equality only at $v=w=0$. Thus the unique gain-minimizing row is $e_x=p_k(0)$, and its continuation gain is $0=g_x$. At $y$ and $z$, the fixed rows $e_x$ and $e_z$ give continuation gains $0=g_y$ and $1=g_z$, respectively. The gain equation therefore holds at every state. On these gain faces, the bias equations $g_j+h_j=r_j+p^\top h$ read
\[
 x:\ 0+0=0+h_x,\qquad
 y:\ 0+(-1)=-1+h_x,\qquad
 z:\ 1+0=1+h_z,
\]
and all three hold. Thus $(g,h)$ solves the vector Bellman system, and Theorem~\ref{m:verification} identifies $g$ with the robust gain.

\emph{Step 3: use the affine defect to bound the estimation error.} The defect vanishes at $y$ and $z$, whose rows are fixed: $(T(h+tg))_y=-1+(h+tg)_x=-1=h_y+(t+1)g_y$ and $(T(h+tg))_z=1+(h+tg)_z=1+t=h_z+(t+1)g_z$. At $x$, a linear objective has the same minimum over the convex hull as over the generating curve. Since $p_k(u)^\top(h+tg)=-u+tu^k$ and the target coordinate $h_x+(t+1)g_x$ is zero, the supremum norm of the defect is
\[
 \omega_h(t)=-\min_{0\le u\le1/2}\{-u+tu^k\}.
\]
The minimum is nonpositive because $u=0$ is feasible; also $-u+tu^k\ge-u\ge-1/2$ for $t\ge0$. Hence $0\le\omega_h(t)\le1/2$. For $t>0$, the derivative of $-u+tu^k$ is $-1+ktu^{k-1}$, which increases with $u$. Its zero is $u_t:=(kt)^{-1/(k-1)}$, and this point belongs to $(0,1/2]$ exactly when $t\ge2^{k-1}/k$. It is then a minimizer on the prescribed interval. Since $tu_t^k=u_t/k$,
\[
 \omega_h(t)=\frac{k-1}{k}(kt)^{-1/(k-1)}.
\]
Since $1/(k-1)<1$, summing over integer times gives $\sum_{j<N}\omega_h(j)=O(N^{1-1/(k-1)})$; the finitely many times before the displayed formula applies contribute at most a constant. Therefore $A_N=O(N^{1-1/(k-1)})$, $\varepsilon_N=O(N^{-1/(k-1)})$, and the gain bound \eqref{ashi:gain-bound} gives $\|\widehat g_N-g\|_\infty=O(N^{-1/(k-1)})$.

\emph{Step 4: exhibit a nature strategy with a matching transient loss.} Fix $N\ge2$ and let nature use the stationary row $p_k(u)$ with $u:=\tfrac14N^{-1/(k-1)}$, which is feasible because $0<u\le1/4$. Starting from $x$, let $a_t,b_t,c_t$ be the probabilities of being at $x,y,z$ at time $t$ under this stationary chain. Since $x$ is the only predecessor of $y$, $b_t=ua_{t-1}$ for $t\ge1$. Since $z$ is absorbing and only $x$ can enter it from outside, $c_{t+1}=c_t+u^ka_t$ with $c_0=0$. Consequently,
\[
 b_t=ua_{t-1}\quad(t\ge1),\qquad
 c_t=u^k\sum_{s=0}^{t-1}a_s\le tu^k\quad(t\ge0),
\]
where the empty sum for $t=0$ is zero. For $0\le t\le N$, we have $b_t\le u$ and $c_t\le Nu^k$. Since $u\le1/4$ and $Nu^{k-1}=4^{-(k-1)}\le1/16$, the probability of being at $x$ is bounded below:
\[
 a_t=1-b_t-c_t
 \ge1-u(1+Nu^{k-1})\ge\frac12 .
\]
Consequently, $b_t=ua_{t-1}\ge u/2$ for $1\le t\le N$, while $c_t\le tu^k$. The expected reward up to time $N-1$ is therefore
\begin{align*}
 \mathbb E_x\sum_{t=0}^{N-1}r(S_t)
 &=-\sum_{t=1}^{N-1}b_t+\sum_{t=1}^{N-1}c_t
 \le-\frac{u(N-1)}2+\frac{u^kN(N-1)}2\\
 &=-\frac{u(N-1)}2\bigl(1-Nu^{k-1}\bigr)
 =-\frac{u(N-1)}2\bigl(1-4^{-(k-1)}\bigr).
\end{align*}
With only one controller action, $(T^N0)_x$ is nature's minimum expected total reward over this horizon (Lemma~\ref{lem:basic}(iv)), so it is no larger than the payoff of the stationary strategy just constructed. Using $g_x=0$, $\widehat g_N=T^N0/N$, and $(N-1)/N\ge1/2$,
\[
 \|\widehat g_N-g\|_\infty
 \ge-\frac{(T^N0)_x}{N}
 \ge\frac u4\bigl(1-4^{-(k-1)}\bigr)
 =\frac{1-4^{-(k-1)}}{16}\,N^{-1/(k-1)}.
\]
This proves the matching lower bound. For any prescribed $\alpha>0$, choose $k$ with $1/(k-1)<\alpha$; the resulting instance cannot satisfy an $O(N^{-\alpha})$ estimate, even with an instance-dependent constant.
\end{proof}

The lower bound concerns the finite-horizon gain estimator in Algorithm~\ref{alg:robust_halpern}, not the complexity of all planning algorithms. Its mechanism is explicit: a small $u>0$ produces repeated visits to the negative-reward state $y$, while transition to the positive absorbing state $z$ occurs with the much smaller probability $u^k$. These rows remain preferable for nature over long finite horizons even though the gain-minimizing row at $u=0$ stays at $x$. The nonzero affine defect records this transient effect.

\subsection{A computable budget under known instance constants}
\label{ashi:budget-section}

\textbf{The scalar residual bracket has a nonzero width.}
For every $x\in\mathbb R^S$, let $c^+:=\max_i(Tx-x)_i$ and $c^-:=\min_i(Tx-x)_i$. The definition gives $x+c^-\one\le Tx\le x+c^+\one$. By monotonicity and additive homogeneity, applying $T$ repeatedly yields
\[
 x+Nc^-\one\le T^Nx\le x+Nc^+\one\qquad(N\ge1).
\]
For example, the upper bound at $N$ implies $T^{N+1}x\le T(x+Nc^+\one)=Tx+Nc^+\one\le x+(N+1)c^+\one$; the lower bound follows by the same induction with both inequalities reversed. Nonexpansiveness gives $\|T^Nx-T^N0\|_\infty\le\|x\|_\infty$, so Theorem~\ref{m:uniform-value} implies $T^Nx/N\to g^\star$. Dividing the displayed bounds by $N$ and taking the limit therefore gives
\[
 \min_i(Tx-x)_i\le g^\star_j\le\max_i(Tx-x)_i\qquad\text{for every }j\in S.
\]
For a constant gain, the width $c^+-c^-$ is the usual span-residual stopping criterion. For a nonconstant gain, the same interval contains every coordinate of $g^\star$, and hence its width is at least $\spn(g^\star)>0$ for every $x$. Thus a stopping rule requiring this width to tend to zero cannot apply. The certificate below instead uses a known bound on the affine defect together with lower bounds on the action and policy gaps.

\begin{corollary}[Computable budget]
\label{ashi:budget}
Suppose \eqref{m:bellman} has a finite solution $(g^\star,h)$. Assume the following bounds are known: a finite $\bar h\ge\|h\|_\infty$, a finite-valued envelope $\bar\omega(t)\ge\omega_h(t)$ for integers $t\ge0$ such that $\bar\omega(t)\to0$, and positive finite lower bounds $\underline\Delta_A\le\Delta_A$ and $\underline\Delta_\Pi\le\Delta_\Pi$. Here $\Delta_A$ and $\Delta_\Pi$ are the action and policy gaps defined in the proof of Theorem~\ref{ashi:main}, with value $+\infty$ when the corresponding set is empty. Set
\[
 \bar A_N:=\bar h+\sum_{j=0}^{N-1}\bar\omega(j),\qquad
 \bar\varepsilon_N:=\frac{\bar A_N+\bar h}{N},
\]
\[
 \bar a_N:=\frac43\bar\varepsilon_N+\frac13\bar\omega(N),\qquad
 \bar\eta_N:=\frac{8\bar A_N}{N+3}+\frac43\bar\omega(N)+\frac73\bar\varepsilon_N.
\]
Then $\|\widehat g_N-g^\star\|_\infty\le\bar\varepsilon_N$ for every $N\ge1$. If
\[
 2(R/N+\bar a_N)<\underline\Delta_A,
 \qquad \bar\eta_N<\underline\Delta_\Pi,
\]
every greedy output $\pi_N$ of Algorithm~\ref{alg:robust_halpern} satisfies $g^{\pi_N}=g^\star$ and is optimal from every initial state against history-dependent nature. Both tests hold for all sufficiently large $N$.
\end{corollary}

\begin{proof}
We first substitute the known bounds into the quantitative estimates, then apply the two gap tests. Since $\|h\|_\infty\le\bar h$ and $\omega_h\le\bar\omega$, we have $A_N\le\bar A_N$ and $\varepsilon_N\le\bar\varepsilon_N$. Proposition~\ref{ashi:quantitative} therefore gives
\[
 \|\widehat g_N-g^\star\|_\infty\le\bar\varepsilon_N,
 \qquad \eta_N\le\bar\eta_N,
 \qquad a_N\le\bar a_N.
\]
If inactive actions exist, \eqref{ashi:action-identification} and the first test give
\[
 m_{i,\pi_N(i)}(g^\star)
 \ge g_i^\star-2(R/N+\bar a_N)
 >g_i^\star-\Delta_A,
\]
so no selected action can be inactive. If no inactive action exists, the conclusion already holds. Step~2 of the proof of Theorem~\ref{ashi:main} now gives $0\le g^\star-g^{\pi_N}\le\bar\eta_N\one$. The second test implies $\|g^\star-g^{\pi_N}\|_\infty<\Delta_\Pi$, excluding every suboptimal controller.

It remains to show that the tests can be met. Because $\bar\omega(N)\to0$, its Ces\`aro average tends to zero, so $\bar A_N/N\to0$. The definitions then give $\bar\varepsilon_N\to0$, $\bar a_N\to0$, and $\bar\eta_N\to0$. The prescribed lower bounds on both gaps are strictly positive, so both tests hold for all sufficiently large $N$.
\end{proof}

The two tests use only the known bounds and the proposed budget. Evaluating them therefore supplies a finite certified budget before Algorithm~\ref{alg:robust_halpern} is run. The requirement $\bar\omega(N)\to0$ ensures that this search can succeed; an arbitrary upper bound on the defect would certify a budget only conditionally on the tests being satisfied.
\section{Sufficient conditions for Bellman solvability}
\label{app:supporting-results}
\label{app:supporting}

This appendix proves the sufficient conditions in Proposition~\ref{m:sufficient-regimes} and applies them to common uncertainty sets. The main quantity is the discounted remainder $V_\eps-g/\eps$. Lemma~\ref{sol:bounded-centering} shows that a bounded remainder yields a finite Bellman bias. Polytopic rows give an invariant half-line and hence a bounded remainder (Theorem~\ref{sol:polytope}). A support gap or continuous Ces\`aro projections instead give a uniform bound on the discounted resolvent remainder (Lemma~\ref{sol:uniform-resolvent} and Theorem~\ref{sol:support-gap}). We then use support preservation to apply these results to divergence balls.

\subsection{Structural multichain existence regimes}
\label{sol:structural-section}

The conditions below concern the original ambiguity sets and allow several recurrent classes. Throughout, $V_\eps$ denotes the unnormalized discounted value, so $V_\eps=T((1-\eps)V_\eps)$ for $0<\eps<1$.

\begin{lemma}
\label{sol:bounded-centering}
Suppose $\eps_k\downarrow0$ and, for some vector $g$, the centered values $h_k:=V_{\eps_k}-g/\eps_k$ are bounded. Then every subsequential limit $h$ of $(h_k)$ gives a solution $g=\widehat Tg$, $g+h=L_gh$.
\end{lemma}

\begin{proof}
Fix a convergent subsequence $h_k\to h$, retaining the same index. We first identify its gain equation by normalizing the discounted values. We then subtract the gain term before passing to the limit, which retains the information needed for the bias equation.

\emph{Step 1: identify the gain.} Set $v_k:=\eps_kV_{\eps_k}=g+\eps_kh_k$. Boundedness of $h_k$ gives $v_k\to g$. Multiplying the discounted equation by $\eps_k$ gives
\[
 (v_k)_i=\max_{a\in A(i)}
 \{\eps_k r_{ia}+(1-\eps_k)m_{ia}(v_k)\}.
\]
Here we used $m_{ia}(\lambda v)=\lambda m_{ia}(v)$ for $\lambda\ge0$, which follows by factoring $\lambda$ out of the row minimum. Each $m_{ia}$ is continuous, and the maximum contains finitely many actions. Passing to the limit therefore gives $g_i=\max_a m_{ia}(g)$ for every $i$, or $g=\widehat Tg$.

\emph{Step 2: retain the finite remainder.} Put $t_k:=\eps_k^{-1}-1$ and $z_k:=(1-\eps_k)h_k$. Substitution into the discounted equation yields
\[
 (1-\eps_k)V_{\eps_k}=t_kg+z_k,
 \qquad
 g+h_k=T(t_kg+z_k)-t_kg.
\]
Since $t_k\to\infty$ and $z_k\to h$, we can compare the second equation with the tangent limit at the fixed vector $h$. Nonexpansiveness of $T$ and Lemma~\ref{fv:lem:tangent}, whose gain hypothesis was established in Step~1, give
\[
 \|T(t_kg+z_k)-t_kg-L_gh\|_\infty
 \le\|z_k-h\|_\infty+
     \|T(t_kg+h)-t_kg-L_gh\|_\infty\longrightarrow0.
\]
Returning to the discounted equation in Step~2, its left side $g+h_k$ tends to $g+h$ and its right side tends to $L_gh$. Hence $g+h=L_gh$, completing both Bellman equations.
\end{proof}

\begin{theorem}
\label{sol:polytope}
If every $\mathcal U_{ia}$ is a polytope, a finite vector gain-bias pair exists. Moreover, there are $g$, $h$, and $t_0$ such that
\begin{equation}
 T(tg+h)=(t+1)g+h\qquad\text{for all }t\ge t_0.
 \label{sol:poly-ray}
\end{equation}
In fact, every finite solution $(g,h)$ of \eqref{fv:eq:gb} satisfies \eqref{sol:poly-ray} for some $t_0$; equivalently, its affine defect $\omega_h(t)$ vanishes for all sufficiently large $t$. Finally, $\sup_{0<\eps<1}\|V_\eps-g^\star/\eps\|_\infty<\infty$.
\end{theorem}

\begin{proof}
We first obtain one invariant half-line from the finite collection of vertices. We then show that every Bellman solution defines such a half-line, and use its exactness to bound the discounted remainder.

\emph{Step 1: obtain a finite solution.} Let $E_{ia}$ be the finite vertex set of $\U_{ia}$. A linear objective attains its minimum over a polytope at a vertex, so
\[
 (Tx)_i=\max_{a\in A(i)}\min_{p\in E_{ia}}
                   \{r_{ia}+p^\top x\}.
\]
This finite maximum of finite minima is piecewise affine with finitely many pieces. Moreover, $T$ is nonexpansive in the supremum norm. These are precisely the hypotheses of Kohlberg's invariant-half-line theorem \cite{kohlberg1980invariant}, which supplies $g,h,t_0$ satisfying \eqref{sol:poly-ray}; see also \cite[Section~2]{akian2012policy}. The equivalence in Lemma~\ref{fv:lem:tangent} now gives the two Bellman equations, and Theorem~\ref{fv:thm:verification} identifies their gain as $g^\star$.

\emph{Step 2: every solution has an eventually exact ray.} Let $(g,h)$ be any finite solution. For each state $i$, the map $t\mapsto(T(tg+h))_i$ is a maximum of minima of the finitely many affine functions $r_{ia}+p^\top h+t\,p^\top g$, with $p\in E_{ia}$. Two nonidentical affine functions have at most one crossing, so their ordering is fixed for all sufficiently large $t$. Consequently, $(T(tg+h))_i=\alpha_i+\beta_it$ on that tail. Lemma~\ref{fv:lem:tangent} gives
\[
 (\beta_i-g_i)t+\alpha_i-g_i-h_i\longrightarrow0.
\]
An affine function can tend to zero at infinity only if its slope and constant term both vanish. Thus $\beta_i=g_i$ and $\alpha_i=g_i+h_i$. There are finitely many states, so a common $t_0$ gives \eqref{sol:poly-ray} at every state. This also proves the assertion about $\omega_h$.

\emph{Step 3: bound the discounted remainder.} Fix vectors satisfying \eqref{sol:poly-ray}, with $g=g^\star$. For $\eps$ small enough that $t:=\eps^{-1}-1\ge t_0$, use $W_\eps:=g/\eps+h$ as an approximate discounted fixed point. Since $(1-\eps)W_\eps=tg+(1-\eps)h$ and $(t+1)g+h=W_\eps$, nonexpansiveness gives
\[
 \|T((1-\eps)W_\eps)-W_\eps\|_\infty
 \le\|T(tg+(1-\eps)h)-T(tg+h)\|_\infty
 \le\eps\|h\|_\infty.
\]
The discounted Bellman map has contraction factor $1-\eps$. Comparing its fixed point $V_\eps$ with $W_\eps$ therefore gives
\[
 \|V_\eps-W_\eps\|_\infty
 \le(1-\eps)\|V_\eps-W_\eps\|_\infty
      +\eps\|h\|_\infty.
\]
Moving the first term to the left and dividing by $\eps>0$ yields $\|V_\eps-W_\eps\|_\infty\le\|h\|_\infty$. Thus $\|V_\eps-g/\eps\|_\infty\le2\|h\|_\infty$ for all sufficiently small $\eps$. On the remaining interval, $\eps$ is bounded away from zero and $\|V_\eps-g/\eps\|_\infty\le(R+\|g\|_\infty)/\eps$. Combining the two bounds proves the claimed supremum is finite.
\end{proof}

For a support gap and for continuous Ces\`aro projections, bounded centering comes instead from a uniform estimate of the discounted resolvent.

\begin{lemma}
\label{sol:uniform-resolvent}
Let $\mathcal P$ be a compact family of finite stochastic matrices such that, for some $\delta>0$, every entry of every $P\in\mathcal P$ is either zero or at least $\delta$. Then $P\mapsto P^\infty$ is continuous on $\mathcal P$, and
\begin{equation}
 \sup_{P\in\mathcal P,\,0\le\beta<1}
 \left\|(I-\beta P)^{-1}-\frac{P^\infty}{1-\beta}\right\|_\infty
 <\infty.
 \label{sol:uniform-resolvent-bound}
\end{equation}
\end{lemma}

\begin{proof}
The singular part of the discounted resolvent is $P^\infty/(1-\beta)$. We first show that its projector varies continuously with $P$. We then express the remaining part using a matrix inverse that stays nonsingular at $\beta=1$.

\emph{Step 1: keep the recurrent structure fixed along convergent sequences.} Write $\Gamma(P):=P^\infty$, and suppose $P_k\to P$ in $\mathcal P$. An entry of $P_k$ whose limit is zero is eventually smaller than $\delta$ and must therefore be zero. An entry whose limit is positive is eventually positive. Since there are finitely many entries, $P_k$ eventually has exactly the support of $P$. The transient states $D$ and recurrent classes $C_j$ are determined by this support, so they are also eventually fixed.

Write $Q$ for the transient block, $B_j$ for the transitions from $D$ into $C_j$, and $\nu_j$ for the invariant probability vector of the irreducible block on $C_j$. The vector $(I-Q)^{-1}B_j\one$ contains the probabilities of eventual absorption in $C_j$. Hence the Ces\`aro projector has blocks
\[
 \Gamma(P)_{D,C_j}=(I-Q)^{-1}B_j\one\,\nu_j^\top,
 \qquad \Gamma(P)_{C_j,C_j}=\one\nu_j^\top,
\]
with all other blocks zero; the first formula is omitted when $D$ is empty. Transience makes $I-Q$ invertible. Each invariant probability vector is continuous along the sequence: every convergent subsequence has a limit satisfying the stationary equations of the limiting irreducible block, and that block has a unique invariant probability vector. The displayed block formulas and continuity of matrix inversion therefore give $\Gamma(P_k)\to\Gamma(P)$. This is the support-gap implication in \cite[Theorem~2]{schweitzer1985undiscounted}.

\emph{Step 2: remove the singular part of the resolvent.} Fix $P$, abbreviate $\Gamma:=\Gamma(P)$, and set $M_\beta:=I-\beta(P-\Gamma)$ for $0\le\beta\le1$. Since $\Gamma^2=\Gamma$, the space is the direct sum of the image and kernel of $\Gamma$. Both subspaces are invariant under $P$ because $P\Gamma=\Gamma P=\Gamma$. On the image, $M_\beta$ is the identity; on the kernel, it agrees with $I-\beta P$. Thus $M_\beta$ is invertible when $\beta<1$, since $I-\beta P$ has the Neumann inverse $\sum_{t\ge0}\beta^tP^t$. At $\beta=1$, invertibility follows from $M_1=I-P+\Gamma$ and Lemma~\ref{sol:markov-algebra}. For $\beta<1$, the identity $(I-\beta P)\Gamma=(1-\beta)\Gamma$ gives
\begin{equation}
 (I-\beta P)^{-1}-\frac{\Gamma}{1-\beta}
 =(I-\beta P)^{-1}(I-\Gamma)
 =M_\beta^{-1}(I-\Gamma),
 \label{sol:resolvent-extension}
\end{equation}
where the last equality holds because $I-\Gamma$ maps into the kernel of $\Gamma$, on which the two matrices being inverted coincide. Step~1 makes $\Gamma(P)$ continuous. The right-hand side of \eqref{sol:resolvent-extension} is consequently continuous on the compact set $\mathcal P\times[0,1]$, including $\beta=1$, because $M_\beta$ is invertible throughout that set. Its norm has a finite maximum, proving \eqref{sol:uniform-resolvent-bound}. This argument uses Ces\`aro projections and therefore also covers periodic recurrent classes.
\end{proof}

\begin{theorem}
\label{sol:support-gap}
Suppose there is $\delta>0$ such that
\begin{equation}
 p_j=0\quad\text{or}\quad p_j\ge\delta
 \qquad\text{for every }i,a,\ p\in\mathcal U_{ia},\ j.
 \label{sol:support-gap-assumption}
\end{equation}
Then a finite gain-bias pair exists. More strongly, for some vector $g$ and some finite constant $C$,
\begin{equation}
 \left\|V_\eps-\frac{g}{\eps}\right\|_\infty\le C
 \qquad(0<\eps<1).
 \label{sol:support-gap-conclusion}
\end{equation}
The same assertion holds for every fixed stationary policy. For a fixed randomized policy, the constant may depend on its positive action probabilities.
\end{theorem}

\begin{proof}
We first obtain a centering bound uniform over stationary nature selectors and deterministic controllers. That uniformity lets us take nature's infimum and the controller's maximum without changing the error bound. Lemma~\ref{sol:bounded-centering} will then produce the Bellman pair. We treat a fixed randomized policy at the end by reducing it to effective rows.

\emph{Step 1: bound every stationary pair with the same constant.} For each deterministic controller $\pi$, the compact kernel family $\{P^{\pi,q}:q\in\mathcal Q_S\}$ inherits the support gap, because its rows are rows of the sets $\U_{i,\pi(i)}$. Lemma~\ref{sol:uniform-resolvent} with $\beta=1-\eps$, the bound $\|r^\pi\|_\infty\le R$, and finiteness of $\Pi_D$ give a constant $C<\infty$, independent of $\pi$, $q$, and $\eps$, such that
\begin{equation}
 \left\|(I-(1-\eps)P^{\pi,q})^{-1}r^\pi
              -\frac{\eta^{\pi,q}}\eps\right\|_\infty\le C,
 \qquad \eta^{\pi,q}:=(P^{\pi,q})^\infty r^\pi.
 \label{sol:pair-bound}
\end{equation}
The first term is the discounted value of the stationary pair $(\pi,q)$.

\emph{Step 2: preserve the bound through optimization.} Define $\gamma_i^\pi:=\inf_{q\in\mathcal Q_S}\eta_i^{\pi,q}$. By Lemma~\ref{lem:basic}(iii), nature has a stationary discounted-optimal reply, so the coordinatewise infimum over $q$ of the discounted pair values is $V_\eps^\pi$. To see explicitly why the error remains uniform after minimization, \eqref{sol:pair-bound} gives, for every state $i$ and selector $q$,
\[
 \frac{\eta_i^{\pi,q}}\eps-C
 \le\bigl[(I-(1-\eps)P^{\pi,q})^{-1}r^\pi\bigr]_i
 \le\frac{\eta_i^{\pi,q}}\eps+C.
\]
Taking the infimum over $q$ in all three expressions preserves both inequalities. Since $C$ is independent of $q$, this yields
\begin{equation}
 \left\|V_\eps^\pi-\frac{\gamma^\pi}\eps\right\|_\infty\le C.
 \label{sol:fixed-discount-bound}
\end{equation}
Theorem~\ref{m:policy-value} identifies $\gamma^\pi=g^\pi$. Lemma~\ref{lem:basic}(ii)--(iii) gives $V_{\eps,i}=\max_{\pi\in\Pi_D}V_{\eps,i}^\pi$ at every state. Taking this finite maximum preserves the error bound, so $\|V_\eps-\bar g/\eps\|_\infty\le C$ with $\bar g_i:=\max_{\pi\in\Pi_D}g_i^\pi$, and Theorem~\ref{m:uniform-value} identifies $\bar g=g^\star$. This is \eqref{sol:support-gap-conclusion}.

\emph{Step 3: extract Bellman solutions.} Lemma~\ref{sol:bounded-centering} now gives a finite optimal Bellman pair. Its one-action version, applied to the model with row sets $\U_{i,\pi(i)}$ and the bound \eqref{sol:fixed-discount-bound}, gives the pair for each fixed deterministic policy.

\emph{Step 4: reduce a fixed randomized policy to effective rows.} For a fixed randomized policy, put $\alpha_\pi:=\min\{\pi(a\mid i):\pi(a\mid i)>0\}>0$. Every positive coordinate of an effective row $\bar p:=\sum_a\pi(a\mid i)p_a$ has a positive summand of size at least $\alpha_\pi\delta$. Thus the compact effective row sets $\U_i^\pi$ have support gap $\alpha_\pi\delta$, and the preceding one-action argument gives bounded centering and a finite fixed-policy pair, with a constant that may depend on $\pi$. Proposition~\ref{fv:prop:fixedrandom} identifies these effective-row equations with \eqref{fv:eq:randomgain}--\eqref{fv:eq:randombias}.
\end{proof}

\begin{corollary}
\label{sol:continuous-projections}
For every deterministic stationary controller $\pi$, define $\mathcal P^\pi:=\{P^{\pi,q}:q\in\mathcal Q_S\}$. Suppose that $P\mapsto P^\infty$ is continuous on each compact family $\mathcal P^\pi$. Then the optimal discounted values have bounded vector centering as in \eqref{sol:support-gap-conclusion}, and the full vector Bellman system has a finite solution. The same conclusion holds for every fixed deterministic policy.
\end{corollary}

\begin{proof}
Only Step~1 of Lemma~\ref{sol:uniform-resolvent} uses the support-gap assumption: its purpose is to make $P^\infty$ continuous. Here that continuity is assumed on each $\mathcal P^\pi$, so the resolvent extension \eqref{sol:resolvent-extension} is continuous on $\mathcal P^\pi\times[0,1]$. Compactness gives a finite bound for each deterministic policy, and finiteness of $\Pi_D$ gives a common bound over all of them. We therefore recover \eqref{sol:pair-bound}. Step~2 of Theorem~\ref{sol:support-gap} passes this uniform bound through nature's infimum and the controller's maximum, yielding bounded centering for $V_\eps^\pi$ and $V_\eps$. Lemma~\ref{sol:bounded-centering}, applied to the corresponding one-action operator and to $T$, gives the fixed-policy and optimal-control solutions.
\end{proof}

\textbf{Fixed supports.}
If every compact row family has a fixed support, each coordinate that is positive throughout a family is a continuous positive function on a compact set and therefore has a positive minimum there. Taking the minimum over the finitely many such state-action-coordinate triples gives a common $\delta>0$, so Theorem~\ref{sol:support-gap} applies.

The support-gap condition allows several support patterns and any number of recurrent classes. Its one-player antecedent is \cite[Theorem~2]{schweitzer1985undiscounted}. Here the uniform resolvent estimate and the finite maximization over controller policies give the robust optimality version.

\subsection{Divergence balls and the support-loss boundary}
\label{sol:divergence-balls}

For a nominal row $p^0\in\Delta(S)$, define $J(p^0):=\{j:p^0_j>0\}$ and $\Delta(J):=\{p\in\Delta(S):p_j=0\text{ for }j\notin J\}$. We use $D_{\mathrm{KL}}(p\|p^0):=\sum_{j:p_j>0}p_j\log(p_j/p^0_j)$, with value $+\infty$ if $p_j>0=p^0_j$ for some $j$. The two orders of the KL divergence lead to different support conditions.

Table~\ref{tab:solvability-sets} collects the sufficient conditions for common uncertainty sets. The polytopic cases include interval constraints, contamination sets $(1-\rho)p^0+\rho\Delta(S)$ with $0\le\rho\le1$, and Wasserstein balls on a finite ground space, which are projections of finite transport polytopes. These are sufficient guarantees; failure of a support-gap test alone does not establish nonsolvability.

\begin{table}[!htbp]
\centering
\caption{Sufficient conditions for finite Bellman solvability. Here $J_{ia}:=J(p^0_{ia})$; conditions apply to every state-action row. The KL orders are written explicitly. If $|J_{ia}|=1$, every support-restricted ball is a singleton and needs no radius restriction.}
\label{tab:solvability-sets}
\small
\begin{tabularx}{\linewidth}{@{}p{0.36\linewidth}X@{}}
\toprule
Uncertainty set & Sufficient guarantee \\
\midrule
TV, interval, contamination, and finite-space Wasserstein sets & Polytopic; Proposition~\ref{m:sufficient-regimes}(A) applies. Wasserstein balls are projections of transport polytopes. \\
Forward KL: $D_{\mathrm{KL}}(p\|p^0_{ia})\le\rho_{ia}$ & Condition (B) applies if $\rho_{ia}<\min_{j\in J_{ia}}[-\log(1-p^0_{ia,j})]$, with $-\log0:=+\infty$. Failure of solvability is possible at equality (Example~\ref{sol:ex-kl-fail}). \\
Reverse KL: $D_{\mathrm{KL}}(p^0_{ia}\|p)\le\rho_{ia}$, $p\in\Delta(J_{ia})$ & Condition (B) applies for every finite radius. The support restriction is automatic for a full-support nominal row. \\
Pearson $\chi^2$ and squared Hellinger balls, $p\in\Delta(J_{ia})$ & Condition (B) applies if $\rho_{ia}<t_{ia}/(1-t_{ia})$ for $\chi^2$, or $\rho_{ia}<1-\sqrt{1-t_{ia}}$ for squared Hellinger, where $t_{ia}:=\min_{j\in J_{ia}}p^0_{ia,j}<1$. Singleton supports impose no radius restriction. \\
$\ell_q$ balls with $1<q<\infty$, $p\in\Delta(J_{ia})$ & Condition (B) applies if $\rho_{ia}<\min_{j\in J_{ia}}p^0_{ia,j}$. \\
Finite scenario sets, e.g., an ensemble of estimated models & Condition (B) applies, since a finite family has finitely many positive coordinates. \\
\bottomrule
\end{tabularx}
\end{table}

All the support-gap entries of the table rest on the fixed-support observation above. The KL, $\chi^2$, and Hellinger entries are proved below. For an $\ell_q$ ball $\{p\in\Delta(J):\|p-p^0\|_q\le\rho\}$, a row with $p_j=0$ for some $j\in J$ has $\|p-p^0\|_q\ge p^0_j$, so $\rho<\min_{j\in J}p^0_j$ forbids deleting any nominal coordinate; the ball is compact, so its support is fixed. A finite scenario set is compact, and the minimum of its finitely many positive coordinates is a support gap. The cases $q=1$ and $q=\infty$ are polytopes.

\begin{corollary}
\label{sol:mixed-row-regime}
Suppose each row family $\U_{ia}$ is either a polytope or is compact and satisfies $p_j=0$ or $p_j\ge\delta_{ia}$ for all its rows and coordinates, where $\delta_{ia}>0$. Then the full vector Bellman system has a finite solution $(g^\star,h)$, and $\sup_{0<\eps<1}\|V_\eps-g^\star/\eps\|_\infty<\infty$. The fixed-policy conclusions of Theorem~\ref{sol:support-gap} also hold.
\end{corollary}

\begin{proof}
Replace each polytopic row family by its finite vertex set. Linear row minimization is unchanged for every continuation vector, so the reduced model has the same $T$, the same fixed-policy operators $T^\pi$, and the same discounted fixed points. Its compact row families have a common positive support gap: take the minimum of the gaps of the unreplaced families and of the positive coordinates of the finitely many vertices, omitting empty collections.

Theorem~\ref{sol:support-gap} applies to the reduced model, since its proof does not require convex row sets. It gives bounded centering for the unchanged optimal and fixed-policy discounted values, and their normalized limits identify the centering vectors with the original gains. Lemma~\ref{sol:bounded-centering} uses only the operator $T$ and the discounted values, both unchanged, together with Lemma~\ref{fv:lem:tangent} for $T$; its conclusion therefore holds in the original model. Applying it to the original operator, and its one-action version to each original effective-row operator, gives the asserted Bellman solutions. This also covers fixed randomized policies, as in Theorem~\ref{sol:support-gap}.
\end{proof}

\begin{corollary}
\label{sol:kl-balls}
For each $(i,a)$, let $p^0_{ia}$ be a nominal row, put $J_{ia}:=J(p^0_{ia})$, and let $\rho_{ia}\ge0$. Assume that every row family is one of the following:
\begin{enumerate}
\item a forward KL ball $\U_{ia}:=\{p\in\Delta(S):D_{\mathrm{KL}}(p\|p^0_{ia})\le\rho_{ia}\}$ with
\begin{equation}
\rho_{ia}<\rho_{ia}^{\mathrm{del}}
:=\min_{j\in J_{ia}}-\log(1-p^0_{ia,j}),
\qquad -\log 0:=+\infty;
\label{sol:kl-deletion-radius}
\end{equation}
\item a support-restricted reverse KL ball $\U_{ia}:=\{p\in\Delta(J_{ia}):D_{\mathrm{KL}}(p^0_{ia}\|p)\le\rho_{ia}\}$ with $\rho_{ia}<\infty$; or
\item a polytope, including a total-variation ball.
\end{enumerate}
Then the conclusions of Corollary~\ref{sol:mixed-row-regime} hold. For a full-support nominal row, the support restriction in item~2 is automatic.
\end{corollary}

\begin{proof}
By Corollary~\ref{sol:mixed-row-regime} and the fixed-support observation, it suffices to show that every row family of type~1 or~2 is compact and has a fixed support. Both balls contain their nominal row. They are closed because the corresponding extended KL divergence is lower semicontinuous, and they are compact because they lie in a finite-dimensional probability simplex.

\emph{Forward KL.} Abbreviate the nominal support by $J$. A finite forward divergence forces $p\in\Delta(J)$, so if $|J|=1$ the ball is the nominal singleton. Otherwise, fix $j\in J$ and a row $p\in\Delta(J)$ with $p_j=0$. Since $\sum_{k\in J\setminus\{j\}}p_k=1$, splitting $\log(p_k/p^0_k)=\log\bigl(p_k(1-p^0_j)/p^0_k\bigr)-\log(1-p^0_j)$ and using nonnegativity of the KL divergence give
\begin{equation}
D_{\mathrm{KL}}(p\|p^0)
=-\log(1-p^0_j)
 +D_{\mathrm{KL}}\!\left(p\middle\|
       \frac{p^0|_{J\setminus\{j\}}}{1-p^0_j}\right)
\ge-\log(1-p^0_j),
\label{sol:kl-face-cost}
\end{equation}
where the conditional distribution is extended by zero outside $J\setminus\{j\}$. Equality holds at this conditional distribution, so $-\log(1-p^0_j)$ is the exact cost of deleting coordinate $j$. Under \eqref{sol:kl-deletion-radius}, no coordinate in $J$ can vanish on the ball. Each such coordinate is continuous and positive on the compact ball, so it has a positive minimum there, and the support is exactly $J$.

\emph{Reverse KL.} Here $D_{\mathrm{KL}}(p^0\|p)<\infty$ forces $p_j>0$ for each $j\in J$, and the bound is quantitative. With $H(p^0):=-\sum_{j\in J}p^0_j\log p^0_j$, and since every term $-p^0_k\log p_k$ is nonnegative,
\[
D_{\mathrm{KL}}(p^0\|p)
=-H(p^0)-\sum_{k\in J}p^0_k\log p_k
\ge-H(p^0)-p^0_j\log p_j .
\]
The constraint $D_{\mathrm{KL}}(p^0\|p)\le\rho$ therefore implies $-p^0_j\log p_j\le\rho+H(p^0)$, that is, $p_j\ge\exp[-(\rho+H(p^0))/p^0_j]>0$, for every $j\in J$ throughout the ball. Together with the restriction to $\Delta(J)$, this makes the support exactly $J$.
\end{proof}

The strict radius in \eqref{sol:kl-deletion-radius} is the exact condition for a forward KL ball to preserve its nominal support. When the nominal support contains at least two states, the threshold is finite. At equality, the conditional nominal row from \eqref{sol:kl-face-cost}, for a minimizing coordinate $j$, lies in the ball and has zero mass at $j$. Its mixtures with $p^0$ remain feasible by convexity and have arbitrarily small positive mass at $j$, so the support gap also fails at this finite threshold.

A larger forward radius may still give a solvable model through another condition. For example, if $\rho_{ia}\ge\log(1/\min_{j\in J_{ia}}p^0_{ia,j})$, then every $p\in\Delta(J_{ia})$ satisfies $D_{\mathrm{KL}}(p\|p^0_{ia})=\sum_jp_j\log p_j-\sum_jp_j\log p^0_{ia,j}\le\log(1/\min_{j\in J_{ia}}p^0_{ia,j})$, because the first sum is nonpositive. The ball is then the entire simplex face, and item~3 applies. If $|J_{ia}|\le2$, every forward KL ball is a closed interval or a singleton in that face, so it is a polytope at every radius. For reverse KL, finite divergence bounds the coordinates in $J$ away from zero but does not bound positive coordinates outside $J$ away from zero. This is why item~2 includes the support restriction when the nominal row is sparse.

\begin{example}[A KL ball allowing transition removal]
\label{sol:ex-kl-removal}
Order the states as $(x,y,z)$, give each state one action, and let $y,z$ be absorbing. At $x$, take $\U_x:=\{p\in\Delta(S):D_{\mathrm{KL}}(p\|(1/5,3/5,1/5))\le0.3\}$. The row $(1/4,3/4,0)$ is feasible because its divergence is $\log(5/4)<0.3$. Its mixtures with the nominal row are also feasible and have arbitrarily small positive $z$-coordinates, so the uniform support-gap condition fails. However, $(1,0,0)$ is excluded because its divergence is $\log5>0.3$. Since $p\mapsto p_x$ is continuous on the compact ball and stays below one, there is $\kappa>0$ with $p_x\le1-\kappa$ throughout $\U_x$.

The absorption probabilities are $P^\infty(x,\cdot)=(0,p_y/(1-p_x),p_z/(1-p_x))$, while the rows at $y,z$ are fixed. Their denominators are bounded below by $\kappa$, so $P\mapsto P^\infty$ is continuous on the induced family. Proposition~\ref{m:sufficient-regimes}(C) thus applies to this nonpolytopic set despite its changing support. For rewards $(0,0,1)$, the finite pair $g=(0,0,1)$, $h=0$ also verifies solvability directly: the minimal continuation gain at $x$ is $\min_{p\in\U_x}p_z=0$, attained at $(1/4,3/4,0)$, and both Bellman equations hold at all three states.
\end{example}

\begin{example}[Failure at the first forward-KL support loss]
\label{sol:ex-kl-fail}
There are three states $(x,y,z)$ and one action per state. State $y$ returns deterministically to $x$, state $z$ is absorbing, and the rewards are $(r_x,r_y,r_z)=(1/2,-1,0)$. At $x$, take the KL ball
\[
 \U_x:=\{p\in\Delta(\{x,y,z\}):D_{\mathrm{KL}}(p\|(1/3,1/3,1/3))\le\log(3/2)\}.
\]
Its radius equals the threshold in \eqref{sol:kl-deletion-radius}. The robust gain is $g^\star=0$, but the vector Bellman system has no finite solution.
\end{example}

\begin{proof}
Step~1 shows that every stationary row has gain zero, so $g^\star=0$. Step~2 exhibits feasible rows that leak to $z$ with probability of order $t^2$ while moving extra mass $t$ to the negative state $y$. Step~3 shows that these rows make every candidate bias fail, and that the canonical biases are unbounded below.

\emph{Step 1: the stationary gains.} Write a row as $p=(1-u-v,u,v)$, where $u$ is the probability of moving to $y$ and $v$ the probability of moving to $z$. On the face $v=0$,
\[
 D_{\mathrm{KL}}(p\|(1/3,1/3,1/3))=\log(3/2)+d_{\mathrm{Ber}}(u\|1/2),
\]
where $d_{\mathrm{Ber}}(u\|q):=u\log(u/q)+(1-u)\log((1-u)/(1-q))$, with the usual zero-term convention. This binary KL divergence vanishes only at $u=q$, so the only feasible row with $v=0$ is $p^*:=(1/2,1/2,0)$. Under $p^*$, the class $\{x,y\}$ has invariant distribution $(2/3,1/3)$ and average reward $\frac23\cdot\frac12+\frac13\cdot(-1)=0$, and state $z$ also has reward zero.

If a fixed feasible row has $v>0$, each visit to $x$ is followed by absorption at $z$ with probability $v$, and a failed attempt returns to $x$ within two steps. Absorption therefore occurs almost surely, with expected time at most $2/v$ from $x$, and the rewards after absorption are zero. This row therefore has gain zero from every state. All stationary rows thus have gain zero. There is only one controller policy, so Theorem~\ref{m:policy-value} gives $g^\star=0$, also against history-dependent nature.

\emph{Step 2: construct feasible rows with small leakage.} For $0<t\le1/8$, define $p_t:=(1/2-t-t^2,\,1/2+t,\,t^2)$. These coordinates are nonnegative and sum to one. The row moves an additional amount $t$ toward the negative-reward state $y$, while its leakage to $z$ is only $t^2$. We check that the KL constraint permits this imbalance for all sufficiently small $t$.

Write $f(s):=s\log(3s)$. Its derivatives at $1/2$ are $f'(1/2)=1+\log(3/2)$ and $f''(1/2)=2$. The first two coordinates of $p_t$ lie in $[1/4,3/4]$, on which $f'''(s)=-1/s^2$ is bounded. Taylor's theorem therefore gives a remainder of order $t^3$ when those two terms are expanded at $1/2$. Keeping the third term $f(t^2)=t^2\log(3t^2)$ exact yields
\[
 D_{\mathrm{KL}}(p_t\|(1/3,1/3,1/3))-\log(3/2)=t^2\bigl[1+\log(2t^2)\bigr]+O(t^3).
\]
Indeed, the two zeroth-order terms sum to $2f(1/2)=\log(3/2)$; the increments sum to $-t^2$, so the first-order terms contribute $-(1+\log(3/2))t^2$; the second-order terms contribute $(t+t^2)^2+t^2=2t^2+O(t^3)$; and adding $f(t^2)=t^2\log(3t^2)$ gives the display, because $1-\log(3/2)+\log(3t^2)=1+\log(2t^2)$. After division by $t^2$, the right side is $1+\log(2t^2)+O(t)$, which tends to $-\infty$. Thus $p_t\in\U_x$ for all sufficiently small $t>0$.

\emph{Step 3: no finite bias.} Suppose a finite Bellman pair existed. Theorem~\ref{m:verification} would identify its gain as $g^\star=0$, so every feasible row would belong to the gain face. The bias equation at $y$ gives $h_y=h_x-1$, and substituting this into the equation at $x$ and cancelling $h_x$ gives
\[
 0=\min_{(1-u-v,u,v)\in\U_x}\bigl\{1/2-u+v(h_z-h_x)\bigr\}.
\]
For the feasible row $p_t$, the expression inside the minimum equals $-t+t^2(h_z-h_x)$, which is negative for small $t$ and any fixed finite $h$. This contradicts the required minimum of zero. The same obstruction appears in the canonical biases, which is the failure of the bound in Theorem~\ref{m:fixed-existence}. The unique recurrent state of the $p_t$-chain is $z$, so $w_z=0$, and the Poisson equations give $w_y=w_x-1$ and $t^2w_x=1/2-(1/2+t)=-t$. Hence $w_x=-1/t$, which has no uniform lower bound as $t\downarrow0$.
\end{proof}

The same reasoning applies to other divergences on a fixed nominal support. Let $f:[0,\infty)\to\mathbb R\cup\{+\infty\}$ be convex and lower semicontinuous, with $f(1)=0$. On $\Delta(J)$, define
\[
 D_f(p\|p^0):=\sum_{j\in J}p^0_j f(p_j/p^0_j).
\]
To determine when the ball preserves its support, fix $j\in J$ with $t:=p^0_j<1$ and suppose $p_j=0$. The remaining coordinates sum to one, so Jensen's inequality with probability weights $p^0_k/(1-t)$ gives
\[
 \sum_{k\in J\setminus\{j\}}\frac{p^0_k}{1-t}
       f\!\left(\frac{p_k}{p^0_k}\right)
 \ge f\!\left(\sum_{k\in J\setminus\{j\}}
               \frac{p_k}{1-t}\right)
 =f\!\left(\frac1{1-t}\right).
\]
Multiplying by $1-t>0$ and adding the deleted coordinate's contribution $t f(0)$ gives
\begin{equation}
D_f(p\|p^0)\ge
\beta_f(p^0_j),\qquad
\beta_f(t):=t f(0)+(1-t)f\!\left(\frac1{1-t}\right),
\label{sol:f-face-cost}
\end{equation}
and the bound is attained at $p=p^0(\cdot\mid J\setminus\{j\})$. Define $\beta_f(1):=+\infty$, since a probability row on a singleton support cannot delete its only coordinate. Consequently, the ball
\[
 \mathcal U_{ia}^{f}
 :=\{p\in\Delta(J_{ia}):D_f(p\|p^0_{ia})\le\rho_{ia}\}
\]
has a fixed support whenever
\[
 \rho_{ia}<\min_{j\in J_{ia}}\beta_f(p^0_{ia,j}),
\]
and it can therefore replace either KL type in Corollary~\ref{sol:kl-balls}. Its compactness follows from lower semicontinuity on the compact simplex face.

For $0<t<1$, the deletion costs are $-\log(1-t)$ for forward KL ($f(s)=s\log s$), $t/(1-t)$ for Pearson's $\chi^2(p\|p^0):=\sum_{j\in J}(p_j-p^0_j)^2/p^0_j$ ($f(s)=(s-1)^2$), and $1-\sqrt{1-t}$ for the squared Hellinger divergence $H^2(p,p^0):=1-\sum_j\sqrt{p_jp^0_j}$ ($f(s)=1-\sqrt s$). Each cost is increasing in $t$, so the minimum over $j$ is attained at the smallest nominal coordinate, which gives the thresholds in Table~\ref{tab:solvability-sets}. For Hellinger balls with sparse nominal rows, the restriction to $\Delta(J)$ is part of this condition, because the Hellinger divergence remains finite for rows charging coordinates outside $J$. Total-variation balls need no radius restriction for solvability because they are polytopes.